\documentclass[11pt]{article}
\usepackage[utf8]{inputenc}
\usepackage[margin=1in]{geometry}
\usepackage{lmodern}

\usepackage{microtype}
\usepackage{graphicx}
\usepackage{subfigure}
\usepackage{booktabs}
\usepackage[round]{natbib}

\usepackage{amsmath,amsthm,amsfonts,amssymb,mathdots,array,mathrsfs,bm,bbm,stmaryrd,xcolor}

\usepackage{breakcites}

\usepackage[T1]{fontenc}
\usepackage{enumerate}
\usepackage{inputenc}

\usepackage{graphicx}
\usepackage{wrapfig}

\usepackage{algorithm}
\usepackage{algorithmic}
\usepackage{color}

\usepackage[utf8]{inputenc} 
\usepackage[T1]{fontenc}    

\usepackage[colorlinks,linkcolor=red,filecolor=blue,citecolor=blue,urlcolor=blue]{hyperref}

\usepackage{url}            
\usepackage{amsfonts}       
\usepackage{nicefrac}       
\usepackage{microtype}      
\usepackage{xcolor}
\usepackage{enumitem}
\usepackage{bbm}

\usepackage{tikz}
\usetikzlibrary{mindmap}

\usepackage{tikzpagenodes}
\usepackage{mathabx}

\usepackage[toc,page]{appendix}

\usepackage{math_notation}

\newtheorem*{theorem*}{Theorem}

\newtheorem{theorem}{Theorem}
\newtheorem{lemma}{Lemma}

\newtheorem{remark}{Remark}

\newtheorem{example}{Example}

\newtheorem{proposition}{Proposition}

\def \bBtrue {\bB^{\natural}}

\def \oracle{\mathsf{oracle}}

\def \vcinfo#1#2{m_{#1\rightarrow #2}}
\def \cvinfo#1#2{\hat{m}_{#1\rightarrow #2}}

\def\ln#1{{#1}^{\mathsf{L}}}
\def\rn#1{{#1}^{\mathsf{R}}}
\def\lloginfo#1#2{h_{\ln{#1}\rightarrow (\ln{#1}, \rn{#2})}}
\def\rloginfo#1#2{h_{\rn{#2}\rightarrow (\ln{#1}, \rn{#2})}}

\def\sym{(\mathsf{sym})}

\begin{document}

\title{\huge \bf The Phase Transition Phenomenon of \\ Shuffled Regression\vspace{-0.05in}}

\author{\vspace{0.2in}\\\\
\textbf{Hang Zhang}, \ \textbf{Ping Li} \\\\
Cognitive Computing Lab\\
Baidu Research\\
10900 NE 8th St. Bellevue, WA 98004, USA\\\\
  \texttt{\{zhanghanghitomi,\ pingli98\}@gmail.com}
}

\date{\vspace{0.2in}}

\maketitle

\begin{abstract}\vspace{0.2in}
\noindent We study the phase transition
phenomenon inherent in the shuffled (permuted) regression problem, which has found numerous applications in databases, privacy, data analysis, etc. For the permuted regression task: $\bY = \bPi^{\natural}\bX\bB^{\natural}$, the goal is to recover the permutation matrix $\bPi^{\natural}$ as well as the coefficient matrix $\bB^{\natural}$. It has been empirically observed in prior studies that when recovering $\bPi^{\natural}$, there exists a phase transition phenomenon: the error rate drops to zero rapidly once the parameters reach certain thresholds. In this study, we aim to precisely identify the locations of the phase transition points by leveraging techniques from {\em message passing} (MP).

\vspace{0.15in}

\noindent In our analysis, we first transform the permutation recovery problem into a probabilistic graphical model. We then leverage the analytical tools rooted in the message passing (MP) algorithm and derive an equation to track the convergence of the MP algorithm. By linking this equation to the branching random walk process, we are able to characterize the impact of the \emph{signal-to-noise-ratio} ($\snr$) on the permutation recovery.  Depending on whether the signal is given or not, we separately investigate the oracle case and the non-oracle case. The bottleneck in identifying the phase transition regimes lies in deriving closed-form formulas for the corresponding critical points, but only in rare scenarios can one obtain such precise expressions. To tackle this technical challenge, this study proposes the Gaussian approximation method, which  allows us to obtain the closed-form formulas in almost all scenarios.
In the oracle case, our method can fairly accurately predict the phase transition $\snr$. In the non-oracle case, our  algorithm can predict the maximum allowed number of permuted rows and uncover its dependency on the sample number.

\vspace{0.15in}

\noindent Our numerical experiments reveal that the observed phase transition  points are well aligned with our theoretical predictions. It is anticipated that our study will motivate exploiting MP algorithms (and the related techniques) as an effective tool for solving the  permuted regression problems, which have found many applications in machine learning, privacy, databases, etc.

\end{abstract}

\newpage

\section{Introduction}

In this paper, we consider the following permuted (shuffled) linear regression problem:
\begin{align}\label{eqn:problem}
\bY = \bPi^{\natural}\bX\bB^{\natural} + \sigma \bW,
\end{align}
where $\bY\in \RR^{n\times m}$ denotes the matrix of observations,
$\bPi^{\natural}\in \{0, 1\}^{n\times n}$ is the permutation
matrix, $\bX\in \RR^{n\times p}$ is the design matrix,
$\bB^{\natural}\in \RR^{p\times m}$ is the matrix of signals (regressors), $\bW \in \RR^{n\times m}$ denotes the
additive noise matrix (with unit variance), and $\sigma^2$ is the noise variance. The task is to recover both the signal matrix $\bB^{\natural}$ and the permutation matrix $\bPi^{\natural}$. The research on this challenging permuted regression problem dates back
at least to 1970s under the name ``broken sample problem''~\citep{degroot1971matchmaking, goel1975re, degroot1976matching, degroot1980estimation,bai2005broken}.
Recent years have witnessed a revival of this problem
due to its broad spectrum of applications in (e.g.,) privacy protection,
data integration, etc.~\citep{unnikrishnan2015unlabeled, pananjady2016linear, slawski2017linear, pananjady2017denoising, slawski2019two, zhang2020optimal}.

\vspace{0.1in}

Specifically, this paper will focus on studying the ``phase transition'' phenomenon in recovering the whole permutation matrix $\bPi^{\natural}$: the error rate for the permutation recovery sharply drops to zero once the parameters reach certain thresholds. In particular, we leverage techniques in the \emph{message passing} (MP) algorithm literature to identify the precise positions of the phase transition thresholds.
The bottleneck in identifying the phase transition regimes lies in deriving closed-form formulas for the corresponding critical points. This is a highly challenging task because  only in rare scenarios can one obtain such precise expressions. To tackle the difficulty, we propose the Gaussian approximation method which allows us to obtain the closed-form formula in almost all scenarios. We should mention that, in previous studies~\citep{slawski2019two, slawski2017linear, pananjady2017denoising, zhang2022benefits, zhang2020optimal}, this phase transition phenomenon was empirically observed.

\subsection{Related work}\ The problem we study in this paper simultaneously touches two distinct areas of research: (A) permutation recovery, and (B) message passing (MP). In the  literature of
permuted linear regression, essentially all existing works used the same setting~\eqref{eqn:problem}.~\citet{pananjady2016linear, slawski2017linear} consider the
single observation model (i.e., $m =1$) and prove that the \emph{signal-to-noise-ratio} ($\snr$) for
the correct permutation recovery is $\OPrate{n^c}$, where $c>0$
is some positive constant.~\citet{slawski2019two, zhang2020optimal, zhang2022benefits} investigate the multiple observations model (i.e., $m>1$) and suggest that the $\snr$ requirement can be significantly decreased, from $\OPrate{n^c}$ to $\OPrate{n^{c/m}}$. In particular,~\citet{zhang2020optimal} develop an estimator which we will leverage and analyze for studying the phase transition phenomenon.  Our analysis leads to the precise identification of  the locations of the phase transition thresholds.

Another line of related research comes from
the field of
statistical physics. For example, using the replica method,~\citet{mezard1985replicas,mezard1986mean}
study the \emph{linear assignment problem} (LAP), i.e.,
$\min_{\bPi}\sum_{i,j}\bPi_{ij}\bE_{ij}$
where $\bPi$ denotes a permutation matrix and
$\bE_{ij}$ is i.i.d random variable uniformly distributed in $[0, 1]$.
\citet{martin2005random} then generalize
 LAP to multi-index matching and presented an investigation
based on MP algorithm.
Recently,~\citet{caracciolo2017finite, malatesta2019fluctuations}
extend the distribution of $\bE_{ij}$ to a broader class.
However, all the above works exhibit no phase transition.
In~\citet{chertkov2010inference}, this method is extended
to the particle tracking problem, where a phase transition phenomenon
is observed. Later,~\citet{semerjian2020recovery} modify
it to fit the graph matching problem, which paves way
for our work in studying the permuted linear regression problem.

\subsection{Our contributions}
We propose the first framework to identify
the precise locations of phase transition thresholds
associated with permuted linear regression.
In the oracle case where $\bB^{\natural}$ is known, our scheme
is able to determine the phase transition $\snr$. In the non-oracle case where
$\bB^{\natural}$ is not given, our method will also predict
the maximum allowed number of permuted rows and uncover its dependence on the ratio $p/n$.   In our analysis, we
identify the precise positions of the phase transition points in the large-system limit, e.g.,
$n$, $m$, $p$ all approach to infinity with
$m/n\rightarrow \tau_m$, $p/n\rightarrow\tau_p$. Interestingly, numerical results well
match predictions  even when $n, m, p$
are not large. There is one additional contribution. In the graphical model based on the linear assignment problem, we can modify the graph and design a scheme for partial recovery, which is a separate contribution and may be further analyzed for future study.

Here, we  would also like to briefly mention the technical challenges.
Compared with the previous works~\citep{mezard1986mean, mezard1987solution, parisi2002finite, linusson2004proof, mezard2009information, talagrand2010mean,   semerjian2020recovery},
where the edge weights are relatively simple,
our edge weights usually involve
high-order interactions across
Gaussian random variables and are densely
correlated.
To tackle this issue, our proposed approximation method
to compute the phase transition thresholds
consists of three parts: 1) performing Gaussian approximation;   2) modifying the leave-one-out technique;
and 3) performing size correction. A detailed
explanation can be found in Section~\ref{sec:non_oracle_case}. Hopefully, our approximation method
will serve independent technical interests for researchers in the machine learning community.

\subsection{Notations}
In this paper, $a\asconverg b$ denotes
$a$ converges almost surely to $b$.
We denote $f(n)\simeq g(n)$ when
$\lim_{n\rightarrow \infty}{f(n)}/{g(n)} = 1$, and  $f(n) = \OPrate{g(n)}$
if the sequence ${f(n)}/{g(n)}$
is bounded~in probability,
and  $f(n)=\oprate{g(n)}$
if  ${f(n)}/{g(n)}$
converges to zero in probability.
The inner product between two vectors (resp. matrices) are denoted
as $\la \cdot, \cdot \ra$.  For two distributions $d_1$ and $d_2$,~we~write
$d_1 \cong d_2$ if they are equal up to  normalization.
Moreover,  $\calP_n$ denotes the set of
all possible~permutation matrices:
$\calP_n \defequal \{\bPi \in \{0, 1\}^{n\times n}, \sum_i \bPi_{ij} = 1, \sum_j \bPi_{ij} = 1\}$.
The \emph{signal-to-noise-ratio} is
$\snr = \frac{\Fnorm{\bB^{\natural}}^2}{m\cdot \sigma^2}$, where
$\Fnorm{\cdot}$ is the Frobenius norm and
$\sigma^2$ is the variance of the sensing noise.

\section{Permutation Recovery Using the Message Passing Algorithm}

Inspired by~\citet{mezard2009information, chertkov2010inference, semerjian2020recovery}, we leverage  tools from the statistical
physics to identify the locations of the phase transition threshold. We start this section with a brief review of
the \emph{linear assignment problem}
(LAP), which reads as
\begin{align}
\label{eq:lap_def}
\wh{\bPi} =~& \
\argmin_{\bPi \in \calP_n} \la \bPi, \bE \ra,
\vspace{-5mm}
\end{align}
where $\bE\in \RR^{n\times n}$ is a fixed matrix
and $\calP_n$ denotes the set of all possible permutation matrices.
We follow the approach in~\citet{mezard2009information,semerjian2020recovery} and introduce a
probability measure over the permutation matrix $\bPi$,
which is written as
\begin{align}
\label{eq:prob_measure}
\mu(\bPi)  =~& (\nfrac{1}{Z})\prod_{i} \Ind\big(1 - \sum_j \bPi_{ij}\big)
\prod_j \Ind\big(1 - \sum_i \bPi_{ij}\big) \times \exp\bigg(-\beta \sum_{i,j}\bPi_{ij} \bE_{ij}\bigg),
\end{align}
where $\Ind(\cdot)$ is the indicator function,
$Z$ is the normalization constant of the probability measure
$\mu(\bPi)$, and $\beta > 0$ is an auxiliary parameter.

\newpage

It is easy to  verify the following two properties:
\begin{itemize}[leftmargin=*]
\item the
ML estimator in~\eqref{eq:lap_def} can be
rewritten as $\wh{\bPi} = \argmax_{\bPi} \mu(\bPi)$;
\item
the probability measure
$\mu(\bPi)$ concentrates on
$\wh{\bPi}$
when letting $\beta\rightarrow \infty$. 	
\end{itemize}

In the next three subsections, we  study the impact of $\{\bE_{ij}\}$ on the
reconstructed permutation $\wh{\bPi}$ with the
\emph{message passing} (MP) algorithm. First, we associate a probabilistic graphical model with the probability measure defined in~\eqref{eq:prob_measure}. Then, we rewrite the solution in~\eqref{eq:lap_def} in the language of the MP algorithm. Finally, we derive an equation \eqref{eq:succ_mp_single} to track the convergence of the MP algorithm. By exploiting relation of \eqref{eq:succ_mp_single} to the \emph{branching random walk} (BRW) process, we identify the phase transition points corresponding to the LAP in~\eqref{eq:lap_def}.

\vspace{0.1in}
\subsection{Construction of the graphical model}
\vspace{0.1in}

Firstly, we construct the factor graph
associated with the probability measure in~\eqref{eq:prob_measure}.
Adopting the same strategy as in
Chapter $16$ of~\citet{mezard2009information}, we conduct the following operations:
\begin{itemize}[leftmargin=*]
\item associating
each variable $\bPi_{ij}$ a variable node $v_{ij}$;
\item connecting the variable node $v_{ij}$ a function node
representing the term $e^{-\beta \bPi_{ij}\bE_{ij}}$;
\item linking each constraint
$\sum_i\bPi_{ij} = 1$ to a function node
and similarly for the constraint $\sum_j\bPi_{ij} = 1$.
\end{itemize}
A graphical representation is available in Figure~\ref{fig:whole_graph}.

\begin{figure}[h]
\centering

\vspace{-0.1in}
\includegraphics[width = 6in]{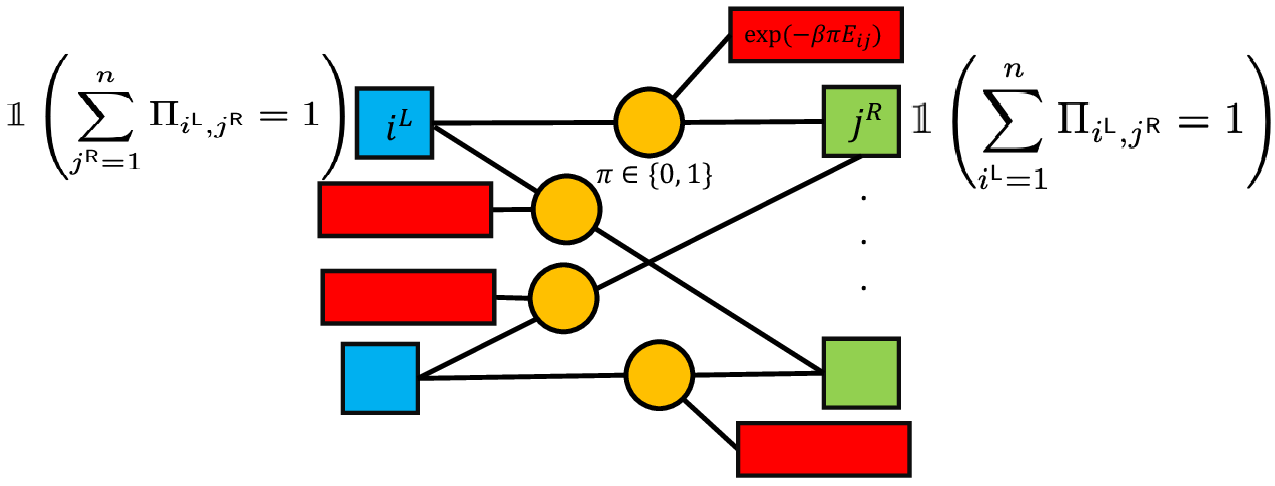}	


\caption{The constructed graphical model.
The circle icons represent the  variable nodes and  the square icons
represent the function nodes: a blue square  for  the constraint on the rows of $\bPi$, a green square   for the constraint on the columns of $\bPi$, and a red square  for the function $e^{-\beta \pi \bE_{ij}}$.}\label{fig:whole_graph}
\end{figure}

Now we briefly review the MP algorithm. Informally speaking, MP is a local algorithm to
compute the
marginal probabilities over the graphical model.
In each iteration, the variable node $v$ transmits the message to its
incident function node $f$ by multiplying all incoming messages
except the message along the edge $(v, f)$.  The function node $f$ transmits the message to its incident
variable node $v$ by computing the weighted summary of all incoming
messages except the message along the edge $(f, v)$. For a detailed introduction  to MP, we refer readers to~\citet{kschischang2001factor}, Chapter~$16$ in~\citet{mackay2003information}, and Chapter~$14$ in~\citet{mezard2009information}.

It is known that MP can obtain the exact marginals
\citep{mezard2009information} for singly connected graphical models.
For other types of graphs, however, whether MP can obtain the
exact solution still remains an open problem~\citep{cantwell2019, kirkley2021}.
At the same time,   numerical evidences have been witnessed to show that  MP can yield meaningful results for graphs
with loops; particular examples include  applications in the coding theory~\citep{chung2000construction, richardson2001capacity, richardson2008modern}
and the LAP (which happens to be our case)~\citep{mezard2009information, chertkov2010inference, caracciolo2017finite, malatesta2019fluctuations, semerjian2020recovery}.

\subsection{The message passing (MP) algorithm}
Next, we perform permutation recovery via MP.
The following derivation follows the standard
procedure, which can be found in the previous works~\citep{mezard2009information,semerjian2020recovery}.
We denote the message flow from the node
$\ln{i}$ to the variable node $(\ln{i}, \rn{j})$ as
$\cvinfo{\ln{i}}{(\ln{i}, \rn{j})}(\cdot)$ and
that from the edge $(\ln{i}, \rn{j})$
to node $\ln{i}$ as
$\vcinfo{(\ln{i}, \rn{j})}{\ln{i}}(\cdot)$.
Similarly, we define
$\cvinfo{\rn{j}}{(\ln{i}, \rn{j})}(\cdot)$
and $\vcinfo{(\ln{i}, \rn{j})}{\rn{j}}(\cdot)$
as the message flow transmitted
between the functional node $\rn{j}$ and
the variable node $\bracket{\ln{i}, \rn{j}}$.
Here the superscripts $\mathsf{L}$ and $\mathsf{R}$ are used to
indicate the positions of the node (left and right, respectively). Roughly speaking, these
transmitted messages can be viewed as (unnormalized) conditional probability
$\Prob(\Pi_{i, j} = \{0, 1\}|(\cdot))$ with the joint PDF being defined in
\eqref{eq:prob_measure}.
The message transmission process
is to iteratively compute
these conditional probabilities.

\vspace{0.1in}

First, we consider the message flows
transmitted between the
functional node $\ln{i}$ and the variable
node $\bracket{\ln{i}, \rn{j}}$, which are
written as
\begin{align}
\label{eq:mp_left_wise}
\vcinfo{(\ln{i}, \rn{j})}{\ln{i}}(\pi)  & \cong \cvinfo{\rn{j}}{(\ln{i}, \rn{j})}(\pi) e^{-\beta \pi \bE_{\ln{i}, \rn{j}}}  , \notag \\\notag \\
\cvinfo{\ln{i}}{(\ln{i}, \rn{j})}(\pi)
 & \cong \sum_{\pi_{\ln{i}, \rn{k}}} \prod_{\rn{k}\neq \rn{j}}
 \cvinfo{\rn{k}}{(\ln{i}, \rn{k})}(\pi_{\ln{i}, \rn{k}}) \cdot e^{-\beta\pi_{\ln{i}, \rn{k}} \bE_{\ln{i}, \rn{k}}}
 \Ind(\pi + \sum_k \pi_{\ln{i}, \rn{k}} = 1),
\end{align}
where $\pi \in \{0, 1\}$ is a binary value.
Similarly, we can write the message flows
between the functional node $\rn{j}$ and
the variable node $\bracket{\ln{i}, \rn{j}}$, which are denoted as $\vcinfo{(\ln{i}, \rn{j})}{\rn{j}}(\pi)$ and
$\cvinfo{\rn{j}}{(\ln{i}, \rn{j})}(\pi)$, respectively.
With the parametrization approach, we define
\[
\lloginfo{i}{j} \defequal \frac{1}{\beta}\log \frac{\cvinfo{\ln{i}}{(\ln{i}, \rn{j})}(1)}{\cvinfo{\ln{i}}{(\ln{i}, \rn{j})}(0)}, ~~~~~~ \rloginfo{i}{j} \defequal \frac{1}{\beta}\log \frac{\cvinfo{\rn{j}}{(\ln{i}, \rn{j})}(1)}{\cvinfo{\rn{j}}{(\ln{i}, \rn{j})}(0)}.
\]
Denote $\zeta_{\ln{i}, \rn{j}}$ as
$ \lloginfo{i}{j} + \rloginfo{i}{j} -  \bE_{\ln{i}, \rn{j}}$. We select the edge
$\bracket{\ln{i}, \rn{j}}$ according to the probability
$m_{(\ln{i}, \rn{j})}(\pi) \defequal
\frac{\exp(\pi\cdot \beta\zeta_{\ln{i}, \rn{j}}) }{1 + \exp(\beta  \zeta_{\ln{i}, \rn{j}} )},~\pi \in \set{0, 1}$.
Provided $m_{(\ln{i}, \rn{j})}(1) > m_{(\ln{i}, \rn{j})}(0)$,
or equivalently,
\begin{align}
\label{eq:mp_edge_select}
\zeta_{\ln{i}, \rn{j}}> 0,
\end{align}
we pick $\wh{\pi}(\ln{i}) = \rn{j}$;
otherwise, we have $\wh{\pi}(\ln{i}) \neq \rn{j}$.
Due to the fact that $\mu(\bPi)$ concentrates on
$\wh{\bPi}$ when $\beta$ is sufficiently large, we can
thus rewrite the MP update equation as
\begin{align}
\label{eq:mp_loginfo_beta_inf}
\lloginfo{i}{j} =  \min_{\rn{k}\neq \rn{j}}
\bE_{\ln{i}, \rn{k}} - \rloginfo{i}{k}, ~~~~~
\rloginfo{i}{j} = \min_{\ln{k}\neq \ln{i}}
\bE_{\ln{k}, \rn{j}} - \lloginfo{k}{j},
\end{align}
which is attained by letting $\beta\rightarrow \infty$.

\subsection{Identification of the phase transition threshold}
\label{sec:mp_density_evolve}

To identify the phase transition phenomenon inherent in the
MP update equation~\eqref{eq:mp_loginfo_beta_inf},
we follow the  strategy  in~\citet{semerjian2020recovery}
and divide all edges $\bracket{\ln{i}, \rn{j}}$ into two
categories according to
whether the edge $\bracket{\ln{i}, \rn{j}}$ corresponds to the ground-truth permutation matrix $\bPi^{\natural}$ or not.
Within each category, we assume the edges' weights and
the message flows along them can be represented by independently identically distributed random variables.
\par
For the edge $(\ln{i}, \pi^{\natural}(\ln{i}))$ for the ground-truth correspondence, we represent
the random variable associated with the
weight $\bE_{ij}$ as $\Omega$. The random variable for the message flow along this edge is denoted  $H$ (for both $\lloginfo{i}{j}$ and $\rloginfo{i}{j}$).
For the rest of edges $(\ln{i}, \rn{j})$ ($\rn{j}\neq \pi^{\natural}(\ln{i})$), we define the corresponding random variables for the edge weight and message flow as $\wh{\Omega}$ and $\wh{H}$, respectively. Then, we can rewrite ~\eqref{eq:mp_loginfo_beta_inf} as
\begin{align}
\wh{H}^{(t+1)} = \min\bracket{\Omega - H^{(t)}, H^{'(t)}}, ~~~~H^{(t+1)} = \min_{1\leq i \leq n-1} \wh{\Omega}_i - \wh{H}^{(t)}_i,
\label{eq:succ_mp_single}
\end{align}
where $(\cdot)^{(t)}$ denotes the update in the $t$-th iteration,
$H^{'}$ is an independent copy of $H$,
 $\{H_i^{(t)}\}_{1\leq i \leq n-1}$ and
$\{\wh{\Omega}_i\}_{1\leq i \leq n-1}$ denote
the i.i.d. copies of random variables
$H_{(\cdot)}^{(t)}$ and $\wh{\Omega}_{(\cdot)}$.
This equation \eqref{eq:succ_mp_single}
can be viewed as the analogous version
of the \emph{density evolution} and
\emph{state evolution},
which are used to analyze the convergence of
the message passing and approximate message
passing algorithm, respectively
\citep{chung2000construction, richardson2001capacity, richardson2008modern,  donoho2009message, maleki2010approximate, bayati2011dynamics, rangan2011generalized}.

\vspace{0.1in}

\begin{remark}
We conjecture that the distribution difference in the
edges' weights is a necessary component in capturing the phase
transition.
On one hand, according to~\citet{mezard1986mean,mezard1987solution,  parisi2002finite, linusson2004proof,mezard2009information, talagrand2010mean}, there is no phase transition phenomenon in LAP
if the edges' weights, i.e., $\bE_{ij}$,
are assumed to be i.i.d uniformly distributed
in $[0, 1]$. On the other hand,~\citet{semerjian2020recovery} show a phase transition phenomenon when assuming the weights $\bE_{ij}$
follow different distributions among the edges associated with the
ground-truth correspondence $\bracket{\ln{i}, \pi^{\natural}(\ln{i})}$
and the rest edges.
\end{remark}

\vspace{0.1in}
\noindent
\textbf{Relation to \emph{branching random walk} (BRW) process.}
Conditional on the event  that the permutation
can be perfectly reconstructed, i.e.,
$H + H^{'} > \Omega$ as in~\eqref{eq:mp_edge_select}, we can simplify \eqref{eq:succ_mp_single} as
\begin{align}
\label{eq:brw_def}
H^{(t+1)} = \min_{1\leq i \leq n-1} H_i^{(t)} + \Xi_i,
\end{align}
where  $\Xi$ is defined as the difference between
$\wh{\Omega}$ and $\Omega$, which is written as
$\Xi\defequal \wh{\Omega} - \Omega$,
and $\{H_i^{(t)}\}_{1\leq i \leq n-1}$ and $\{\Xi_i\}_{1\leq i \leq n-1}$ denote the
 i.i.d. copies of random variables $H^{(t)}_{(\cdot)}$ and
$\Xi_{(\cdot)}$.

Adopting the same viewpoint
of~\citet{semerjian2020recovery}, we treat~\eqref{eq:brw_def} as a
\emph{branching random walk} (BRW) process,
which enjoys the following property.

\vspace{0.1in}
\begin{theorem}[\citet{hammersley1974postulates, kingman1975first, semerjian2020recovery}]
\label{thm:brw}
Consider the \emph{recursive distributional equation}
$K^{(t+1)} = \min_{1\leq i\leq n} K_i^{(t)} + \Xi_i$, where
$K_i^{(t)}$ and $\Xi_i$ are i.i.d copies
of random variables $K_{(\cdot)}^{(t)}$ and $\Xi_{(\cdot)}$,
we have $\frac{K^{(t+1)}}{t}  \asconverg -
\inf_{\theta > 0} \frac{1}{\theta}
\log\Bracket{\sum_{i=1}^n \Expc e^{-\theta \Xi_i}}$,
conditional on the event that  $\lim_{t\rightarrow \infty}K^{(t)}\neq \infty$.
\end{theorem}

\newpage

With Theorem~\ref{thm:brw}, we can compute phase transition point
for the correct (full) permutation recovery, i.e.,
$H + H^{'} > \Omega$, by letting
$\inf_{\theta > 0} \frac{1}{\theta}
\log\Bracket{\sum_{i=1}^n \Expc e^{-\theta \Xi_i}}= 0$,
since otherwise the condition in ~\eqref{eq:mp_edge_select} will be violated.
In practice, directly computing the infimum
of $\inf_{\theta > 0} \frac{1}{\theta}
\log\Bracket{\sum_{i=1}^n \Expc e^{-\theta \Xi_i}}$
is only possible for limited scenarios.
In the next section, we propose
an approximate computation method for the phase transition points,
which is capable of covering a broader class of scenarios.

\vspace{0.1in}
\begin{remark}
This paper considers the phase transition phenomenon w.r.t. the full permutation recovery.
Informally speaking, this can be partly deduced from \eqref{eq:succ_mp_single} and \eqref{eq:brw_def}.

Here, we regard message flows $\lloginfo{i}{j}$ and
 $\rloginfo{i}{j}$ i.i.d. samples from certain distributions (represented by the random variable $H$).
When studying the evolution behavior of the random variable
$H$, we track the behaviors of all message flows. Hence, if we find an arbitrary sample $H$
 that will yield the correct recovery, we can say that the correspondence between all pairs is correct. On the other hand, we can say that there exist some pairs with wrong correspondence if $H$ leads to incorrect recovery. This can explain why the phase transition phenomenon exists.
\end{remark}

\section{Analysis of the Phase Transition Points}\label{sec:oracle_case}

Recall that, in this paper, we consider the following linear regression problem with permuted labels
\begin{align}\notag
\bY = \bPi^{\natural}\bX\bB^{\natural} + \sigma \bW,
\end{align}
where $\bY\in \RR^{n\times m}$ represents the matrix of observations,
$\bPi^{\natural} \in \calP_n$ denotes the permutation matrix  to be
reconstructed, $\bX\in \RR^{n\times p}$ is the sensing matrix with each entry
$\bX_{ij}$ following the i.i.d standard normal distribution,
$\bB^{\natural}\in \RR^{p\times m}$ is the matrix of signals,
and $\bW\in \RR^{n\times m}$ represents the additive noise matrix
and its entries $\bW_{ij}$ are i.i.d standard normal random variables.
In addition, we denote $h$ as the number
of permuted rows corresponding to the permutation matrix
$\bPi^{\natural}$.

In this work, we  focus on studying the ``phase transition'' phenomenon in recovering  $\bPi^{\natural}$ from the pair $\bracket{\bY, \bX}$. That is,  the error rate for the permutation recovery sharply drops to zero once certain parameters reach the thresholds. In particular, our analysis will
identify the precise positions of the phase transition points in the large-system limit, i.e.,
$n$, $m$, $p$, and $h$
all approach to infinity with
$m/n\rightarrow \tau_m$, $p/n\rightarrow\tau_p$, $h/n\rightarrow \tau_h$.  We will separately study
the phase transition phenomenon in 1) the oracle case where $\bB^{\natural}$ is given as a prior,
and 2)  the non-oracle case where $\bB^{\natural}$ is unknown.

\vspace{0.1in}
In this section, we consider the oracle scenario, as a warm-up example.
To reconstruct the permutation matrix $\bPi^{\natural}$,
we adopt the following \emph{maximum-likelihood} (ML) estimator:
\begin{align}
\label{eq:oracle_optim}
\wh{\bPi}^{\oracle}& = \argmin_{\bPi} \la \bPi, -\bY \bB^{\natural \rmt} \bX^{\rmt}\ra,~\St~\sum_i \bPi_{ij} = 1, \sum_j \bPi_{ij} = 1, \bPi \in \set{0, 1}^{n\times n}.
\end{align}
Denoting the variable $\bE^{\oracle}_{ij}$ as
$-\bX_{\pi^{\natural}(i)}^{\rmt}\bB^{\natural} \bB^{\natural \rmt}\bX_j - \
\sigma \bW_{i}^{\rmt} \bB^{\natural \rmt} \bX_j$,~($1\leq i,j \leq n$),
we can transform the objective function in
~\eqref{eq:oracle_optim} as
the canonical form
of LAP, i.e.,
$\sum_{i,j}\bPi_{ij}\bE^{\oracle}_{ij}$.

\subsection{The phase transition threshold for the oracle case}\label{subsec:snr_phase_transition_oracle}

In the oracle case where $\bB^{\natural}$ is known, we define the following random variable $\Xi$:
\begin{align}
\label{eq:oracle_xi_def}
\Xi = \bx^{\rmt}
\bB^{\natural}\bB^{\natural\rmt}\bracket{ \bx - \by}
+ \sigma \bw \bB^{\natural \rmt}\bracket{\bx - \by},
\end{align}
where $\bx$ and $\by$ follow the distribution $\normdist(\bZero, \bI_{p\times p})$,
and $\bw$ follows the distribution $\normdist(\bZero, \bI_{m\times m})$.

Recalling Theorem~\ref{thm:brw}, we predict the phase transition
point by letting
\begin{align}
\label{eq:snr_oracle_phase_transition_point}
\inf_{\theta > 0} \nfrac{1}{\theta}\cdot \log \bracket{\sum_{i=1}^n \Expc e^{-\theta \Xi_i}} =
\inf_{\theta > 0} \nfrac{1}{\theta}\cdot \bracket{\log n + \log\Expc e^{-\theta \Xi}} = 0.
\end{align}
The computation procedure consists of two stages:

\begin{itemize}[leftmargin=*]

\item
\textbf{Stage I.}
We compute the optimal $\theta_*$, which is written as
$
\theta_* = \argmin_{\theta > 0} \nfrac{1}{\theta}\cdot \bracket{\log n + \log\Expc e^{-\theta \Xi_i}}.
$

\item
\textbf{Stage II.}
We plug the optimal $\theta^*$ into \eqref{eq:snr_oracle_phase_transition_point}
and obtain the phase transition $\snr$ accordingly.
\end{itemize}
The following context illustrates the computation details.

\paragraph{Stage I: Determine $\theta_*$.} The key in determining
$\theta_*$ lies in the computation of $\Expc e^{-\theta\Xi}$, which
is summarized in the following proposition.

\begin{proposition}
\label{prop:oracle_analysis}
For the random variable $\Xi$ defined in~\eqref{eq:oracle_xi_def},
we can write its expectation as
\begin{align}
\label{eq:oracle_expc_chernoff}
\Expc e^{-\theta \Xi}=
\prod_{i=1}^{\rank(\bBtrue)}
\Bracket{1+ 2\theta \lambda_i^2 - \theta^2 \lambda_i^2\bracket{\lambda_i^2 + 2\sigma^2}}^{-\frac{1}{2}},
\end{align}
provided that
\vspace{-1mm}
\begin{align}
\label{eq:oracle_integral_condition}
\theta^2 \sigma^2 \lambda_i^2 <~& 1 \text{ and}~~
\theta^2\lambda_i^2\bracket{\lambda_i^2 + 2\sigma^2} \leq 1 + 2\theta \lambda_i^2
\end{align}
hold for all singular values $\lambda_i$ of $\bBtrue{}$, $1\leq i \leq \rank(\bBtrue)$.
\end{proposition}

\begin{proof}
Denote the singular values of $\bBtrue$ as
$\set{\lambda_i}_{i=1}^{\rank(\bBtrue)}$.
We exploit the rotation invariance property
of Gaussian random variables; and
have $\bXi$ be identically distributed as
\[
\bXi =
\sum_{i=1}^{\rank(\bBtrue)} \lambda_i^2 x_i\bracket{x_i - y_i}
+ \sigma \sum_{i=1}^{\rank(\bBtrue)} \lambda_i w_i \bracket{x_i - y_i}.
\]
Due to the independence across $\bw,~\bx$, and $\by$,
we have
\[
\Expc e^{-\theta \Xi} =~& \prod_{i=1}^{\rank(\bBtrue)}
\Expc_{x, y, w}
\exp\Bracket{
-\theta\lambda_i^2 x\bracket{x - y}
-\theta \sigma  \lambda_i w \bracket{x - y}} \\
=~& \prod_{i=1}^{\rank(\bBtrue)}
\Expc_{x, y}
\exp \left(\frac{\theta  \lambda_i ^2 (x-y) \left(\theta  \sigma ^2 (x-y)-2 x\right)}{2}\right) \\
\stackrel{\cirone}{=}~&
\prod_{i=1}^{\rank(\bBtrue)}\Expc_{x}\frac{\exp \left(\frac{\theta  \lambda_i^2 x^2 \left(\theta  \left(\lambda_i^2+\sigma ^2\right)-2\right)}{2-2 \theta ^2 \lambda_i^2 \sigma ^2}\right)}{\sqrt{1-\theta ^2 \lambda_i^2 \sigma ^2}} \\
\stackrel{\cirtwo}{=}~&
\prod_{i=1}^{\rank(\bBtrue)}
\bracket{1+ 2\theta \lambda_i^2 - \theta^2 \lambda_i^2\bracket{\lambda_i^2 + 2\sigma^2}}^{-\frac{1}{2}},
\]
where in $\cirone$ we use the fact
$\theta^2\sigma^2 \lambda_i^2 < 1$ and
in $\cirtwo$ we use the fact
$\theta^2\lambda_i^2\bracket{\lambda_i^2 + 2\sigma^2} \leq 1 + 2\theta \lambda_i^2$.	
\end{proof}

\begin{remark}
When the conditions in ~\eqref{eq:oracle_integral_condition} is violated,
we have the expectation $\Expc e^{-\theta \Xi}$ diverge to
infinity, which suggests the optimal $\theta_{*}$ for $\inf_{\theta > 0}\nfrac{\log\bracket{n\cdot \Expc e^{-\theta \Xi}}}{\theta}$ cannot be achieved.
\end{remark}

With \eqref{eq:oracle_expc_chernoff}, we can compute
the optimal $\theta_*$ by setting the gradient
$\frac{\partial\Bracket{\nfrac{\log (n \cdot \Expc e^{-\theta \Xi)}}{\theta}}}{\partial \theta} = 0$.
However, a closed-form of the exact solution for
$\theta^*$ is out of reach. As a mitigation, we resort to approximating $\log \Expc e^{-\theta \Xi}$ by its
lower-bound, which reads as
\[
\log\Expc e^{-\theta \Xi}
\geq~&
\frac{\theta^2}{2}\Bracket{\Fnorm{\bB^{\natural\rmt}\bBtrue}^2 + 2\sigma^2 \Fnorm{\bBtrue}^2}
- \theta \Fnorm{\bBtrue}^2.
\]
The corresponding minimum value $\wt{\theta}_*$
is thus obtained by minimizing the lower-bound, which
is written as
$\wt{\theta}_* = {{2\log n}/\bracket{\Fnorm{\bB^{\natural\rmt}\bB^{\natural}}^2 + 2\sigma^2\Fnorm{\bB^{\natural}}^2 }}$.


\paragraph{Stage II: Compute the phase transition $\snr$.}
We predict the phase transition point $\snr_{\textup{oracle}}$
by letting the lower bound being zero, which can be written as
\[
\frac{\log n}{\theta^*} - \Fnorm{\bBtrue}^2 +
\frac{\theta^*}{2} \bracket{\Fnorm{\bB^{\natural \rmt} \bBtrue}^2 + 2\sigma^2 \Fnorm{\bBtrue}^2}
= 0.
\]
With standard algebraic manipulations, we obtain the equation
\begin{align}
\label{eq:snr_oracle_phase_transition}
2(\log n)\snr_{\textup{oracle}}\cdot
\Fnorm{\nfrac{\bB^{\natural\rmt}}{\Fnorm{\bBtrue}}\cdot  \nfrac{\bBtrue}{\Fnorm{\bBtrue}}}^2
+ \nfrac{4\log n}{m} = \snr_{\textup{oracle}}.
\end{align}

To evaluate the accuracy of our predicted phase
transition threshold, we compare the predicted values with the numerical values.
The results are shown in Table~\ref{tab:snr_threshold}, from which we can conclude the phase transition threshold $\snr$ can be predicted to a good extent. In addition, we observe that the gap between the theoretical values and the numerical values keeps shrinking as $m$ increases.

\begin{table}[!h]
\centering
\caption{
Comparison between the predicted value
of the phase transition threshold $\snr_{\oracle}$ and
its simulated value when $n = 500$.
\textbf{P} denotes the predicted value while
\textbf{S}  denotes the simulated value (i.e., $\textup{mean}\pm \textup{std}$).
\textbf{S}  corresponds to
the $\snr$ when the error rate drops below $0.05$. A detailed description of the numerical method can be found in the Appendix.
\vspace{0.1in}}
\label{tab:snr_threshold}
\begin{tabular}{@{}c|cccccc@{}}\toprule
$m$  & $20$ & $30$ & $40$ & $50$ & $60$
&  $70$   \\
\midrule
\textbf{P} & $3.283$ & $1.415$ & $0.902$ & $0.662$ & $0.523$
&   $0.432$  \\
\textbf{S} & $2.529\pm 0.079$ & $1.290\pm 0.054$ & $0.872\pm 0.034$ & $0.649\pm 0.012$ & $0.515 \pm 0.016$ &
$0.429 \pm 0.015$ \\
\bottomrule
\end{tabular}
\vspace{0.1in}
\begin{tabular}{@{}c|cccccc@{}}\toprule
$m$  & $100$ & $110$ & $120$ & $130$ & $140$
&  $150$   \\
\midrule
\textbf{P} & $0.284$ & $0.255$ & $0.231$ & $0.211$ & $0.195$
&   $0.181$  \\
\textbf{S} & $0.282\pm 0.008$ & $0.256\pm 0.006$ & $0.232\pm 0.006$ & $0.212\pm 0.004$ & $0.196 \pm 0.006$ &
$0.183 \pm 0.005$ \\
\bottomrule
\end{tabular}
\end{table}


\subsection{Gaussian approximation of the phase transition threshold}
\label{subsec:gauss_approx_oracle}
From the above analysis, we can see that deriving
a closed-form expression of the infimum value
$\theta$ of $\nfrac{\log(n\Expc e^{-\theta \Xi})}{\theta}$ can be   difficult. In fact, in certain scenarios, even obtaining a
closed-form expression of $\Expc e^{-\theta \Xi}$
is difficult.
To handle such challenge, we propose to approximate
random variable $\Xi$ by a Gaussian  $\normdist(\Expc \Xi, \Var \Xi)$,
namely,
\begin{align}\label{eq:taylor_approx}
\Expc e^{-\theta \Xi} \approx
\exp\bracket{-\theta \Expc \Xi +
\frac{\theta^2}{2}\Var\Xi}.
\end{align}
With this approximation, we can express
$\theta_* \defequal \inf \nfrac{\log\bracket{n\cdot \Expc e^{-\theta \Xi}}}{\theta}$ in a closed form, which is
$\sqrt{\nfrac{2\log n}{\Var\Xi}}$. Thus, the critical
point corresponding to the phase transition in \eqref{eq:snr_oracle_phase_transition_point} is written as
\begin{align}
\label{eq:taylor_phase_transition_def}
2(\log n)\cdot \Var\Xi = \bracket{\Expc \Xi}^2.
\end{align}
\paragraph{Comparison with \eqref{eq:snr_oracle_phase_transition}.} To verify that this approximation can yield meaningful results,
we revisit the oracle case and have
\begin{align}
\label{eq:taylor_phase_transition_expc_var_def}
\Expc \Xi =  \Fnorm{\bB^{\natural}}^2, ~~~\Var \Xi = 3\Fnorm{\bB^{\natural}\bB^{\natural \rmt}}^2
+ 2\sigma^2 \Fnorm{\bB^{\natural}}^2.
\end{align}
Plugging \eqref{eq:taylor_phase_transition_expc_var_def}
into \eqref{eq:taylor_phase_transition_def}
then yields the relation
\begin{align}
\label{eq:snr_oracle_phase_transition_approx}
6(\log n) \snr_{\textup{oracle}}\cdot
\Fnorm{\nfrac{\bB^{\natural\rmt}}{\Fnorm{\bBtrue}}\cdot  \nfrac{\bBtrue}{\Fnorm{\bBtrue}}}^2
+ \nfrac{4 \log n}{m}
= \snr_{\textup{oracle}},
\end{align}
from which we can determine the critical
point of $\snr$.

\begin{example}[Scaled identity matrix] We consider the scenario
where
$\bB^{\natural}=\lambda\bI_{m\times m}$. Then, we have
$\nfrac{\bBtrue}{\Fnorm{\bBtrue}} = m^{-1/2}\bI$.
The phase transition threshold $\snr_{\-{oracle}}$
in~\eqref{eq:snr_oracle_phase_transition} is
then  $\nfrac{4\log n}{(m-2\log n)}$, and  the
phase transition threshold $\wt{\snr}_{\oracle}$
in~\eqref{eq:snr_oracle_phase_transition_approx}  as $\nfrac{4\log n}{(m-6\log n)}$.
This solution is almost identical to \eqref{eq:snr_oracle_phase_transition} in the
 limit as
$\snr_{\oracle} \approx \wt{\snr}_{\oracle} \approx
\nfrac{4\log n}{m} \simeq n^{\frac{4}{m}} - 1$.
\end{example}

Moreover, we should mention that $1)$
our approximation method applies to a
general matrix $\bBtrue$, not limited to a
scaled identity matrix;
and $2)$ our approximation method can also predict the phase transition thresholds to a good extent when the entries $\bX_{ij}$ are sub-Gaussian.
The numerical experiments are given in Table \ref{tab:snr_subgauss_threshold}, from which
we can conclude that the predicted values are well aligned with the simulation results.
The numerical method to compute the phase transition points can be found in the Appendix.

\begin{table}[h]
\centering
\caption{
Comparison between the predicted value
of the phase transition threshold $\wt{\snr}_{\oracle}$ and
its simulated value when $n = 600$.
In (\textbf{Case 1}), half of singular values are with $\lambda$ and
the other half are with $\nfrac{\lambda}{2}$; while in (\textbf{Case 2}),
half of the singular values
are with $\lambda$ and the other half are with
$\nfrac{(3\cdot \lambda)}{4}$.
\textbf{Gauss} refers
to $\bX_{ij}\iid \normdist(0, 1)$ and
\textbf{Unif} refers to $\bX_{ij}\iid \textup{Unif}[-1, 1]$.
\textbf{P} denotes the predicted value and
\textbf{S}  denotes the simulated value (i.e., $\textup{mean}\pm \textup{std}$).
\textbf{S} corresponds to
the $\snr$ when the error rate drops below $0.05$.
Averaged over $20$ repetitions.
\vspace{0.1in}}
\label{tab:snr_subgauss_threshold}
\resizebox{6.5in}{!}{
\begin{tabular}{@{}l||cccccc@{}}\toprule
$m$  & $100$ &  $110$ & $120$ & $130$ & $140$ & $150$\\
\midrule
(\textbf{Case 1}) ~\textbf{P} & $0.297$ & $0.266$ & $0.241$ & $0.220$ & $0.203$ & $0.188$  \\
(\textbf{Gauss})~~  \textbf{S} & $0.307\pm 0.009$ & $0.275\pm 0.005$ & $0.246\pm 0.006$ & $0.227\pm 0.007$ & $0.210 \pm 0.005$ & $0.194 \pm 0.004$ \\
(\textbf{Unif})~~~~~ \textbf{S} & $0.294\pm 0.008$ & $0.266\pm 0.005$ & $0.239\pm 0.008$ & $0.216\pm 0.004$ & $0.201\pm 0.005$ & $0.189 \pm 0.006$ \\ \\
(\textbf{Case 2}) ~\textbf{P} & $0.310$ & $0.276$ &  $0.249$ & $0.227$ & $0.209$ & $0.193$\\
(\textbf{Gauss}) ~~\textbf{S} & $0.294\pm 0.008$ & $0.266\pm 0.006$ & $0.241 \pm 0.005$ & $0.220 \pm 0.004$ & $0.204\pm 0.006$ & $0.190\pm 0.003$ \\
(\textbf{Unif})~~~~~ \textbf{S}  & $0.287\pm 0.007$ & $0.255\pm .0043$ & $0.234 \pm 0.007$ & $0.213\pm 0.005$ & $0.197 \pm 0.003$
& $0.185 \pm 0.005$ \\
\bottomrule
\end{tabular}
}
\end{table}

\newpage

\section{Extension to Non-Oracle Case}
\label{sec:non_oracle_case_intro}

Having analyzed the oracle case in the previous section,
we now extend the analysis to the non-oracle case, where
the value of $\bB^{\natural}$ is not given.
Different from the oracle case,
the ML estimator reduces to a
\emph{quadratic assignment problem} (QAP) as opposed to LAP.
As a mitigation, we adopt the estimator in~\citet{zhang2020optimal},
which reconstructs the permutation matrix within the LAP framework, i.e.,
\begin{align}
\label{eq:nonoracle_optim}
\wh{\bPi}^{\oracle}& = \argmin_{\bPi} \la \bPi, -\bY \bY^{\rmt} \bX \bX^{\rmt}\ra, ~\St~~\sum_i \bPi_{ij} = 1, \sum_j \bPi_{ij} = 1, \bPi \in \set{0, 1}^{n\times n}.
\end{align}
We expect this estimator can yield good
insights of the permuted linear regression since
\begin{itemize}[leftmargin=*]
\item  this estimator can reach the statistical optimality	in a broad range of parameters;

\item this estimator exhibits a phase transition
phenomenon, which follows a similar pattern to that in the oracle case.
\end{itemize}
Following the same procedure as in Section~\ref{sec:oracle_case}, we identify the
phase transition threshold $\snr$ with Theorem~\ref{thm:brw}. To begin with, we write the random variable $\Xi$ as
\[
\Xi \cong \bY_i \bY^{\rmt}\bX \bracket{\bX_{\pi^{\natural}(i)} - \bX_j}^{\rmt},
\]
where $i$ and $j$ are uniformly distributed among the set
$\set{1, 2,\cdots, n}$. Afterwards, we adopt the Gaussian approximation scheme illustrated in Subsection~\ref{subsec:gauss_approx_oracle} and
determine the phase transition points by first computing
$\Expc \Xi$ and $\Var\Xi$, respectively.
\begin{theorem}
\label{thm:nonoracle_mean_var}
For the random variable $\Xi$ defined
in~\eqref{eq:xi_decomposition},
its mean $\Expc \Xi$ and
variance $\Var \Xi$ are
\[
\Expc \Xi
\simeq~& n\bracket{1-\tau_h}
\Bracket{
\bracket{1 + \tau_p}
\Fnorm{\bB^{\natural}}^2 +
m \tau_p\sigma^2}, \\
\Var \Xi \simeq ~&
n^2\tau_h \bracket{1-\tau_h}\tau_p^2
\Bracket{\Fnorm{\bB^{\natural}}^2 + m\sigma^2}^2
+ n^2\Bracket{2\tau_p + 3\bracket{1-\tau_h}^2}\Fnorm{\bB^{\natural \rmt}\bB^{\natural}}^2 \\
+~& n^2\Bracket{ 6\tau_p \bracket{1-\tau_h}^2 +
\bracket{3-\tau_h}\tau_p^2
}\Fnorm{\bB^{\natural \rmt}\bB^{\natural}}^2,
\]
\noindent respectively,
where the definitions of
$\tau_p$ and $\tau_h$ can be found
in Section~\ref{sec:oracle_case}.
\end{theorem}

The proof of Theorem~\ref{thm:nonoracle_mean_var} is quite complicated, which is a combination of Wick's theorem, Stein's lemma, the conditional technique, and the leave-one-out technique, etc. For the presentation conciseness, we only give an outline of the proof and defer the technical details to Appendix.

\subsection{Proof outline}

To begin with, we decompose the random variable $\Xi$ as
\begin{align}
\label{eq:xi_decomposition}
\Xi
=~& \Xi_1 + \sigma\bracket{\Xi_2 + \Xi_3}
+ \sigma^2 \Xi_4,
\end{align}
where $\Xi_i$ $(1\leq i\leq 4)$ are respectively defined as
\[
\Xi_1 \defequal ~& \bX_{\pi^{\natural}(i)}^{\rmt}\bB^{\natural}\bB^{\natural\rmt}\bX^{\rmt}\bPi^{\natural \rmt} \bX
(\bX_{\pi^{\natural}(i)} - \bX_j), \\
\Xi_2 \defequal~& \bX_{\pi^{\natural}(i)}^{\rmt}\bB^{\natural} \bW^{\rmt}\bX(\bX_{\pi^{\natural}(i)} - \bX_j), \\
\Xi_3 \defequal ~& \bW_i^{\rmt} \bB^{\natural\rmt}\bX^{\rmt}\bPi^{\natural \rmt} \bX (\bX_{\pi^{\natural}(i)} - \bX_j), \\
\Xi_4 \defequal~& \bW_i^{\rmt} \bW^{\rmt}\bX (\bX_{\pi^{\natural}(i)} - \bX_j).
\]

Unlike the oracle case,
obtaining a closed-form expression of $\Expc e^{-\theta \Xi}$ would be too difficult.
Hence, we adopt the Gaussian approximation method
as presented in Section~\ref{subsec:gauss_approx_oracle}.
The task then transforms to computing the
expectation and variance of $\Xi$.

\paragraph{Computation of the mean $\Expc \Xi$.}
For the computation of the mean $\Expc \Xi$,  we can
verify that $\Expc \Xi_2$ and $\Expc \Xi_3$ are both zero, due to the independence between $\bX$ and $\bW$.
For $\Expc \Xi_1$ and $\Expc \Xi_4$,
we adopt Wick's theorem \citep{janson_1997} to obtain
\[
\Expc \Xi_1 =~& n\bracket{1-\tau_h}\bracket{1 + \tau_p}\Bracket{1 + \oprate{1}}\Fnorm{\bB^{\natural}}^2,  \\
\Expc \Xi_4 =~& n m \tau_p \bracket{1-\tau_h}[1 + \oprate{1}].
\vspace{-2mm}
\]

\paragraph{Computation of the variance $\Var\Xi$.}
Since
$\Var \Xi = \Expc \Xi^2 - (\Expc \Xi)^2$, we just need to
compute $\Expc \Xi^2$, which can be expanded into the following six terms
\begin{align}
\label{eq:non_oracle_var}
\Expc \Xi^2 =~&\Expc \Xi_1^2 + \sigma^2 \Expc \Xi_2^2 + \sigma^2\Expc \Xi_3^2
+ \sigma^4\Expc \Xi_4^2 + 2\sigma^2 \Expc \Xi_1 \Xi_4 + 2\sigma^2\Expc \Xi_2 \Xi_3.
\vspace{-2mm}
\end{align}
The computation of the above terms turns out to be quite complex due to the
high-order Gaussian random variables. For example, the term $\Expc \Xi_1^2$
involves the eighth-order Gaussian moments,
the terms $\Expc \Xi_2^2, \Expc \Xi_3^2, \Expc \Xi_1 \Xi_4$ and
$\Expc \Xi_2\Xi_3$ all involve the sixth-order Gaussian variables, etc.
To handle the difficulties in computing $\Expc \Xi^2$,
we propose the following computation procedure, which
can be roughly divided into $3$ phases.
\begin{itemize}[leftmargin=*]
\item
\textbf{Phase I: Leave-one-out decomposition.}
The major technical difficulty comes from the
correlation between the product $\bX^{\rmt}\bPi^{\natural}\bX$ and
the difference $\bX_{\pi^{\natural}(i)} - \bX_j$.
We decouple this correlation by first
rewriting the matrix
$\bX^{\rmt}\bPi^{\natural}\bX$ as the sum
$\sum_{\ell} \bX_{\ell}\bX_{\pi^{\natural}(\ell)}^{\rmt}$.
Then we collect all terms $\bX_{\ell} \bX_{\pi^{\natural}(\ell)}^{\rmt}$
independent of $\bX_{\pi^{\natural}(i)}$ and $\bX_j$ in
the matrix $\bSigma$ and leave the remaining terms to the
 matrix
$\bDelta$, i.e., $\bDelta
\defequal \bX^{\rmt}\bPi^{\natural}\bX - \bSigma$.
This decomposition shares the same
spirit as the leave-one-out technique~\citep{karoui2013asymptotic, bai2010spectral, karoui2018impact, sur2019likelihood}. Then, we divide all terms  in
$\Expc \Xi^2$ into  $3$ categories:
$1)$ terms only containing
matrix $\bSigma$; $2)$ terms containing both
$\bSigma$ and $\bDelta$; and $3)$ terms only containing~$\bDelta$.

\item
\textbf{Phase II: Conditional technique.}
Concerning the terms in the first two categories, which covers the
majority of terms,
we can exploit the independence among the rows in the sensing matrix
$\bX$. With the conditional technique, we can
reduce the order of  Gaussian moments by
separately taking the expectation w.r.t $\bSigma$ and
w.r.t vectors $\bX_{\pi^{\natural}(i)}$ and $\bX_j$.

\item
\textbf{Phase III: Direct computation.}
For the few terms in the third category (i.e., terms only containing~$\bDelta$),
we compute the high-order Gaussian moments by exhausting
all terms and iterative applying
of Wick's Theorem and Stein's Lemma, which
can reduce the higher-order Gaussian moments to lower-order Gaussian moments.
\end{itemize}

\newpage

Adopting the above proof strategy, we obtain the following results for each term listing as
\[
\Expc \Xi_1^2
\approx~& \bracket{n-h}^2\bracket{1+ \frac{2p}{n} +\frac{p^2}{n(n-h)} }\Bracket{\trace(\bM)}^2 \\
+~& n^2\Bracket{\frac{2p}{n} + 3\bracket{1-\frac{h}{n}}^2 + \frac{6(n-h)^2 p}{n^3}
+ \frac{(3n-h)p^2}{n^3}}\trace(\bM\bM), \\
\Expc \Xi_2^2 \approx ~&
2np\bracket{1 +p/n}\trace(\bM), \\
\Expc \Xi_3^2 \approx ~&
2 n^2\bracket{\frac{p}{n}+ \bracket{1 - \frac{h}{n}}^2
+ \frac{p^2}{n^2} + \frac{4p(n-h)^2}{n^3}} \trace(\bM), \\
\Expc \Xi_4^2 \approx ~& \frac{(n-h)m^2p^2}{n}, \\
\Expc \Xi_1 \Xi_4 \approx ~&
\frac{mp(n-h)(n+p-h)}{n}\trace(\bM), \\
\Expc \Xi_2 \Xi_3 \approx ~& \
\frac{p(n-h)(n+p-h)}{n}\trace(\bM).
\]
Plugging the calculation results thereof to \eqref{eq:non_oracle_var} and exploiting the relation $\Var\Xi = \Expc \Xi^2 - (\Expc \Xi)^2$, we complete the proof
of Theorem~\ref{thm:nonoracle_mean_var}.

\subsection{An illustrating example}
\label{subsec:non_oralce_illustrate_example}
This subsection predicts the phase transition points with Theorem~\ref{thm:nonoracle_mean_var}.
Unlike the oracle case,
we notice the edge weights $\bE_{ij}$ are
strongly correlated, especially when
$j = \pi^{\natural}(j)$, which
corresponds to the non-permuted rows. To factor
out these dependencies, we only take the permuted rows
into account and correct the sample size from $n$ to
$\tau_h n$. The prediction $\snr_{\-{non-oracle}}$ is then computed by solving  $2\log(n\tau_h)\Var\Xi = \bracket{\Expc \Xi}^2$, where $\Expc \Xi$ and
$\Var \Xi$ are in Theorem~\ref{thm:nonoracle_mean_var}.

To illustrate the prediction accuracy, we consider the case where $\bBtrue$'s singular values
are of the same order, i.e., $\frac{\lambda_i(\bBtrue)}{\lambda_j(\bBtrue)} = O(1),~1\leq i, j \leq m$, where $\lambda_i(\cdot)$ denotes the $i$-th singular value.
Then, we obtain the $\snr_{\textup{non-oracle}}$, which is written as
\begin{align}
\label{eq:nonoracle_snr_phase_transition}
\snr_{\-{non-oracle}} \approx {\eta_1}/{\eta_2}.
\end{align}
Here, $\eta_1$ and $\eta_2$ are defined as
\[
\eta_1 \defequal~&
2 \tau _h\tau_p^2\log  \left(n \tau _h\right) -
\tau_p(\tau_p+1)\left(1-\tau _h\right) + \tau_p\sqrt{2(1-\tau _h) \tau_h \cdot \log  \bracket{n \tau_h}}, \\
\eta_2 \defequal~&
\bracket{1-\tau_h}\left(\tau _p+1\right){}^2-
2\tau_h \tau_p^2 \log(n\tau_h).
\]

Note that the predicted $\snr_{\textup{non-oracle}}$ varies for different $\tau_h$ and $\tau_p$.
Viewing $\snr_{\textup{non-oracle}}$ as a function of
$\tau_h$, we observe a singularity point of
$\tau_h$, i.e., $\snr_{\textup{non-oracle}}(\tau_h)=\infty$, or equivalently, $\eta_2(\tau_h) = 0$. This suggests a potential phase transition
phenomenon w.r.t. $\tau_h$.
This predicted phenomenon is confirmed by the numerical experiments, in which we vary the proportion of the permuted rows and study the change in the reconstruction error rate.

\begin{remark}
To isolate the reconstruction performance from the impact of $\snr$, we adopt the noiseless setting,
which corresponds to infinite $\snr$. Hence, the change in the error rate comes solely from the increasing number of permuted rows rather than the insufficient $\snr$.
\end{remark}

\begin{remark}
The capability to predict the precise phase transition point of $\tau_h$ is a novel feature of our method. In contrast, previous proof in \citet{zhang2020optimal} only establishes $\tau_h$ as being of the order $O(1)$, without specifying its exact values, which our method can now predict.
\end{remark}

\begin{figure}[h]
\centering

\mbox{
\includegraphics[width = 2.8in]{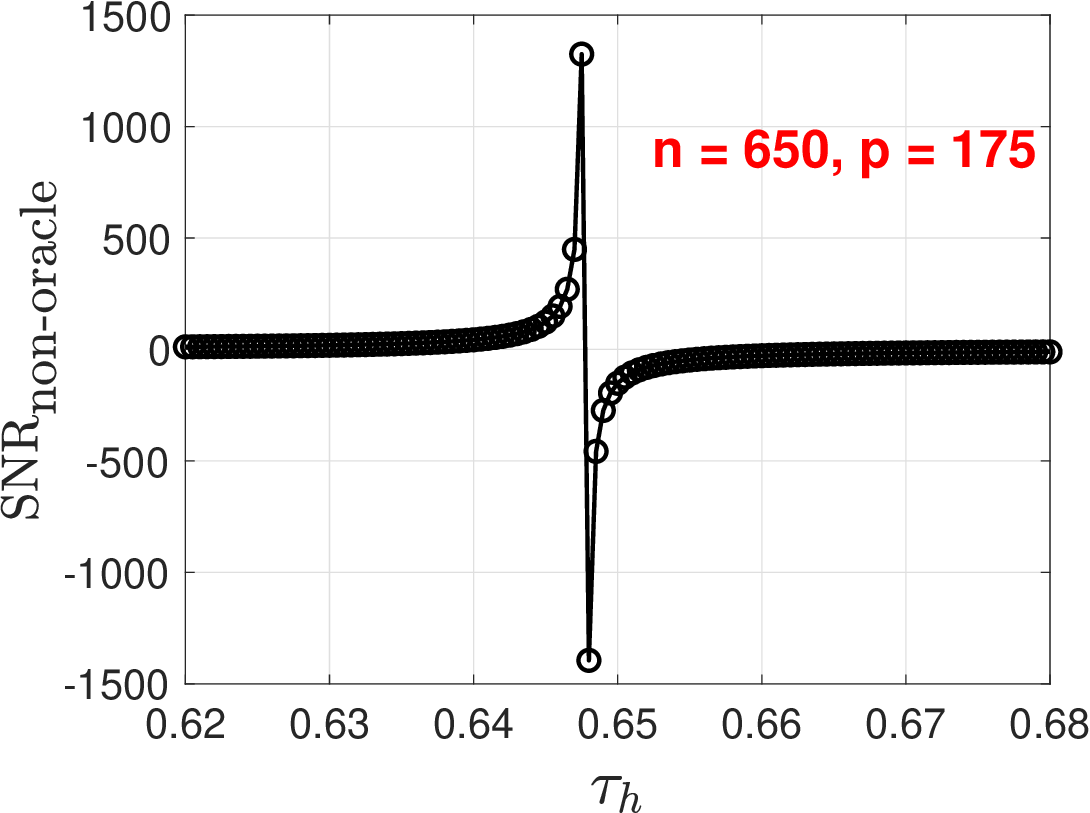}\hspace{0.2in}
\includegraphics[width = 2.7in]{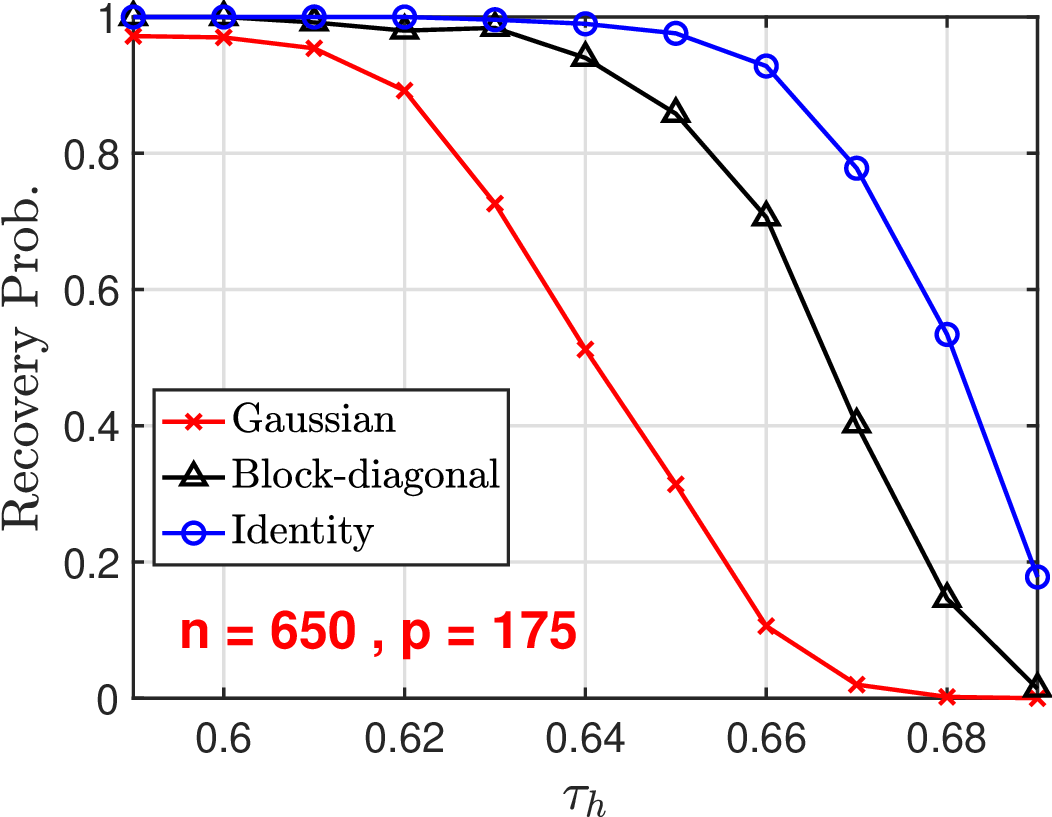}
}

\mbox{
\includegraphics[width = 2.8in]{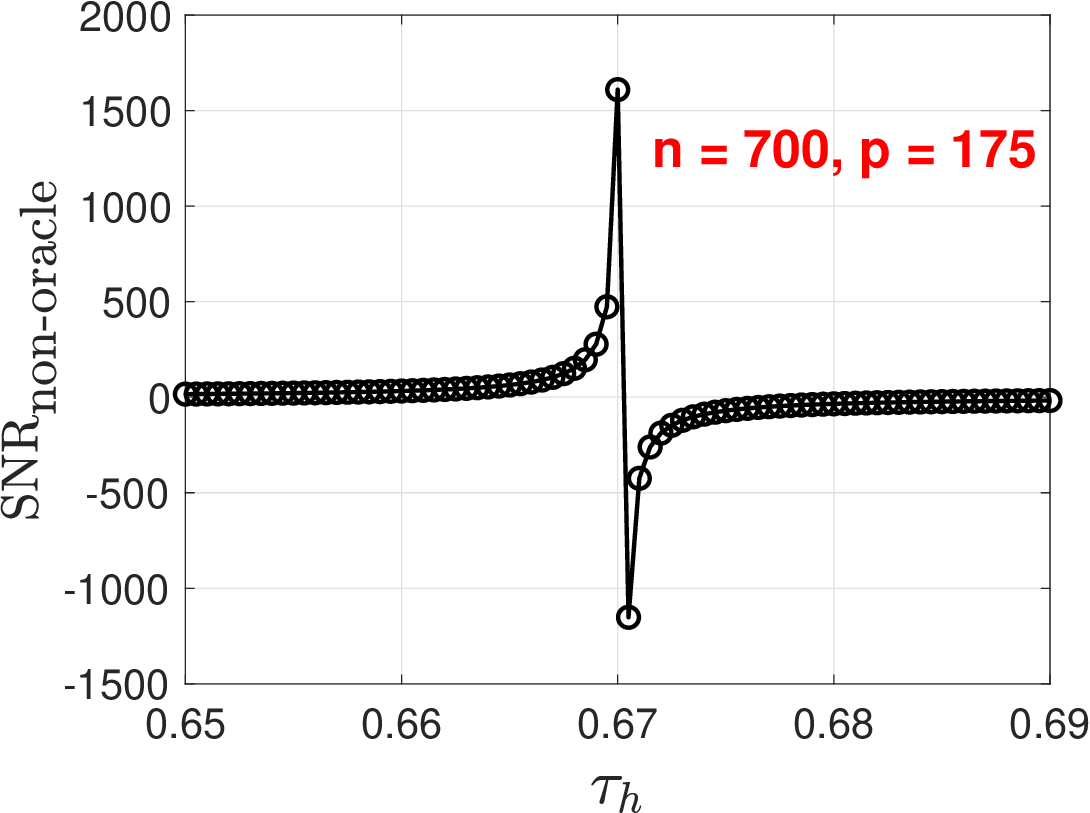}\hspace{0.2in}
\includegraphics[width = 2.7in]{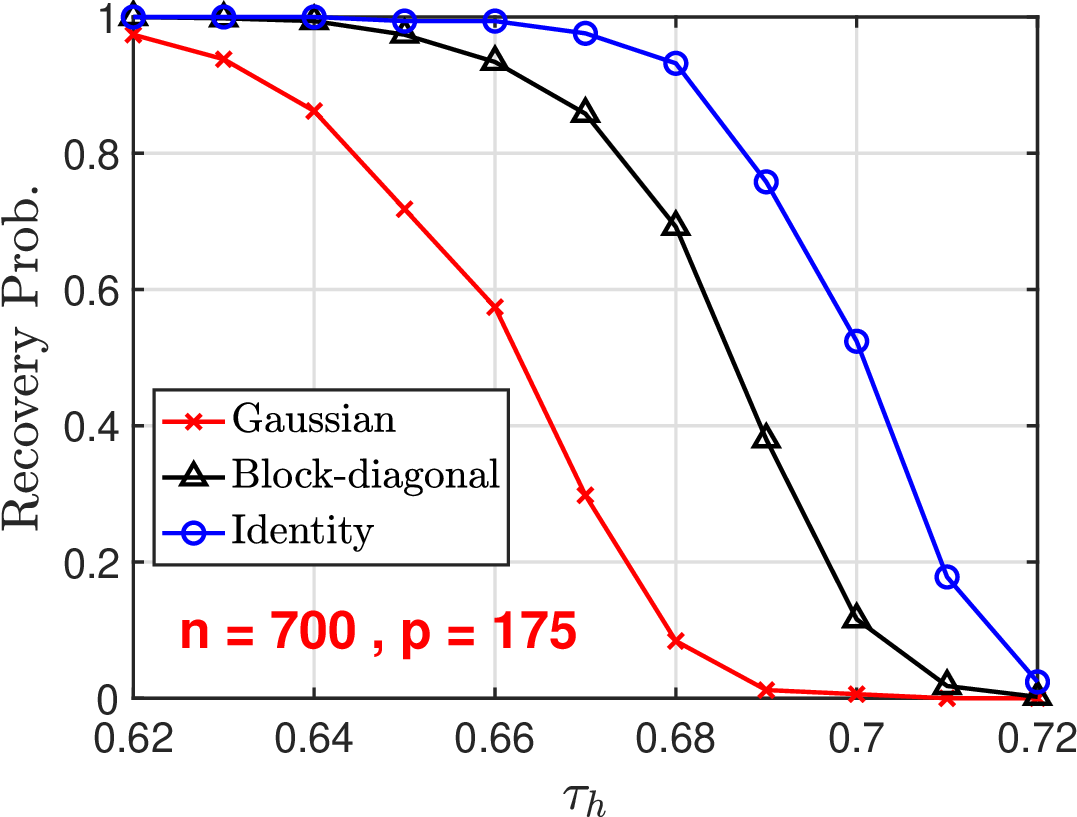}
}

\mbox{
\includegraphics[width = 2.8in]{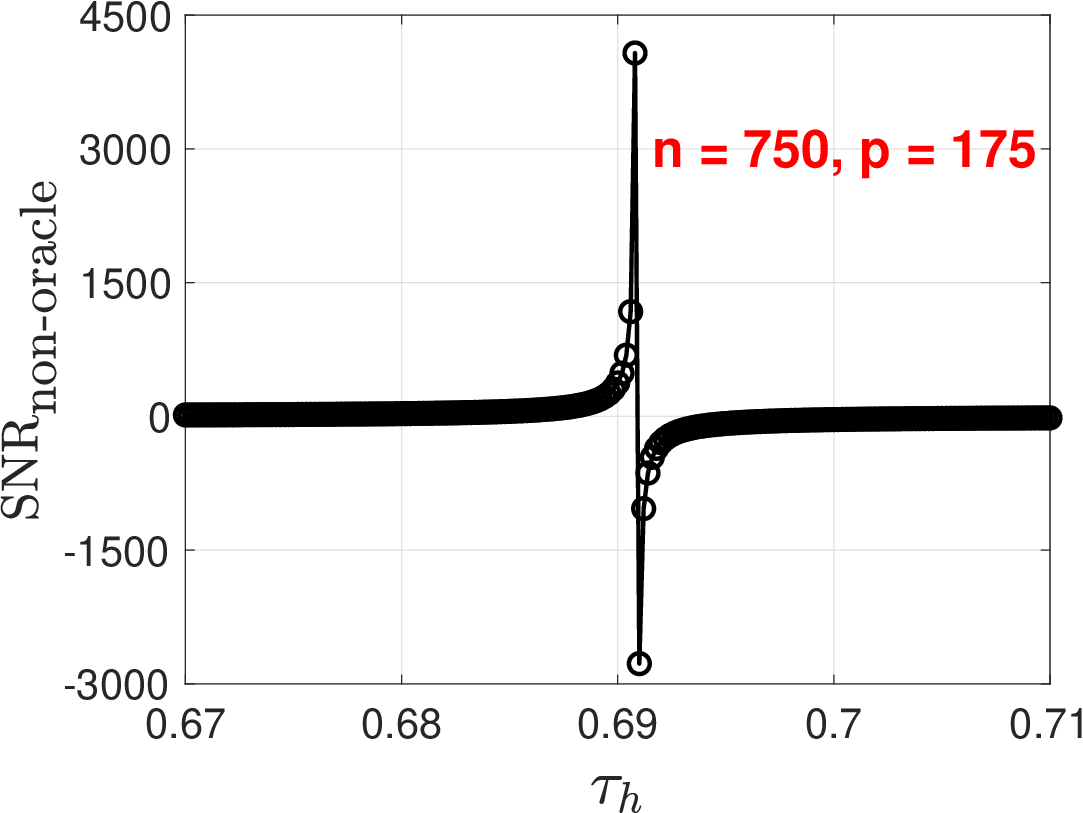}\hspace{0.2in}
\includegraphics[width = 2.7in]{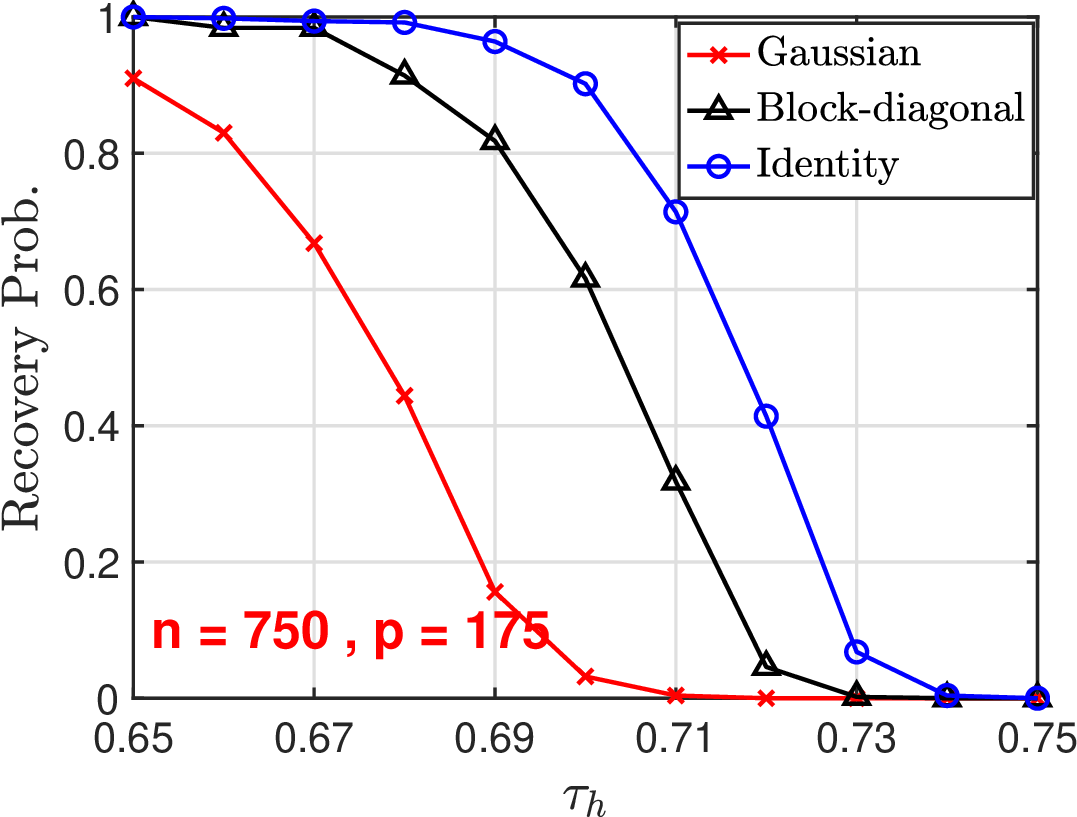}
}

\mbox{
\includegraphics[width = 2.8in]{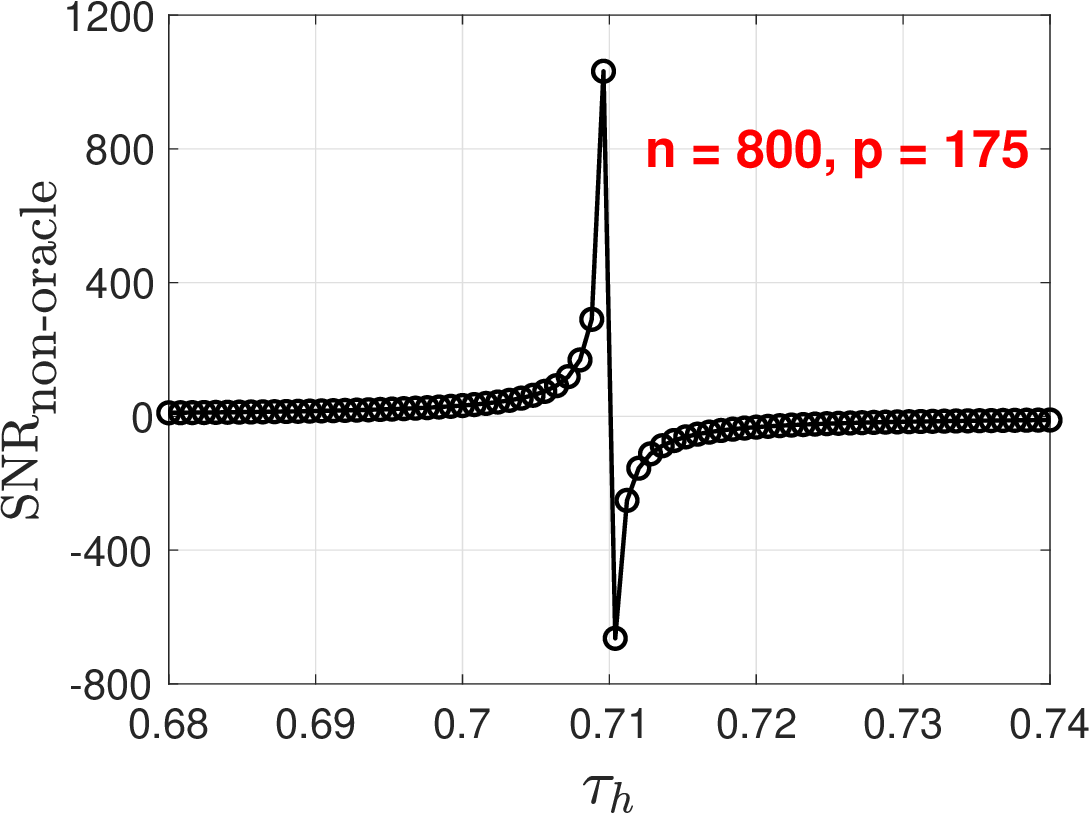}\hspace{0.2in}
\includegraphics[width = 2.7in]{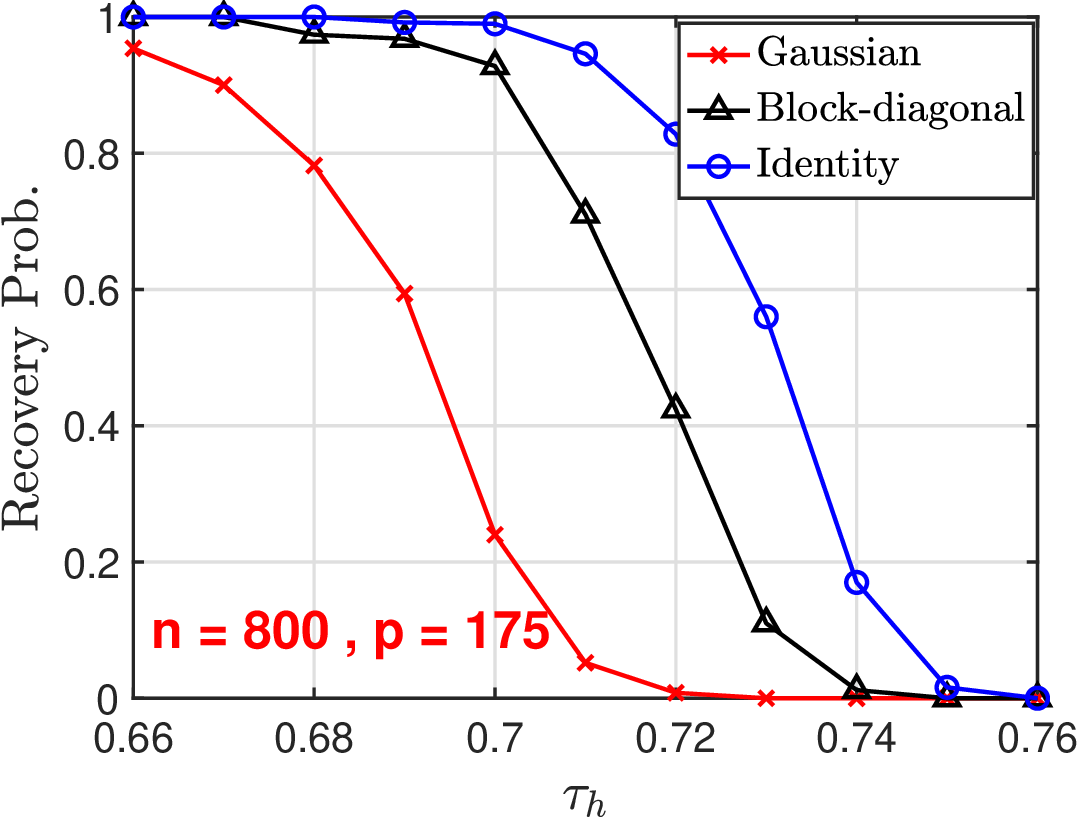}
}

\vspace{-0.1in}

\caption{\textbf{Left panel}:
Predicted phase transition points $\snr_{\textup{non-oralce}}$.
\textbf{Right panel}: Plot of the recovery rate under the noiseless
setting, i.e., $\snr = \infty$. \textbf{Gaussian}: $\bBtrue_{ij}\iid \normdist(0, 1)$; \textbf{Identity}: $\bBtrue = \bI_{p\times p}$;
\textbf{Block-diagonal}: $\bBtrue = \diag\set{1, \cdots, 1, 0.5, \cdots, 0.5}$. We observe that the
correct recovery rates drop sharply within the
regions of our predicted value.
}\label{fig:nonoracle_snr}
\end{figure}

\newpage\clearpage

\subsubsection{Impact of $n$ on the phase transition point}
We study the impact of $n$ on $\tau_h$.
The numerical experiment
is shown in Figure~\ref{fig:nonoracle_snr}, where we study the dependence of
$\snr_{\textup{non-oralce}}$ on $\tau_h$.
We can see the predicted phase transition $\tau_h$ matches
to a good extent to the numerical experiments.
Then, we fix the $p$ and study the impact of $n$ on $\tau_h$.
We observe that the phase transition $\tau_h$ increases together with the
sample number $n$, which is also captured by our formula in
\eqref{eq:nonoracle_snr_phase_transition}.

\subsubsection{Limits of $\tau_h$}

In addition, we consider the limiting behavior of $\tau_h$ when $\tau_p$ approaches $0$, or equivalently,
$p = \oprate{n}$. We can simplify $\Expc \Xi$ and $\Var \Xi$ in
Theorem~\ref{thm:nonoracle_mean_var} as
\[
\Expc \Xi
\simeq~& n\bracket{1-\tau_h}\Fnorm{\bB^{\natural}}^2 , \\
\Var \Xi \simeq ~&  3 n^2\bracket{1-\tau_h}^2\Fnorm{\bB^{\natural \rmt}\bB^{\natural}}^2.
\]
We notice
that the singularity point in~\eqref{eq:nonoracle_snr_phase_transition} disappears. In other words,
we can have the correct permutation matrix $\bPi^{\natural}$
even when $h\approx n$. This is (partly) verified by
Figure~\ref{fig:corr_rate}, from which we observe that
the phase transition point w.r.t. $\tau_h$ approaches to one, or equivalently, $h$ approaches $n$,
as $\tau_p$ decreases to zero.

\begin{figure}[h]
\centering
    \mbox{
    \includegraphics[width = 2.8in]{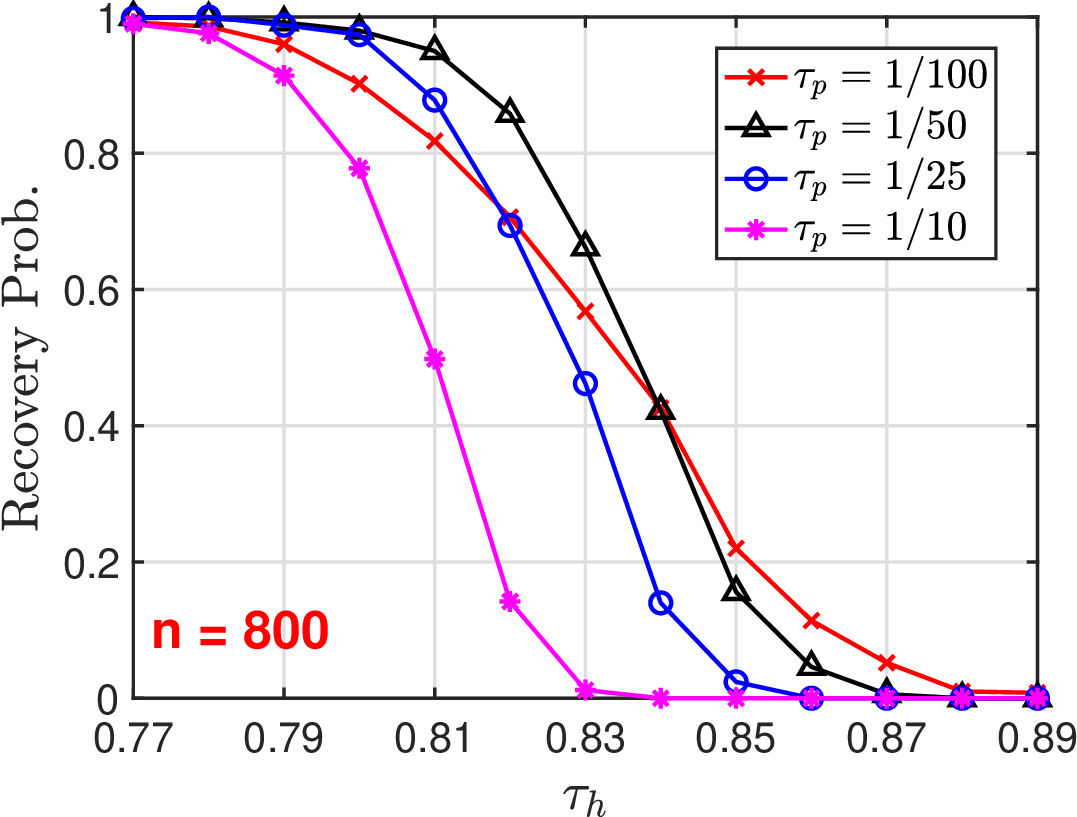}\hspace{0.2in}
    \includegraphics[width = 2.8in]{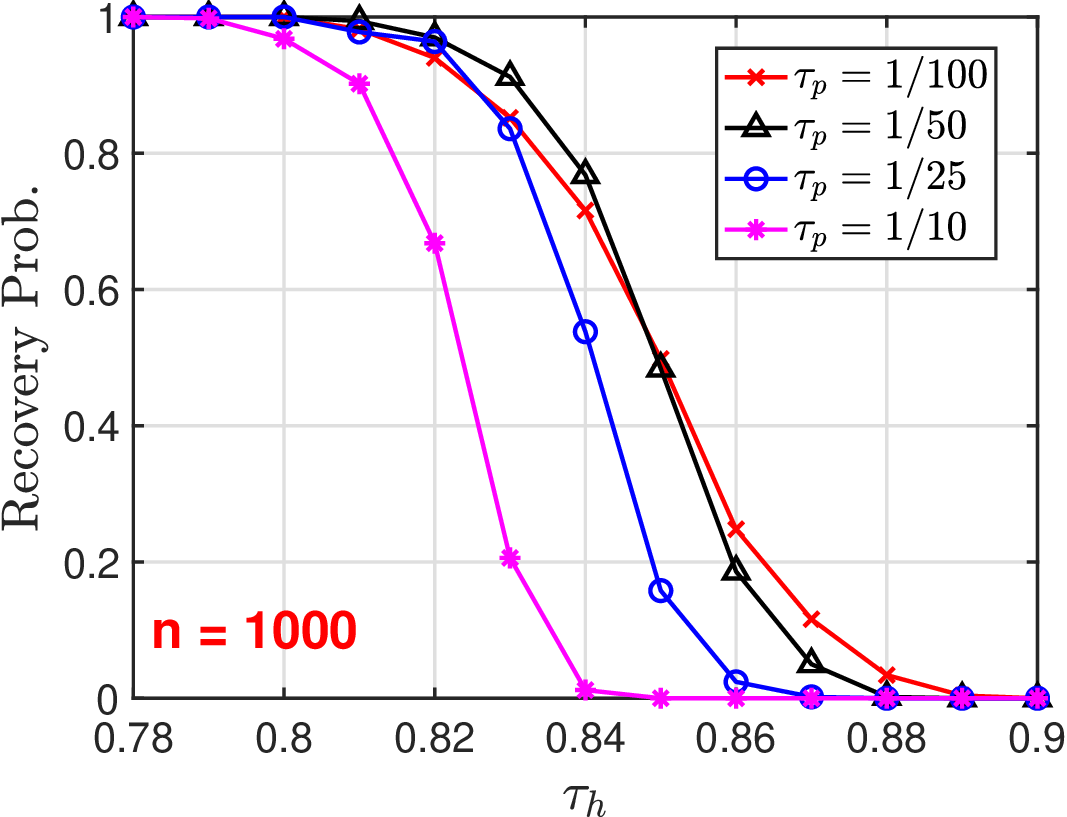}
    }

    \mbox{
    \includegraphics[width = 2.8in]{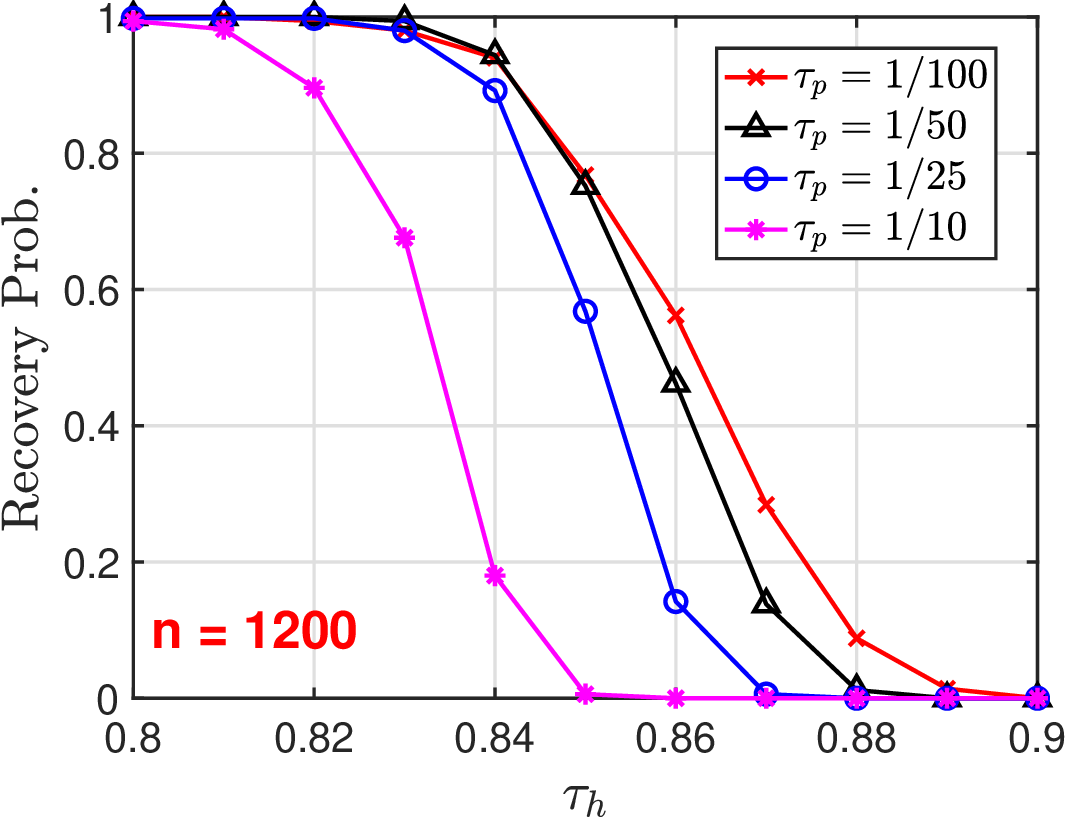}\hspace{0.2in}
    \includegraphics[width = 2.8in]{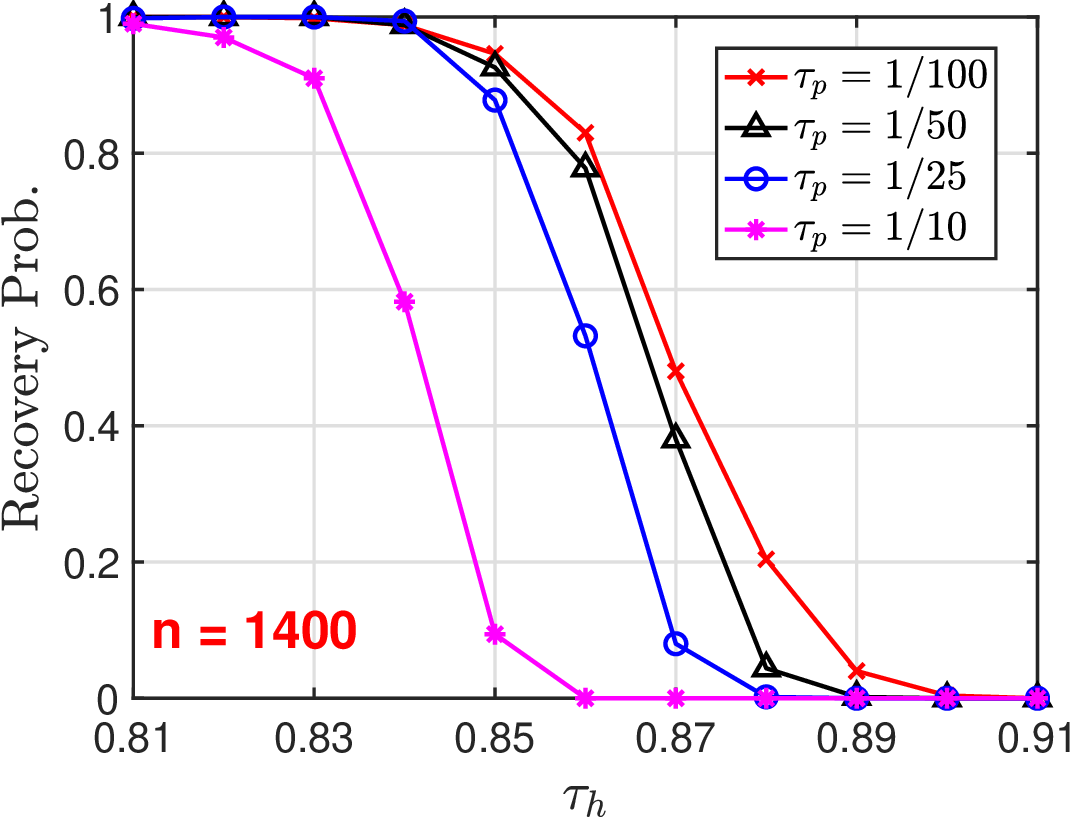}
    }

    \vspace{-0.1in}

    \caption{Plot of correct recovery rate w.r.t. $\tau_h$.
    We consider the noiseless scenario (i.e., $\snr= \infty$) and
    pick $n= \{800, 1000, 1200, 1400\}$.
    }
    \label{fig:corr_rate}
\end{figure}

\section{Conclusion}

The shuffled (permuted) regression problem is a well-known challenging task, with  practical applications in databases, machine learning, and privacy. This is the first work that can
identify the precise location of
phase transition thresholds of permuted
linear regressions.
For the oracle case
where the signal $\bB^{\natural}$ is given as a prior, our analysis can predict the
phase transition threshold $\snr_{\oracle}$ to a good extent.
For the non-oracle case
where $\bB^{\natural}$ is not given, we have
modified the leave-one-out technique
to approximately compute the phase critical $\snr_{\textup{non-oracle}}$ value
for the phase transition, as the precise computation becomes significantly complicated
as the high-order interaction between Gaussian random variables is involved.
Moreover, we have associated the singularity point in $\snr_{\textup{non-oracle}}$ with
a phase transition point w.r.t the maximum allowed number of permuted rows. Finally, we have   presented  many numerical experiments to corroborate the accuracy of our theoretical predictions.

\bibliography{refs_scholar}
\bibliographystyle{plainnat}

\newpage\clearpage

\appendix
\allowdisplaybreaks

\section{Numerical Methods for the Phase Transition Points}
In this section, we present the numerical method to compute the phase transition points. Notice that the correct recovery rate is in monotonic non-decreasing relation with the $\snr$, we adopt a binary-search-based method.
\par
First, we fix the $\snr$ and run the experiments for
 $100$ times. Then, we calculate the error rate of permutation recovery (full permutation recovery). If the error rate is below $0.05$, we regard the corresponding $\snr$ as above the phase transition point and try a smaller value. Otherwise, we regard the $\snr$ as below the phase transition point and try a larger value. The detailed description is given in Algorithm~\ref{alg:phase_transition_numeric}.

In our numerical experiments, we run $20$ times of Algorithm~\ref{alg:phase_transition_numeric}  for each parameter setting. Then, we estimate its mean and the standard deviation from these estimated phase transition points.

\vspace{0.2in}

\begin{algorithm}[!ht]
\begin{algorithmic}[1]

\STATE
\textbf{Initialization.} Set the initial search range for $\snr$ as $[l, r]$.
Define the precision threshold $\varepsilon$.

\STATE
\WHILE{$\abs{l-r} > \varepsilon$}

\STATE Set $\snr_{\textup{middle}} = \frac{l + r}{2}$.

\STATE Run experiments $100$ times for this $\snr_{\textup{middle}}$.

\STATE Compute the error rate of full permutation recovery.

\STATE
\IF{the error rate is below $0.05$}
\STATE $\snr_{\textup{middle}} \rightarrow r - \varepsilon$, ~~~\# we have  $\snr_{\textup{middle}}$ be greater than the phase transition point
\ELSE
\STATE $\snr_{\textup{middle}} \rightarrow l + \varepsilon$.~~~\# we have  $\snr_{\textup{middle}}$ be no greater than the phase transition point
\ENDIF
\STATE
\ENDWHILE
\STATE
\STATE
\textbf{Output.} Return the phase transition point $\snr_{\textup{middle}}$.
\end{algorithmic}
\caption{Numerical method to compute the phase transition points.}
\label{alg:phase_transition_numeric}
\end{algorithm}

\vspace{0.2in}
\noindent \textbf{Complexity analysis}. For a given precision threshold $\varepsilon$, each iteration (Line $3$ to Line $14$ in Algorithm~\ref{alg:phase_transition_numeric}) takes $O(\log \frac{1}{\varepsilon})$ rounds to converge and it runs the permutation recovery algorithm $100$ times in each round. Hence, we run $20\times 100\times \log_2(10^4)~(\approx 26575)$ permutation recovery experiments for each parameter.

\newpage

\section{Analysis of the Non-Oracle Case}
\label{sec:non_oracle_case}

This section presents the technical details
in analyzing the non-oracle case, to put more specifically, Theorem~\ref{thm:nonoracle_mean_var}.

\subsection{Notations}
Note that our analysis can involve the terms
containing $\bracket{\bX_{\pi^{\natural}}(i) - \bX_j}$
and $\bX^{\rmt}\bPi^{\natural}\bX$ simultaneously.
To decouple the dependence between
$\bracket{\bX_{\pi^{\natural}}(i) - \bX_j}$ and
$\bX^{\rmt}\bPi^{\natural}\bX$, we first rewrite the matrix
$\bX^{\rmt}\bPi^{\natural}\bX$ as the sum
$\sum_{\ell} \bX_{\ell}\bX_{\pi^{\natural}(\ell)}^{\rmt}$
and then collect all terms $\bX_{\ell} \bX_{\pi^{\natural}(\ell)}^{\rmt}$
independent of $\bX_{\pi^{\natural}(i)}$ and $\bX_j$ in
the matrix $\bSigma$, which is written as
\begin{align}
\label{eq:sigma_matrix_def}
\bSigma \defequal \sum_{\ell, \pi^{\natural}(\ell) \neq \pi^{\natural}(i), j}
\bX_{\ell}\bX_{\pi^{\natural}(\ell)}^{\rmt}.
\end{align}
The rest terms are then put in the matrix $\bDelta$ such that
$\bX^{\rmt}\bPi^{\natural}\bX = \bSigma + \bDelta$. Note that
the expression of $\bDelta$ varies under different cases such that
\begin{itemize}[leftmargin=*]
\item
\textbf{Case $(s, s)$: $i = \pi^{\natural}(i)$ and $j = \pi^{\natural}(j)$}.
We have
\begin{align}
\label{eq:bdelta_ss_def}
\bDelta = \bDelta^{(s,s)}= \bX_i \bX_i^{\rmt} + \bX_j \bX_j^{\rmt}.
\end{align}
\item
\textbf{Case $(s, d)$: $i = \pi^{\natural}(i)$ and $j \neq \pi^{\natural}(j)$}.
We have
\begin{align}
\label{eq:bdelta_sd_def}
\bDelta = \bDelta^{(s, d)} =  \bX_{i}\bX_{i}^{\rmt}
+ \bX_{j}\bX_{\pi^{\natural}(j)}^{\rmt} +
\bX_{\pi^{\natural -1}(j)}\bX_{j}^{\rmt}.
\end{align}
\item
\textbf{Case $(d, s)$: $i\neq \pi^{\natural}(i)$ and $j = \pi^{\natural}(j)$}.
We have
\begin{align}
\label{eq:bdelta_ds_def}
\bDelta = \bDelta^{(d, s)} = \bX_{i}\bX_{\pi^{\natural}(i)}^{\rmt}
+ \bX_{\pi^{\natural}(i)}\bX_{\pi^{\natural 2}(i)}^{\rmt}
+ \bX_{j}\bX_{j}^{\rmt}.
\end{align}
\item
\textbf{Case $(d, d)$: $i \neq \pi^{\natural}(i)$ and $j\neq \pi^{\natural}(j)$}.
We have
\begin{align}
\label{eq:bdelta_dd_def}
\bDelta = \bDelta^{(d,d)} = \bX_{i}\bX_{\pi^{\natural}(i)}^{\rmt}
+ \bX_{\pi^{\natural}(i)}\bX_{\pi^{\natural 2}(i)}^{\rmt}
+ \bX_{j}\bX_{\pi^{\natural}(j)}^{\rmt} +
\bX_{\pi^{\natural -1}(j)}\bX_{j}^{\rmt}.
\end{align}
\end{itemize}

\noindent In addition, we define the matrix $\bM$ as $\bB^{\natural}\bB^{\natural \rmt}$,
and define the index sets $\calS, \calD$, and $\calD_{\textup{pair}}$ as
\begin{align}
\calS \defequal & \set{\ell~|~\ell \neq i~\-{or}~j,~\ell = \pi^{\natural}(\ell)}, \label{eq:index_sset_def}\\
\calD \defequal & \set{\ell~|~\ell, \pi^{\natural}(\ell) \neq i~\-{or}~j,~\ell \neq \pi^{\natural}(\ell)}, \label{eq:index_dset_def}\\
\calD_{\textup{pair}} \defequal & \set{(\ell_1, \ell_2): \ell_1 = \pi^{\natural}(\ell_2),
 \ell_2 = \pi^{\natural}(\ell1), \ell_1, \ell_2\in \calD},
 \label{eq:index_dpairset_def}
\end{align}
respectively.

\subsection{Main Computation}
First, we decompose $\Xi$ as
\[
\Xi
=~& \underbrace{\bX_{\pi^{\natural}(i)}^{\rmt}\bB^{\natural}\bB^{\natural\rmt}\bX^{\rmt}\bPi^{\natural \rmt} \bX
\Bracket{\bX_{\pi^{\natural}(i)} - \bX_j} }_{\defequal \Xi_1}
+ \sigma \underbrace{\bX_{\pi^{\natural}(i)}^{\rmt}\bB^{\natural} \bW^{\rmt}\bX\Bracket{\bX_{\pi^{\natural}(i)} - \bX_j}}_{\defequal \Xi_2} \\
+~& \sigma \underbrace{\bW_i^{\rmt} \bB^{\natural\rmt}\bX^{\rmt}\bPi^{\natural \rmt} \bX \Bracket{\bX_{\pi^{\natural}(i)} - \bX_j}}_{\defequal \Xi_3}
+ \sigma^2 \underbrace{\bW_i^{\rmt} \bW^{\rmt}\bX \Bracket{\bX_{\pi^{\natural}(i)} - \bX_j}}_{\defequal \Xi_4}.
\]
The following context separately computes its expectation
$\Expc \Xi$ and its variance $\Var\Xi$.

\paragraph{Expectation.}
We can easily verify that both $\Expc \Xi_2$ and $\Expc \Xi_3$ are zero.
Then our goal turns to calculating the expectation of $\Expc \Xi_1$ and
$\Expc \Xi_4$.
First, we have
\[
\Expc \Xi_1
=~& \Expc \sum_{\ell = \pi^{\natural}(\ell)}
\bX_{\pi^{\natural}(i)}^{\rmt} \bM \bX_{\ell}\bX_{\ell}^{\rmt}\bX_{\pi^{\natural}(i)}- \Expc \sum_{\ell} \bX_{\pi^{\natural}(i)}^{\rmt} \bM
\bX_{\pi^{\natural}(\ell)}\bX_{\ell}^{\rmt} \bX_j.
\]
With Lemma~\ref{lemma:xioneone_expc} and Lemma~\ref{lemma:xionetwo_expc}, we conclude
\begin{align}
\label{eq:xione_expc}
\Expc \Xi_1 =~&
(n-h)\trace(\bM) + (p+1)\Expc \Ind_{i = \pi^{\natural}(i)}\trace(\bM)
- \bracket{p \Expc \Ind_{i = j} + \Expc\Ind_{j = \pi^{\natural 2}(i)}}\trace(\bM) \notag \\
=~& (n+p-h - hp/n)\Bracket{1 + o(1)}\trace(\bM).
\end{align}
Meanwhile, we have
\begin{align}
\label{eq:xifour_expc}
\Expc \Xi_4 =~&\
\Expc\Bracket{\bW_i^{\rmt}\bW_1~\cdots~\bW_i^{\rmt}\bW_i~\cdots \bW_i^{\rmt}\bW_n} \bX\bracket{\bX_{\pi^{\natural}(i)} - \bX_j} \notag \\
=~& m \Expc \bX_i^{\rmt}\bracket{\bX_{\pi^{\natural}(i)} - \bX_j}
= mp \bracket{\Expc \Ind_{i = \pi^{\natural}(i)} - \Expc \Ind_{i = j}}
= \frac{mp(n-h)\sigma^2}{n}\Bracket{1 + o(1)}.
\end{align}
Combining ~\eqref{eq:xione_expc} and ~\eqref{eq:xifour_expc} and neglecting the $o(1)$ terms yields
\[
\Expc \Xi \approx ~&
(n+p)\bracket{1- h/n}\Fnorm{\bB^{\natural}}^2 +
\frac{mp(n-h)\sigma^2}{n}.
\]

\paragraph{Variance.}
Then we study the variance of $\Xi$. With the relation
$\Var(\Xi) = \Expc \Xi^2 - (\Expc \Xi)^2$, our goal reduces
to computing $\Expc \Xi^2$, which can be written as
\[
\Expc \Xi^2 = \Expc \Xi_1^2 + \sigma^2 \Expc \Xi_2^2 + \sigma^2\Expc \Xi_3^2
+ \sigma^4\Expc \Xi_4^2 + 2\sigma^2 \Expc \Xi_1 \Xi_4 + 2\sigma^2\Expc \Xi_2 \Xi_3.
\]
The following context separately computes each term as
\[
\Expc \Xi_1^2
\approx~& \bracket{n-h}^2\bracket{1+ \frac{2p}{n} +\frac{p^2}{n(n-h)} }\Bracket{\trace(\bM)}^2 \\
+~& n^2\Bracket{\frac{2p}{n} + 3\bracket{1-\frac{h}{n}}^2 + \frac{6(n-h)^2 p}{n^3}
+ \frac{(3n-h)p^2}{n^3}}\trace(\bM\bM), \\
\Expc \Xi_2^2 \approx ~&
2np\bracket{1 +p/n}\trace(\bM), \\
\Expc \Xi_3^2 \approx ~&
2 n^2\bracket{\frac{p}{n}+ \bracket{1 - \frac{h}{n}}^2
+ \frac{p^2}{n^2} + \frac{4p(n-h)^2}{n^3}} \trace(\bM), \\
\Expc \Xi_4^2 \approx ~& \frac{(n-h)m^2p^2}{n}, \\
\Expc \Xi_1 \Xi_4 \approx ~&
\frac{mp(n-h)(n+p-h)}{n}\trace(\bM), \\
\Expc \Xi_2 \Xi_3 \approx ~& \
\frac{p(n-h)(n+p-h)}{n}\trace(\bM).
\]
The detailed computation is attached as follows.

\newpage

\begin{lemma}
\label{lemma:xione_square_expc}
We have
\[
\Expc \Xi_1^2
=~& \bracket{n-h}^2\bracket{1+ \frac{2p}{n} +\frac{p^2}{n(n-h)} + o(1)}\Bracket{\trace(\bM)}^2 \\
+~& n^2\Bracket{\frac{2p}{n} + 3\bracket{1-\frac{h}{n}}^2 + \frac{6(n-h)^2 p}{n^3}
+ \frac{(3n-h)p^2}{n^3} + o(1)}\trace(\bM\bM),
\]
where $\Xi_1$ is defined in ~\eqref{eq:xi_decomposition}.
\end{lemma}

\begin{proof}
We begin the proof by decomposing $\Xi_1^2$ as
\[
\Expc \Xi_1^2
=~& \Expc
\underbrace{\bracket{\bX_{\pi^{\natural}(i)} - \bX_j}^{\rmt}
\bSigma\bM
\bX_{\pi^{\natural}(i)} \bX_{\pi^{\natural}(i)}^{\rmt}
\bM\bSigma^{\rmt}
\bracket{\bX_{\pi^{\natural}(i)} - \bX_j}}_{\Lambda_1} \\
+~&  2 \Expc
\underbrace{\bracket{\bX_{\pi^{\natural}(i)} - \bX_j}^{\rmt}
\bSigma\bM
\bX_{\pi^{\natural}(i)} \bX_{\pi^{\natural}(i)}^{\rmt}
\bM \bDelta^{\rmt}
\bracket{\bX_{\pi^{\natural}(i)} - \bX_j}}_{\Lambda_2} \\
+~&  \Expc
\underbrace{\bracket{\bX_{\pi^{\natural}(i)} - \bX_j}^{\rmt}
\bDelta \bM
\bX_{\pi^{\natural}(i)} \bX_{\pi^{\natural}(i)}^{\rmt}
\bM \bDelta^{\rmt}
\bracket{\bX_{\pi^{\natural}(i)} - \bX_j}}_{\Lambda_3},
\]
and separately bound each term as in
Lemma~\ref{lemma:xione_square_lambda1_summary_expc}, Lemma~\ref{lemma:xione_square_lambda2_summary_expc},
and Lemma~\ref{lemma:xione_square_lambda3_summary_expc}.

\end{proof}

\begin{lemma}
\label{lemma:xione_square_lambda1_summary_expc}
We have
\[
& \Expc \bracket{\bX_{\pi^{\natural}(i)} - \bX_j}^{\rmt}
\bSigma\bM
\bX_{\pi^{\natural}(i)} \bX_{\pi^{\natural}(i)}^{\rmt}
\bM\bSigma^{\rmt}
\bracket{\bX_{\pi^{\natural}(i)} - \bX_j} \\
=~&
\bracket{n-h}^2\bracket{1 + o(1)}\Bracket{\trace(\bM)}^2 + \
n^2\Bracket{\frac{2p}{n} + 3\bracket{1-\frac{h}{n}}^2 + o(1)}
\trace(\bM\bM).
\]	
\end{lemma}

\begin{proof}
Due to the independence among different rows of the sensing matrix
$\bX$, we condition on $\bSigma$ and take expectation w.r.t.
$\bX_{\pi^{\natural}(i)}$ and $\bX_j$, which leads to
\[
\Expc \Lambda_1
=~&
\Expc\underbrace{
\bX_{\pi^{\natural}(i)}^{\rmt}
\bSigma\bM
\bX_{\pi^{\natural}(i)} \bX_{\pi^{\natural}(i)}^{\rmt}
\bM\bSigma^{\rmt}
\bX_{\pi^{\natural}(i)}}_{\Lambda_{1,1}} +
\Expc \underbrace{\bX_j^{\rmt}
\bSigma\bM
\bX_{\pi^{\natural}(i)} \bX_{\pi^{\natural}(i)}^{\rmt}
\bM\bSigma^{\rmt}\bX_j}_{\Lambda_{1, 2}}.
\]
For $\Expc \Lambda_{1, 1}$, we obtain
\[
\Expc \Lambda_{1,1}
\stackrel{\cirone}{=}~& \Expc\Bracket{\trace\bracket{\bSigma \bM}\trace\bracket{\bSigma \bM}}
+ \Expc \trace\bracket{\bSigma \bM \bM^{\rmt}\bSigma^{\rmt}}
+ \Expc \trace\bracket{\bSigma \bM \bSigma \bM }  \\
\stackrel{\cirtwo}{=}~&
\bracket{n-h}^2\Bracket{1 + o(1)}\Bracket{\trace(\bM)}^2 + \
n^2\Bracket{\frac{p}{n} + 2\bracket{1-\frac{h}{n}}^2 + o(1)}
\trace(\bM\bM),
\]
where $\cirone$ is due to ~\eqref{eq:double_quad_expc},
and $\cirtwo$ is due to Lemma~\ref{lemma:sigmaMsigma_trace_expc}, Lemma~\ref{lemma:sigmaMsigmaM_trace_expc}, and Lemma~\ref{lemma:sigmaM_square}.
As for $\Expc \Lambda_{1, 2}$, we have
\[
\Expc \Lambda_{1, 2} =
\Expc \trace\bracket{\bSigma\bM\bM^{\rmt}\bSigma^{\rmt}} =
n^2\Bracket{\frac{p}{n}+ \bracket{1 - \frac{h}{n}}^2
+ o(1)} \trace\bracket{\bM^{\rmt}\bM},
\]
and hence complete the proof.
\end{proof}

\begin{lemma}
\label{lemma:xione_square_lambda2_summary_expc}
We have
\[
& \Expc
\bracket{\bX_{\pi^{\natural}(i)} - \bX_j}^{\rmt}
\bSigma\bM
\bX_{\pi^{\natural}(i)} \bX_{\pi^{\natural}(i)}^{\rmt}
\bM \bDelta^{\rmt}
\bracket{\bX_{\pi^{\natural}(i)} - \bX_j} \approx \frac{(n-h)^2 p}{n}
\Bracket{
\bracket{\trace(\bM)}^2
+ 3\trace(\bM\bM)}.
\]
\end{lemma}

\begin{proof}
Similar as above, we first expand $\Lambda_{2}$ as
\[
\Expc \Lambda_2 =~&
(n-h)\Expc \underbrace{\bX_{\pi^{\natural}(i)}^{\rmt}
\bM
\bX_{\pi^{\natural}(i)} \bX_{\pi^{\natural}(i)}^{\rmt}
\bM
\bDelta^{\rmt}
\bX_{\pi^{\natural}(i)}}_{\Lambda_{2, 1}} +
(n-h)\Expc \underbrace{\bX_j^{\rmt}
\bM
\bX_{\pi^{\natural}(i)} \bX_{\pi^{\natural}(i)}^{\rmt}
\bM
\bDelta^{\rmt}\bX_j}_{\Lambda_{2, 2}} \\
-~&
(n-h)\Expc \underbrace{\bX_{\pi^{\natural}(i)}^{\rmt}
\bM
\bX_{\pi^{\natural}(i)} \bX_{\pi^{\natural}(i)}^{\rmt}
\bM
\bDelta^{\rmt}
\bX_j}_{\Lambda_{2, 3}} -
(n-h)\Expc \underbrace{\bX_j^{\rmt}
\bM
\bX_{\pi^{\natural}(i)} \bX_{\pi^{\natural}(i)}^{\rmt}
\bM
\bDelta^{\rmt}
\bX_{\pi^{\natural}(i)}}_{\Lambda_{2, 4}}.
\]

\noindent \textbf{Case $(s, s)$: $i = \pi^{\natural}(i)$ and $j = \pi^{\natural}(j)$.}
We first compute $\Lambda_{2, 1}$ as
\[
\Expc \Lambda_{2, 1}
=~&
\underbrace{\Expc \bX_{i}^{\rmt}
\bM
\bX_{i} \bX_{i}^{\rmt}
\bM
\bX_i \bX_i^{\rmt}
\bX_{i}}_{
\Expc \norm{\bX_i}{2}^2
\bracket{\bX_i^{\rmt}\bM \bX_i}^2
}  +
\underbrace{\Expc \bX_{i}^{\rmt}
\bM
\bX_{i} \bX_{i}^{\rmt}
\bM
\bX_j \bX_j^{\rmt}
\bX_{i}}_{
\Expc
\bracket{\bX_{i}^{\rmt}
\bM
\bX_{i}}^2
}  \\
=~&
\bracket{p+5}\Bracket{\bracket{\trace(\bM)}^2 + \trace(\bM \bM)+ \trace\bracket{\bM^{\rmt}\bM}}.
\]
We consider $\Lambda_{2, 2}$ as
\[
\Expc \Lambda_{2, 2}
=~&
\underbrace{\Expc\bracket{\bX_i^{\rmt} \bM \bX_i \bX_i^{\rmt}\bM \bX_i}}_{
\Expc\bracket{\bX_i^{\rmt}\bM \bX_i}^2
}
+ \underbrace{\Expc \bracket{\bX_j^{\rmt} \bM \bM \bX_j \bX_j^{\rmt}\bX_j}}_{\Expc \norm{\bX_j}{2}^2 \bX_j^{\rmt} \bM \bM\bX_j } \\
=~&
\bracket{\trace(\bM)}^2 + \trace(\bM \bM)+ \trace\bracket{\bM^{\rmt}\bM}
+ (p+2)\trace\bracket{\bM \bM}.
\]

As for $\Lambda_{2, 3}$ and $\Lambda_{2, 4}$, we can verify that
they are both zero, which gives
\begin{align}
\label{eq:xione_square_lambda2_ss_expc}
\Expc \Lambda_2
=~& (n-h)p\Bracket{1 + o(1)}\Bracket{\trace(\bM)}^2
+ 3(n-h)p\Bracket{1 + o(1)}\trace(\bM\bM).
\end{align}

\noindent \textbf{Case $(s, d)$: $i = \pi^{\natural}(i)$ and $j \neq \pi^{\natural}(j)$.} We can compute $\Lambda_{2, 1}$ as
\[
\Expc \Lambda_{2, 1}= ~&
\Expc \bX_{i}^{\rmt}
\bM
\bX_{i} \bX_{i}^{\rmt}
\bM
\bX_{i}\bX_{i}^{\rmt}
\bX_{i}  +
\underbrace{\Expc \bX_{\pi^{\natural}(i)}^{\rmt}
\bM
\bX_{\pi^{\natural}(i)} \bX_{\pi^{\natural}(i)}^{\rmt}
\bM
\bX_{\pi^{\natural}(j)}\bX_{j}^{\rmt}
\bX_{\pi^{\natural}(i)}}_{0}  \\
+~&
\underbrace{\Expc \bX_{\pi^{\natural}(i)}^{\rmt}
\bM
\bX_{\pi^{\natural}(i)} \bX_{\pi^{\natural}(i)}^{\rmt}
\bM
\bX_{j}\bX_{\pi^{\natural -1}(j)}^{\rmt}
\bX_{\pi^{\natural}(i)}}_{0} \\
=~& \bracket{p+4}\Bracket{\bracket{\trace(\bM)}^2 + \trace(\bM \bM)+ \trace\bracket{\bM^{\rmt}\bM}}.
\]

We consider $\Lambda_{2, 2}$ as
\[
\Expc \Lambda_{2, 2}
=~&
\Expc \bX_j^{\rmt}
\bM
\bX_{i} \bX_{i}^{\rmt}
\bM
\bX_{i}\bX_{i}^{\rmt}
\bX_j  +
\underbrace{\Expc \bX_j^{\rmt}
\bM
\bX_{\pi^{\natural}(i)} \bX_{\pi^{\natural}(i)}^{\rmt}
\bM
\bX_{\pi^{\natural}(j)}\bX_{j}^{\rmt}
\bX_j}_{0}  \\
+~&
\underbrace{\Expc \bX_j^{\rmt}
\bM
\bX_{\pi^{\natural}(i)} \bX_{\pi^{\natural}(i)}^{\rmt}
\bM
\bX_{j}\bX_{\pi^{\natural -1}(j)}^{\rmt}
\bX_j}_{0} \\
=~&
\Expc \bracket{\bX_i^{\rmt}\bM \bX_i}^2 =
\bracket{\trace(\bM)}^2 + \trace(\bM \bM)+ \trace\bracket{\bM^{\rmt}\bM}.
\]
Similarly, we can verify that both
$\Expc\Lambda_{2,3}$ and $\Expc \Lambda_{2, 4}$ are
zero and hence have
\begin{align}
\label{eq:xione_square_lambda2_sd_expc}
\Expc\Lambda_2
=~& (n-h)p\Bracket{1 + o(1)}\Bracket{\trace(\bM)}^2 +
2(n-h)p\Bracket{1 + o(1)}\trace(\bM\bM).
\end{align}

\vspace{0.1in}

\noindent\textbf{Case $(d, s)$: $i\neq \pi^{\natural}(i)$ and $j = \pi^{\natural}(j)$.}
We compute $\Lambda_{2, 1}$ as
\[
\Expc \Lambda_{2, 1} =~&
\underbrace{\Expc \bX_{\pi^{\natural}(i)}^{\rmt}
\bM
\bX_{\pi^{\natural}(i)} \bX_{\pi^{\natural}(i)}^{\rmt}
\bM
\bX_{\pi^{\natural}(i)}\bX_{i}^{\rmt}
\bX_{\pi^{\natural}(i)}}_{0}
+
\underbrace{\Expc \bX_{\pi^{\natural}(i)}^{\rmt}
\bM
\bX_{\pi^{\natural}(i)} \bX_{\pi^{\natural}(i)}^{\rmt}
\bM
\bX_{\pi^{\natural 2}(i)}\bX_{\pi^{\natural}(i)}^{\rmt}
\bX_{\pi^{\natural}(i)}}_{0}  \\
+~&
\underbrace{\Expc \bX_{\pi^{\natural}(i)}^{\rmt}
\bM
\bX_{\pi^{\natural}(i)} \bX_{\pi^{\natural}(i)}^{\rmt}
\bM
\bX_{j}\bX_{j}^{\rmt}
\bX_{\pi^{\natural}(i)}}_{
\Expc \bX_{\pi^{\natural}(i)}^{\rmt}
\bM
\bX_{\pi^{\natural}(i)} \bX_{\pi^{\natural}(i)}^{\rmt}
\bM
\bX_{\pi^{\natural}(i)}
} = \bracket{\trace(\bM)}^2 + \trace(\bM \bM)+ \trace\bracket{\bM^{\rmt}\bM}.
\]

We consider $\Lambda_{2, 2}$ as
\[
\Expc \Lambda_{2, 2} =~&
\underbrace{\Expc \bX_j^{\rmt}
\bM
\bX_{\pi^{\natural}(i)} \bX_{\pi^{\natural}(i)}^{\rmt}
\bM
\bX_{\pi^{\natural}(i)}\bX_{i}^{\rmt}
\bX_j}_{0}
+ \underbrace{\Expc \bX_j^{\rmt}
\bM
\bX_{\pi^{\natural}(i)} \bX_{\pi^{\natural}(i)}^{\rmt}
\bM
\bX_{\pi^{\natural 2}(i)}\bX_{\pi^{\natural}(i)}^{\rmt}
\bX_j}_{0} \\
+~& \underbrace{
\Expc \bX_j^{\rmt}
\bM
\bX_{\pi^{\natural}(i)} \bX_{\pi^{\natural}(i)}^{\rmt}
\bM
\bX_{j}\bX_{j}^{\rmt}
\bX_j}_{
\Expc \norm{\bX_j}{2}^2
\bX_j^{\rmt} \bM \bM \bX_j
}  = (p+2)\trace(\bM\bM).
\]
As for $\Expc \Lambda_{2, 3}$ and $\Lambda_{2, 4}$,
we can follow the same strategy and prove they are both
zero, which yields
\begin{align}
\label{eq:xione_square_lambda2_ds_expc}
\Expc \Lambda_2
=~& (n-h)\Bracket{\trace(\bM)}^2 + (n-h)p\Bracket{1 + o(1)}\trace(\bM\bM).
\end{align}

\vspace{0.1in}
\noindent\textbf{Case $(d, d)$: $i \neq \pi^{\natural}(i)$ and $j\neq \pi^{\natural}(j)$.}
Contrary to the previous cases, we have
$\Expc \Lambda_{2, 1}$ and $\Expc \Lambda_{2,2 }$ to be
zero in this case rather than $\Expc \Lambda_{2, 3}$ and $\Expc\Lambda_{2, 4}$.

Hence our focus turns to the calculation of
$\Expc \Lambda_{2, 3}$ and that of $\Expc \Lambda_{2, 4}$.
For $\Lambda_{2, 3}$, we have
\[
\Expc \Lambda_{2, 3} =~&
\underbrace{
\Expc \bX_{\pi^{\natural}(i)}^{\rmt}
\bM
\bX_{\pi^{\natural}(i)} \bX_{\pi^{\natural}(i)}^{\rmt}
\bM
\bX_{\pi^{\natural}(i)}\bX_{i}^{\rmt}
\bX_j}_{
p\Ind_{i = j}  \Expc \bX_{\pi^{\natural}(i)}^{\rmt}
\bM
\bX_{\pi^{\natural}(i)} \bX_{\pi^{\natural}(i)}^{\rmt}
\bM
\bX_{\pi^{\natural}(i)}
} +
\underbrace{\Expc \bX_{\pi^{\natural}(i)}^{\rmt}
\bM
\bX_{\pi^{\natural}(i)} \bX_{\pi^{\natural}(i)}^{\rmt}
\bM
\bX_{\pi^{\natural 2}(i)}\bX_{\pi^{\natural}(i)}^{\rmt}
\bX_j}_{
\Ind_{j = \pi^{\natural 2}(i)}
\Expc \bX_{\pi^{\natural}(i)}^{\rmt}
\bM
\bX_{\pi^{\natural}(i)} \bX_{\pi^{\natural}(i)}^{\rmt}
\bM
\bX_{\pi^{\natural}(i)}
} \\
+~&
\underbrace{
\Expc \bX_{\pi^{\natural}(i)}^{\rmt} \bM
\bX_{\pi^{\natural}(i)} \bX_{\pi^{\natural}(i)}^{\rmt}
\bM
\bX_{\pi^{\natural}(j)}\bX_{j}^{\rmt}
\bX_j}_{
p\Ind_{i=j}
\Expc \bX_{\pi^{\natural}(i)}^{\rmt}
\bM
\bX_{\pi^{\natural}(i)} \bX_{\pi^{\natural}(i)}^{\rmt}
\bM
\bX_{\pi^{\natural}(i)}
} +
\underbrace{\Expc \bX_{\pi^{\natural}(i)}^{\rmt}
\bM
\bX_{\pi^{\natural}(i)} \bX_{\pi^{\natural}(i)}^{\rmt}
\bM
\bX_{j}\bX_{\pi^{\natural -1}(j)}^{\rmt}
\bX_j}_{
\Ind_{j = \pi^{\natural 2}(i)}
\Expc \bX_{\pi^{\natural}(i)}^{\rmt}
\bM
\bX_{\pi^{\natural}(i)} \bX_{\pi^{\natural}(i)}^{\rmt}
\bM
\bX_{\pi^{\natural}(i)}
} \\
=~&
2\bracket{p\Ind_{i = j}  + \Ind_{j = \pi^{\natural 2}(i)}}
\Bracket{\bracket{\trace(\bM)}^2 + \trace(\bM \bM)+ \trace\bracket{\bM^{\rmt}\bM}}.
\]
Then we turn to the calculation of $\Expc \Lambda_{2, 4}$, which proceeds as
\[
\Expc \Lambda_{2, 4} =~&
\underbrace{
\Expc \bX_j^{\rmt}
\bM
\bX_{\pi^{\natural}(i)} \bX_{\pi^{\natural}(i)}^{\rmt}
\bM
{ 
\bX_{\pi^{\natural}(i)}\bX_{i}^{\rmt}
}
\bX_{\pi^{\natural}(i)}
}_{
\Ind(i = j) \Expc \bX_{\pi^{\natural}(i)}^{\rmt}\bM
\bX_{\pi^{\natural}(i)} \bX_{\pi^{\natural}(i)}^{\rmt}
\bM
\bX_{\pi^{\natural}(i)}
}
+
\underbrace{
\Expc \bX_j^{\rmt}
\bM
\bX_{\pi^{\natural}(i)} \bX_{\pi^{\natural}(i)}^{\rmt}
\bM
\bX_{\pi^{\natural 2}(i)}\bX_{\pi^{\natural}(i)}^{\rmt}
\bX_{\pi^{\natural}(i)}
}_{
\Ind_{j = \pi^{\natural 2}(i)}
\Expc \norm{\bX_{\pi^{\natural}(i)}}{2}^2 \bX_{\pi^{\natural}(i)}^{\rmt}\bM \bM \bX_{\pi^{\natural}(i)}
}  \\
+~&
\underbrace{
\Expc \bX_j^{\rmt}
\bM
\bX_{\pi^{\natural}(i)} \bX_{\pi^{\natural}(i)}^{\rmt}
\bM
\bX_{\pi^{\natural}(j)}\bX_{j}^{\rmt}
\bX_{\pi^{\natural}(i)}
}_{
\Ind(i = j) \Expc \bX_{\pi^{\natural}(i)}^{\rmt}\bM
\bX_{\pi^{\natural}(i)} \bX_{\pi^{\natural}(i)}^{\rmt}
\bM \bX_{\pi^{\natural}(i)}
}  +
\underbrace{\Expc \bX_j^{\rmt}
\bM
\bX_{\pi^{\natural}(i)} \bX_{\pi^{\natural}(i)}^{\rmt}
\bM
\bX_{j}\bX_{\pi^{\natural -1}(j)}^{\rmt}
\bX_{\pi^{\natural}(i)}}_{
\Ind_{j = \pi^{\natural 2}(i)}
\Expc \norm{\bX_{\pi^{\natural}(i)}}{2}^2 \bX_{\pi^{\natural}(i)}^{\rmt}\bM \bM \bX_{\pi^{\natural}(i)}
} \\
=~& 2\Ind_{j = \pi^{\natural 2}(i)}(p+2)\trace(\bM\bM)
+ 2\Ind(i =j) \Bracket{\bracket{\trace(\bM)}^2 + \trace(\bM \bM)+ \trace\bracket{\bM^{\rmt}\bM}}.
\]
Then we conclude
\begin{align}
\label{eq:xione_square_lambda2_dd_expc}
\Expc \Lambda_2 = -~& 2(n-h)\Bracket{(p+1)\Ind_{i = j} + \Ind_{j = \pi^{\natural 2}(i) }} \Bracket{\bracket{\trace(\bM)}^2 + \trace(\bM \bM)+ \trace\bracket{\bM^{\rmt}\bM}} \notag \\
-~& 2(n-h)(p+2)\Ind_{j = \pi^{\natural 2}(i)}\trace(\bM\bM).
\end{align}
The proof is thus completed by combining
~\eqref{eq:xione_square_lambda2_ss_expc},
~\eqref{eq:xione_square_lambda2_sd_expc}, ~\eqref{eq:xione_square_lambda2_ds_expc},
and ~\eqref{eq:xione_square_lambda2_dd_expc}.

\end{proof}

\begin{lemma}
\label{lemma:xione_square_lambda3_summary_expc}
We have
\[
& \Expc\bracket{\bX_{\pi^{\natural}(i)} - \bX_j}^{\rmt}
\bDelta \bM
\bX_{\pi^{\natural}(i)} \bX_{\pi^{\natural}(i)}^{\rmt}
\bM \bDelta^{\rmt}
\bracket{\bX_{\pi^{\natural}(i)} - \bX_j} \\
=~&
\bracket{1- \frac{h}{n}+o(1)}p^2\Bracket{\trace(\bM)}^2
+ \bracket{3- \frac{h}{n} + o(1)}p^2\trace(\bM\bM).
\]
\end{lemma}

\begin{proof}
We begin the proof by expanding $\Lambda_3$ as
\[
\Expc \Lambda_3 =~&
\Expc\underbrace{\bX_{\pi^{\natural}(i)}^{\rmt}
\bDelta \bM
\bX_{\pi^{\natural}(i)} \bX_{\pi^{\natural}(i)}^{\rmt}
\bM \bDelta^{\rmt}
\bX_{\pi^{\natural}(i)}}_{\Lambda_{3, 1}} +
\Expc \underbrace{\bX_j^{\rmt}
\bDelta \bM
\bX_{\pi^{\natural}(i)} \bX_{\pi^{\natural}(i)}^{\rmt}
\bM \bDelta^{\rmt} \bX_j}_{\Lambda_{3,2}} \\
-~& \Expc\underbrace{\bX_{\pi^{\natural}(i)}^{\rmt}
\bDelta \bM
\bX_{\pi^{\natural}(i)} \bX_{\pi^{\natural}(i)}^{\rmt}
\bM \bDelta^{\rmt}\bX_j}_{\Lambda_{3,3}} -
\Expc \underbrace{\bX_j^{\rmt}
\bDelta \bM
\bX_{\pi^{\natural}(i)} \bX_{\pi^{\natural}(i)}^{\rmt}
\bM \bDelta^{\rmt}\bX_{\pi^{\natural}(i)}}_{\Lambda_{3, 4}}.
\]

\noindent\textbf{Case $(s, s)$: $i = \pi^{\natural}(i)$ and $j = \pi^{\natural}(j)$.}
First we compute $\Lambda_{3, 1}$ as
\[
\Expc\Lambda_{3, 1} =~&
\underbrace{\Expc \bX_{i}^{\rmt}
\bX_i \bX_i^{\rmt}
\bM
\bX_{i} \bX_{i}^{\rmt}
\bM \bX_i \bX_i^{\rmt}
\bX_{i}}_{
\Expc \norm{\bX_i}{2}^4
\bracket{\bX_i^{\rmt}\bM \bX_i}^2
}
+
\underbrace{
\Expc \bX_{i}^{\rmt}
\bX_i \bX_i^{\rmt}
\bM
\bX_{i} \bX_{i}^{\rmt}
\bM
\bX_j \bX_j^{\rmt}
\bX_{i}}_{
\Expc \norm{\bX_i}{2}^2
\bracket{\bX_i^{\rmt}\bM \bX_i}^2
} \\
+~&
\underbrace{
\Expc \bX_{i}^{\rmt}
 \bX_j \bX_j^{\rmt}
\bM
\bX_{i} \bX_{i}^{\rmt}
\bM
\bX_i \bX_i^{\rmt}
\bX_{i}}_{
\Expc \norm{\bX_i}{2}^2
\bracket{\bX_i^{\rmt}\bM \bX_i}^2
} + \underbrace{
\Expc \bX_{i}^{\rmt}
\bX_j \bX_j^{\rmt}
\bM
\bX_{i} \bX_{i}^{\rmt}
\bM
\bX_j \bX_j^{\rmt}
\bX_{i}}_{
\Expc \bracket{\bX_j^{\rmt}\bX_i}^2
\bX_j^{\rmt}\bM \bX_i \bX_i^{\rmt}\bM \bX_j
} \\
=~& \bracket{p+4}\bracket{p+8}\Bracket{\bracket{\trace(\bM)}^2 + 2\trace(\bM \bM)} + 2\bracket{\trace(\bM)}^2 + (p+6)\trace\bracket{\bM \bM}.
\]
Then, we consider $\Lambda_{3, 2}$ as
\[
\Expc \Lambda_{3, 2} =~&
\underbrace{
\Expc \bX_j^{\rmt}
\bX_i \bX_i^{\rmt}
\bM
\bX_{i} \bX_{i}^{\rmt}
\bM
\bX_i \bX_i^{\rmt}
\bX_j
}_{
\Expc \norm{\bX_i}{2}^2
\bracket{\bX_i^{\rmt} \bM \bX_i}^2
}
+
\underbrace{\Expc \bX_j^{\rmt}
\bX_i \bX_i^{\rmt}
\bM
\bX_{i} \bX_{i}^{\rmt}
\bM
\bX_j \bX_j^{\rmt}
\bX_j}_{
(p+2)
\Expc \bracket{\bX_i^{\rmt}\bM \bX_i}^2
} \\
+~&
\underbrace{\Expc \bX_j^{\rmt}
 \bX_j \bX_j^{\rmt}
\bM
\bX_{i} \bX_{i}^{\rmt}
\bM
\bX_i \bX_i^{\rmt}
\bX_j}_{
(p+2)
\Expc\bracket{\bX_i^{\rmt}\bM \bX_i}^2
} +
\underbrace{\Expc
\bX_j^{\rmt}
\bX_j \bX_j^{\rmt}
\bM
\bX_{i} \bX_{i}^{\rmt}
\bM \bX_j \bX_j^{\rmt}
\bX_j}_{
\Expc \norm{\bX_j}{2}^4
\bX_j^{\rmt}\bM\bM\bX_j
} \\
=~& \bracket{3p+ 8}\Bracket{\bracket{\trace(\bM)}^2 + \trace(\bM \bM)+ \trace\bracket{\bM^{\rmt}\bM}} +
(p+2)(p+4)\trace(\bM\bM).
\]
In addition, we can verify that
$\Expc \Lambda_{3, 3}$ and $\Expc \Lambda_{3, 4}$ are both
zero. Hence we conclude
\begin{align}
\label{eq:xione_square_lambda3_ss_expc}	
\Expc \Lambda_3
= \bracket{p^2 + 15p+42}\Bracket{\trace(\bM)}^2
+ \bracket{3p^2 + 37p+94}\trace(\bM\bM).
\end{align}

\newpage

\noindent\textbf{Case $(s, d)$: $i = \pi^{\natural}(i)$ and $j \neq \pi^{\natural}(j)$.} We can compute $\Lambda_{3, 1}$ as
\[
\Expc \Lambda_{3, 1}
=~&
\underbrace{\Expc \bX_{i}^{\rmt}
\bX_{i}\bX_{i}^{\rmt}
\bM
\bX_{i} \bX_{i}^{\rmt}
\bM
\bX_{i}\bX_{i}^{\rmt}
\bX_{i}}_{
\Expc \norm{\bX_i}{2}^4
\bracket{\bX_i^{\rmt}\bM \bX_i}^2
} +
\underbrace{\Expc \bX_{i}^{\rmt}
\bX_{i}\bX_{i}^{\rmt}
\bM
\bX_{i} \bX_{i}^{\rmt}
\bM
\bX_{\pi^{\natural}(j)} \bX_{j}^{\rmt}
\bX_{i}}_{0} \\
+~&
\underbrace{\Expc \bX_{i}^{\rmt}
\bX_{i}\bX_{i}^{\rmt}
\bM
\bX_{i} \bX_{i}^{\rmt}
\bM
\bX_{j}\bX_{\pi^{\natural -1}(j)}^{\rmt}
\bX_{i}}_{0}
+
\underbrace{
\Expc \bX_{i}^{\rmt}
\bX_{j}\bX_{\pi^{\natural}(j)}^{\rmt}
\bM
\bX_{i} \bX_{i}^{\rmt}
\bM
\bX_{i}\bX_{i}^{\rmt}
\bX_{i}}_{0} \\
+~&
\underbrace{\Expc \bX_{i}^{\rmt}
\bX_{j}\bX_{\pi^{\natural}(j)}^{\rmt}
\bM
\bX_{i} \bX_{i}^{\rmt}
\bM
\bX_{\pi^{\natural}(j)} \bX_{j}^{\rmt}
\bX_{i}}_{
\Expc \norm{\bX_i}{2}^2 \bX_i^{\rmt}\bM \bM \bX_i
}
+
\underbrace{\Expc \bX_{i}^{\rmt} 
\bX_{j}\bX_{\pi^{\natural}(j)}^{\rmt}
\bM
\bX_{i} \bX_{i}^{\rmt}
\bM
\bX_{j}\bX_{\pi^{\natural -1}(j)}^{\rmt}
\bX_{i}}_{
\Ind_{j = \pi^{\natural 2}(j)}
\Expc\bracket{\bX_i^{\rmt}\bM \bX_i}^2
} \\
+~&
\underbrace{
\Expc \bX_{i}^{\rmt}
 \bX_{\pi^{\natural -1}(j)}\bX_{j}^{\rmt}
\bM
\bX_{i} \bX_{i}^{\rmt}
\bM
\bX_{i}\bX_{i}^{\rmt}
\bX_{i}}_{0}
+
\underbrace{\Expc \bX_{i}^{\rmt} 
\bX_{\pi^{\natural -1}(j)}\bX_{j}^{\rmt}
\bM
\bX_{i} \bX_{i}^{\rmt}
\bM
\bX_{\pi^{\natural}(j)} \bX_{j}^{\rmt}
\bX_{i}}_{
\Ind_{j = \pi^{\natural 2}(j)}
\Expc\bracket{\bX_i^{\rmt}\bM \bX_i}^2
}  \\
+ ~&
\underbrace{\Expc \bX_{i}^{\rmt}
\bX_{\pi^{\natural -1}(j)}\bX_{j}^{\rmt}
\bM
\bX_{i} \bX_{i}^{\rmt}
\bM \bX_{j}\bX_{\pi^{\natural -1}(j)}^{\rmt}
\bX_{i}}_{
\Expc \norm{\bX_i}{2}^2 \bX_i^{\rmt}\bM \bM \bX_i
} \\
=~&
\bracket{p^2 + 10p + 24 + 2\Ind_{j = \pi^{\natural 2}(j)}}
\Bracket{\bracket{\trace(\bM)}^2 + 2\trace(\bM \bM)} + 2(p+2)\trace(\bM\bM).
\]
We consider $\Lambda_{3, 2}$ as
\[
\Expc \Lambda_{3, 2} =~&
\underbrace{
\Expc \bX_j^{\rmt}
\bX_{i}\bX_{i}^{\rmt}
\bM
\bX_{i} \bX_{i}^{\rmt}
\bM
\bX_{i}\bX_{i}^{\rmt}
\bX_j}_{
\Expc \norm{\bX_i}{2}^2
\bracket{\bX_{i}^{\rmt}  \bM \bX_{i}}^2
}  +
\underbrace{\Expc \bX_j^{\rmt}
\bX_{i}\bX_{i}^{\rmt}
\bM
\bX_{i} \bX_{i}^{\rmt}
\bM
\bX_{\pi^{\natural}(j)} \bX_{j}^{\rmt}
\bX_j}_{0}  \\
+~&
\underbrace{\Expc \bX_j^{\rmt}
\bX_{i}\bX_{i}^{\rmt}
\bM
\bX_{i} \bX_{i}^{\rmt}
\bM
\bX_{j}\bX_{\pi^{\natural -1}(j)}^{\rmt}
\bX_j}_{0}  \\
+~&
\underbrace{\Expc \bX_j^{\rmt}
\bX_{j}\bX_{\pi^{\natural}(j)}^{\rmt}
\bM
\bX_{i} \bX_{i}^{\rmt}
\bM
\bX_{i}\bX_{i}^{\rmt}
\bX_j}_{0}  +
\underbrace{\Expc \bX_j^{\rmt}
\bX_{j}\bX_{\pi^{\natural}(j)}^{\rmt}
\bM
\bX_{i} \bX_{i}^{\rmt}
\bM
\bX_{\pi^{\natural}(j)} \bX_{j}^{\rmt}
\bX_j}_{
\Expc \norm{\bX_i}{2}^4 \trace(\bM\bM)
}  \\
+~&
\underbrace{\Expc \bX_j^{\rmt} 
\bX_{j}\bX_{\pi^{\natural}(j)}^{\rmt}
\bM
\bX_{i} \bX_{i}^{\rmt}
\bM
\bX_{j}\bX_{\pi^{\natural -1}(j)}^{\rmt}
\bX_j}_{
\Ind_{j = \pi^{\natural 2}(j)}
\Expc \norm{\bX_j}{2}^2 \bX_j^{\rmt}\bM \bM \bX_j
}
+
\underbrace{\Expc \bX_j^{\rmt}
\bX_{\pi^{\natural -1}(j)}\bX_{j}^{\rmt}
\bM
\bX_{i} \bX_{i}^{\rmt}
\bM
\bX_{i}\bX_{i}^{\rmt}
\bX_j}_{0}  \\
+~&
\underbrace{\Expc \bX_j^{\rmt}
\bX_{\pi^{\natural -1}(j)}\bX_{j}^{\rmt}
\bM
\bX_{i} \bX_{i}^{\rmt}
\bM
\bX_{\pi^{\natural}(j)} \bX_{j}^{\rmt}
\bX_j }_{
\Ind_{j = \pi^{\natural 2}(j)}
\Expc \norm{\bX_j}{2}^2 \bX_j^{\rmt}\bM \bM \bX_j
}
+
\underbrace{\Expc \bX_j^{\rmt}
{ 
\bX_{\pi^{\natural -1}(j)}\bX_{j}^{\rmt}
} \bM
\bX_{i} \bX_{i}^{\rmt}
\bM
{
\bX_{j}\bX_{\pi^{\natural -1}(j)}^{\rmt}
}
\bX_j}_{\Expc \norm{\bX_j}{2}^2 \bX_j^{\rmt}\bM\bM \bX_j}
\\
=~&
(p+4)\Bracket{\bracket{\trace(\bM)}^2 + 2\trace(\bM \bM)} +
(p+2)\bracket{p+1 + 2\Ind_{j = \pi^{\natural 2}(j)}}\trace(\bM\bM).
\]
As for $\Lambda_{3, 3}$ and $\Lambda_{3, 4}$, we can prove that
they are both zero in this case. Hence we conclude
\begin{align}
\label{eq:xione_square_lambda3_sd_expc}
\Expc \Lambda_3
=~& \bracket{p^2 + 11p + 28 + 2\Ind_{j = \pi^{\natural 2}(j)}}\bracket{\trace(\bM)}^2 \notag \\
+~& \Bracket{3p^2 + 27p + 62 + 2\bracket{p+4}\Ind_{j = \pi^{\natural 2}(j)}}\trace(\bM\bM).
\end{align}

\newpage

\noindent\textbf{Case $(d, s)$: $i\neq \pi^{\natural}(i)$ and $j = \pi^{\natural}(j)$.}
We consider the term $\Lambda_{3, 1}$ as
\[
\Expc \Lambda_{3, 1}
=~&
\underbrace{\Expc\bX_{\pi^{\natural}(i)}^{\rmt}
\bX_{i}\bX_{\pi^{\natural}(i)}^{\rmt}
\bM
\bX_{\pi^{\natural}(i)} \bX_{\pi^{\natural}(i)}^{\rmt}
\bM
\bX_{\pi^{\natural}(i)}\bX_{i}^{\rmt}
\bX_{\pi^{\natural}(i)}}_{
\Expc \Fnorm{\bX_{\pi^{\natural}(i)}}^2
\bracket{\bX_{\pi^{\natural}(i)}^{\rmt} \bM \bX_{\pi^{\natural}(i)}}^2
} \\
+~&
\underbrace{
\Expc\bX_{\pi^{\natural}(i)}^{\rmt}
\bX_{i}\bX_{\pi^{\natural}(i)}^{\rmt}
\bM
\bX_{\pi^{\natural}(i)} \bX_{\pi^{\natural}(i)}^{\rmt}
\bM
\bX_{\pi^{\natural 2}(i)}\bX_{\pi^{\natural}(i)}^{\rmt}
\bX_{\pi^{\natural}(i)}
}_{
\Ind_{i = \pi^{\natural 2}(i)}
\Expc \Fnorm{\bX_{\pi^{\natural}(i)}}^2
\bracket{\bX_{\pi^{\natural}(i)}^{\rmt} \bM \bX_{\pi^{\natural}(i)}}^2
} \\
+~&
\underbrace{\Expc\bX_{\pi^{\natural}(i)}^{\rmt}
\bX_{i}\bX_{\pi^{\natural}(i)}^{\rmt}
\bM
\bX_{\pi^{\natural}(i)} \bX_{\pi^{\natural}(i)}^{\rmt}
\bM
\bX_{j}\bX_{j}^{\rmt}
\bX_{\pi^{\natural}(i)}}_{0} \\
+~&
\underbrace{
\Expc\bX_{\pi^{\natural}(i)}^{\rmt} 
\bX_{\pi^{\natural}(i)}\bX_{\pi^{\natural 2}(i)}^{\rmt}
\bM
\bX_{\pi^{\natural}(i)} \bX_{\pi^{\natural}(i)}^{\rmt}
\bM
\bX_{\pi^{\natural}(i)}\bX_{i}^{\rmt}
\bX_{\pi^{\natural}(i)}}_{
\Ind_{i = \pi^{\natural 2}(i)}
\Expc \Fnorm{\bX_{\pi^{\natural}(i)}}^2
\bracket{\bX_{\pi^{\natural}(i)}^{\rmt} \bM \bX_{\pi^{\natural}(i)}}^2
} \\
+~&
\underbrace{
\Expc\bX_{\pi^{\natural}(i)}^{\rmt}
\bX_{\pi^{\natural}(i)}\bX_{\pi^{\natural 2}(i)}^{\rmt}
\bM
\bX_{\pi^{\natural}(i)} \bX_{\pi^{\natural}(i)}^{\rmt}
\bM
\bX_{\pi^{\natural 2}(i)}\bX_{\pi^{\natural}(i)}^{\rmt}
\bX_{\pi^{\natural}(i)}}_{
\Expc \norm{\bX_i}{2}^4
\bX_i^{\rmt}\bM \bM \bX_i
} \\
+~&
\underbrace{\Expc\bX_{\pi^{\natural}(i)}^{\rmt}
\bX_{\pi^{\natural}(i)}\bX_{\pi^{\natural 2}(i)}^{\rmt}
\bM
\bX_{\pi^{\natural}(i)} \bX_{\pi^{\natural}(i)}^{\rmt}
\bM
\bX_{j}\bX_{j}^{\rmt}
\bX_{\pi^{\natural}(i)}}_{0}\\
+~&
\underbrace{\Expc\bX_{\pi^{\natural}(i)}^{\rmt} 
\bX_{j}\bX_{j}^{\rmt}
\bM
\bX_{\pi^{\natural}(i)} \bX_{\pi^{\natural}(i)}^{\rmt}
\bM
\bX_{\pi^{\natural}(i)}\bX_{i}^{\rmt}
\bX_{\pi^{\natural}(i)}}_{0} \\
+~&
\underbrace{\Expc\bX_{\pi^{\natural}(i)}^{\rmt} 
\bX_{j}\bX_{j}^{\rmt}
\bM
\bX_{\pi^{\natural}(i)} \bX_{\pi^{\natural}(i)}^{\rmt}
\bM
\bX_{\pi^{\natural 2}(i)}\bX_{\pi^{\natural}(i)}^{\rmt}
\bX_{\pi^{\natural}(i)}}_{0} \\
+~&
\underbrace{\Expc\bX_{\pi^{\natural}(i)}^{\rmt} 
\bX_{j}\bX_{j}^{\rmt}
\bM
\bX_{\pi^{\natural}(i)} \bX_{\pi^{\natural}(i)}^{\rmt}
\bM
\bX_{j}\bX_{j}^{\rmt}
\bX_{\pi^{\natural}(i)}}_{
\Expc\bracket{\bX_i^{\rmt}\bX_j}^2
\bX_j^{\rmt}\bM \bX_i \bX_i^{\rmt}\bM \bX_j
} \\
=~& \bracket{p+6} \bracket{\trace(\bM)}^2 +
\bracket{p^2 + 9p + 22}\trace(\bM\bM)
+ 2\Ind_{i = \pi^{\natural 2}(i)}\bracket{p+4}\Bracket{\bracket{\trace(\bM)}^2 + 2\trace(\bM \bM)}.
\]
We consider the term $\Lambda_{3, 2}$ as
\[
\Expc \Lambda_{3, 2} =~&
\underbrace{
\Expc \bX_j^{\rmt}
\bX_{i}\bX_{\pi^{\natural}(i)}^{\rmt}
\bM
\bX_{\pi^{\natural}(i)} \bX_{\pi^{\natural}(i)}^{\rmt}
\bM
\bX_{\pi^{\natural}(i)}\bX_{i}^{\rmt}
\bX_j}_{
p \times \Expc  
\bracket{\bX_{\pi^{\natural}(i)}^{\rmt}
\bM
\bX_{\pi^{\natural}(i)}}^2
} \\
+~&
\underbrace{\Expc \bX_j^{\rmt}
\bX_{i}\bX_{\pi^{\natural}(i)}^{\rmt}
\bM
\bX_{\pi^{\natural}(i)} \bX_{\pi^{\natural}(i)}^{\rmt}
\bM
\bX_{\pi^{\natural 2}(i)}\bX_{\pi^{\natural}(i)}^{\rmt}
\bX_j}_{
\Ind_{i = \pi^{\natural 2}(i)}
\Expc \bracket{\bX_i^{\rmt}\bM \bX_i}^2
} \\
+~&
\underbrace{\Expc \bX_j^{\rmt}
\bX_{i}\bX_{\pi^{\natural}(i)}^{\rmt}
\bM
\bX_{\pi^{\natural}(i)} \bX_{\pi^{\natural}(i)}^{\rmt}
\bM
\bX_{j}\bX_{j}^{\rmt}
\bX_j}_{0} \\
+~&
\underbrace{
\Expc \bX_j^{\rmt}
\bX_{\pi^{\natural}(i)}\bX_{\pi^{\natural 2}(i)}^{\rmt}
\bM
\bX_{\pi^{\natural}(i)} \bX_{\pi^{\natural}(i)}^{\rmt}
\bM \bX_{\pi^{\natural}(i)}\bX_{i}^{\rmt}
\bX_j}_{
\Ind_{i = \pi^{\natural 2}(i)}
\Expc\bracket{\bX_i^{\rmt}\bM \bX_i}^2
} \\
+~&
\underbrace{\Expc \bX_j^{\rmt}
\bX_{\pi^{\natural}(i)}\bX_{\pi^{\natural 2}(i)}^{\rmt}
\bM
\bX_{\pi^{\natural}(i)} \bX_{\pi^{\natural}(i)}^{\rmt}
\bM
\bX_{\pi^{\natural 2}(i)}\bX_{\pi^{\natural}(i)}^{\rmt}
\bX_j}_{
\Expc \norm{\bX_i}{2}^2 \bX_i^{\rmt}\bM\bM\bX_i
} \\
+~&
\underbrace{\Expc \bX_j^{\rmt}
\bX_{\pi^{\natural}(i)}\bX_{\pi^{\natural 2}(i)}^{\rmt}
\bM
\bX_{\pi^{\natural}(i)} \bX_{\pi^{\natural}(i)}^{\rmt}
\bM
\bX_{j}\bX_{j}^{\rmt}
\bX_j}_{0} \\
+~&
\underbrace{\Expc \bX_j^{\rmt}
\bX_{j}\bX_{j}^{\rmt}
\bM
\bX_{\pi^{\natural}(i)} \bX_{\pi^{\natural}(i)}^{\rmt}
\bM
\bX_{\pi^{\natural}(i)}\bX_{i}^{\rmt}
\bX_j}_{0} \\
+~&
\underbrace{\Expc \bX_j^{\rmt} 
\bX_{j}\bX_{j}^{\rmt}
\bM
\bX_{\pi^{\natural}(i)} \bX_{\pi^{\natural}(i)}^{\rmt}
\bM
\bX_{\pi^{\natural 2}(i)}\bX_{\pi^{\natural}(i)}^{\rmt}
\bX_j}_{0} \\
+~&
\underbrace{\Expc \bX_j^{\rmt}
\bX_{j}\bX_{j}^{\rmt}
\bM
\bX_{\pi^{\natural}(i)} \bX_{\pi^{\natural}(i)}^{\rmt}
\bM
\bX_{j}\bX_{j}^{\rmt}
\bX_j}_{\Expc \Fnorm{\bX_j}^4 \times \bX_j^{\rmt}\bM \bM \bX_j} \\
=~& p\bracket{\trace(\bM)}^2 + \bracket{p^2 + 9p + 10}\trace(\bM\bM)
+ 2\Ind_{i = \pi^{\natural 2}(i)}\Bracket{\bracket{\trace(\bM)}^2 + 2\trace(\bM\bM)}.
\]
As for $\Lambda_{3,3}$ and $\Lambda_{3, 4}$, easily we can verify
they are both zero and hence
\begin{align}
\label{eq:xione_square_lambda3_ds_expc}
\Expc \Lambda_3 =~&
2\bracket{p+3}\bracket{\trace(\bM)}^2 \
+ 2\bracket{p^2 + 9p + 16}\trace(\bM\bM) \notag \\
+~& 2\Ind_{i = \pi^{\natural 2}(i)}\bracket{p+5}\Bracket{\bracket{\trace(\bM)}^2 + 2\trace(\bM\bM)}.
\end{align}

\noindent\textbf{Case $(d, d)$: $i \neq \pi^{\natural}(i)$ and $j\neq \pi^{\natural}(j)$.}
First, We compute $\Expc \Lambda_{3, 1}$ as
\begin{align}
\label{eq:xione_square_lambda31_dd_expc}
\Expc \Lambda_{3,1} =~&
\underbrace{\Expc \bX_{\pi^{\natural}(i)}^{\rmt}
\bX_{i}\bX_{\pi^{\natural}(i)}^{\rmt}
\bM
\bX_{\pi^{\natural}(i)} \bX_{\pi^{\natural}(i)}^{\rmt}
\bM
\bX_{\pi^{\natural}(i)}\bX_{i}^{\rmt}
\bX_{\pi^{\natural}(i)}}_{
\Expc \Fnorm{\bX_{\pi^{\natural}(i)}}^2 \bracket{\bX_{\pi^{\natural}(i)}^{\rmt}\bM \bX_{\pi^{\natural}(i)} }^2
}\notag  \\
+~&
\underbrace{\Expc \bX_{\pi^{\natural}(i)}^{\rmt}
\bX_{i}\bX_{\pi^{\natural}(i)}^{\rmt}
\bM
\bX_{\pi^{\natural}(i)} \bX_{\pi^{\natural}(i)}^{\rmt}
\bM
\bX_{\pi^{\natural 2}(i)}\bX_{\pi^{\natural}(i)}^{\rmt}
\bX_{\pi^{\natural}(i)}}_{
\Ind_{i = \pi^{\natural 2}(i)}
\Expc \Fnorm{\bX_{\pi^{\natural}(i)}}^2 \bracket{\bX_{\pi^{\natural}(i)}^{\rmt}\bM \bX_{\pi^{\natural}(i)} }^2
} \notag \\
+~&
\underbrace{
\Expc \bX_{\pi^{\natural}(i)}^{\rmt} 
\bX_{i}\bX_{\pi^{\natural}(i)}^{\rmt}
\bM
\bX_{\pi^{\natural}(i)} \bX_{\pi^{\natural}(i)}^{\rmt}
\bM
\bX_{\pi^{\natural}(j)}\bX_{j}^{\rmt}
\bX_{\pi^{\natural}(i)}}_{
\Ind_{i = j}
\Expc \Fnorm{\bX_{\pi^{\natural}(i)}}^2 \bracket{\bX_{\pi^{\natural}(i)}^{\rmt}\bM \bX_{\pi^{\natural}(i)} }^2
} \notag \\
+~& \underbrace{\Expc \bX_{\pi^{\natural}(i)}^{\rmt}
\bX_{i}\bX_{\pi^{\natural}(i)}^{\rmt}
\bM
\bX_{\pi^{\natural}(i)} \bX_{\pi^{\natural}(i)}^{\rmt}
\bM
\bX_{j}\bX_{\pi^{\natural -1}(j)}^{\rmt}
\bX_{\pi^{\natural}(i)}}_{
\Ind_{j = i}\Ind_{j = \pi^{\natural 2}(i)}
\Expc \Fnorm{\bX_{\pi^{\natural}(i)}}^2 \bracket{\bX_{\pi^{\natural}(i)}^{\rmt}\bM \bX_{\pi^{\natural}(i)} }^2
} \notag \\
+~&
\underbrace{
\Expc \bX_{\pi^{\natural}(i)}^{\rmt}
\bX_{\pi^{\natural}(i)}\bX_{\pi^{\natural 2}(i)}^{\rmt}
\bM
\bX_{\pi^{\natural}(i)} \bX_{\pi^{\natural}(i)}^{\rmt}
\bM
\bX_{\pi^{\natural}(i)}\bX_{i}^{\rmt}
\bX_{\pi^{\natural}(i)}}_{
\Ind_{i = \pi^{\natural 2}(i)}
\Expc \Fnorm{\bX_{\pi^{\natural}(i)}}^2 \bracket{\bX_{\pi^{\natural}(i)}^{\rmt}\bM \bX_{\pi^{\natural}(i)} }^2
} \notag  \\
+~& \underbrace{
\Expc \bX_{\pi^{\natural}(i)}^{\rmt}
\bX_{\pi^{\natural}(i)}\bX_{\pi^{\natural 2}(i)}^{\rmt}
\bM
\bX_{\pi^{\natural}(i)} \bX_{\pi^{\natural}(i)}^{\rmt}
\bM
\bX_{\pi^{\natural 2}(i)}\bX_{\pi^{\natural}(i)}^{\rmt}
\bX_{\pi^{\natural}(i)}}_{
\Expc \Fnorm{\bX_{\pi^{\natural}(i)}}^4
\bracket{\bX_{\pi^{\natural}(i)}^{\rmt}\bM\bM \bX_{\pi^{\natural}(i)}}
} \notag \\
+~&
\underbrace{
\Expc \bX_{\pi^{\natural}(i)}^{\rmt}
\bX_{\pi^{\natural}(i)}\bX_{\pi^{\natural 2}(i)}^{\rmt}
\bM
\bX_{\pi^{\natural}(i)} \bX_{\pi^{\natural}(i)}^{\rmt}
\bM
\bX_{\pi^{\natural}(j)}\bX_{j}^{\rmt}
\bX_{\pi^{\natural}(i)}}_{
\Ind_{j = i}\Ind_{j = \pi^{\natural 2}(i)}
\Expc \Fnorm{\bX_{\pi^{\natural}(i)}}^2 \bracket{\bX_{\pi^{\natural}(i)}^{\rmt}\bM \bX_{\pi^{\natural}(i)} }^2
} \notag \\
+~& \underbrace{
\Expc \bX_{\pi^{\natural}(i)}^{\rmt} 
\bX_{\pi^{\natural}(i)}\bX_{\pi^{\natural 2}(i)}^{\rmt}
\bM
\bX_{\pi^{\natural}(i)} \bX_{\pi^{\natural}(i)}^{\rmt}
\bM
\bX_{j}\bX_{\pi^{\natural -1}(j)}^{\rmt}
\bX_{\pi^{\natural}(i)}}_{
\Ind_{j = \pi^{\natural 2}(i)}
\Expc \Fnorm{\bX_{\pi^{\natural}(i)}}^4
\bracket{\bX_{\pi^{\natural}(i)}^{\rmt}\bM\bM \bX_{\pi^{\natural}(i)}}
} \notag \\
+~&
\underbrace{
\Expc \bX_{\pi^{\natural}(i)}^{\rmt}
\bX_{j}\bX_{\pi^{\natural}(j)}^{\rmt}
\bM
\bX_{\pi^{\natural}(i)} \bX_{\pi^{\natural}(i)}^{\rmt}
\bM
\bX_{\pi^{\natural}(i)}\bX_{i}^{\rmt}
\bX_{\pi^{\natural}(i)}}_{
\Ind_{i = j}
\Expc \Fnorm{\bX_{\pi^{\natural}(i)}}^2 \bracket{\bX_{\pi^{\natural}(i)}^{\rmt}\bM \bX_{\pi^{\natural}(i)} }^2
} \notag \\
+~&
\underbrace{
\Expc \bX_{\pi^{\natural}(i)}^{\rmt}
 \bX_{j}\bX_{\pi^{\natural}(j)}^{\rmt}
\bM
\bX_{\pi^{\natural}(i)} \bX_{\pi^{\natural}(i)}^{\rmt}
\bM
\bX_{\pi^{\natural 2}(i)}\bX_{\pi^{\natural}(i)}^{\rmt}
\bX_{\pi^{\natural}(i)}}_{
\Ind_{j = i} \Ind_{j = \pi^{\natural 2}(i)}
\Expc \Fnorm{\bX_{\pi^{\natural}(i)}}^2 \bracket{\bX_{\pi^{\natural}(i)}^{\rmt}\bM \bX_{\pi^{\natural}(i)} }^2
} \notag \\
+~&
\underbrace{\Expc \bX_{\pi^{\natural}(i)}^{\rmt}
\bX_{j}\bX_{\pi^{\natural}(j)}^{\rmt}
\bM
\bX_{\pi^{\natural}(i)} \bX_{\pi^{\natural}(i)}^{\rmt}
\bM
\bX_{\pi^{\natural}(j)}\bX_{j}^{\rmt}
\bX_{\pi^{\natural}(i)}}_{
\Ind_{i = j} \Expc \Fnorm{\bX_{\pi^{\natural}(i)}}^2 \bracket{\bX_{\pi^{\natural}(i)}^{\rmt}\bM \bX_{\pi^{\natural}(i)} }^2
+ \Ind_{i\neq j}
\Expc \Fnorm{\bX_{\pi^{\natural}(i)}}^2
\bX_{\pi^{\natural}(i)}^{\rmt}\bM \bM \bX_{\pi^{\natural}(i)}
} \notag \\
+~&
\underbrace{\Expc \bX_{\pi^{\natural}(i)}^{\rmt}
 \bX_{j}\bX_{\pi^{\natural}(j)}^{\rmt}
\bM
\bX_{\pi^{\natural}(i)} \bX_{\pi^{\natural}(i)}^{\rmt}
\bM
 \bX_{j}\bX_{\pi^{\natural -1}(j)}^{\rmt}
\bX_{\pi^{\natural}(i)}}_{
\Ind_{j = \pi^{\natural 2}(j)}
\Bracket{
\Ind_{i = j}\Expc  \Fnorm{\bX_{\pi^{\natural}(i)}}^2 \bracket{\bX_{\pi^{\natural}(i)}^{\rmt}\bM \bX_{\pi^{\natural}(i)} }^2
+ \Ind_{i \neq j} \Expc\bracket{\bX_{\pi^{\natural}(i)}^{\rmt}\bM \bX_{\pi^{\natural}(i)} }^2
}
}\notag  \\
+~&
\underbrace{\Expc \bX_{\pi^{\natural}(i)}^{\rmt} 
\bX_{\pi^{\natural -1}(j)}\bX_{j}^{\rmt}
\bM
\bX_{\pi^{\natural}(i)} \bX_{\pi^{\natural}(i)}^{\rmt}
\bM
\bX_{\pi^{\natural}(i)}\bX_{i}^{\rmt}
\bX_{\pi^{\natural}(i)}}_{
\Ind_{i = j}\Ind_{j = \pi^{\natural 2}(i)}
\Expc \Fnorm{\bX_{\pi^{\natural}(i)}}^2 \bracket{\bX_{\pi^{\natural}(i)}^{\rmt}\bM \bX_{\pi^{\natural}(i)} }^2
}  \notag \\
+~&
\underbrace{\Expc \bX_{\pi^{\natural}(i)}^{\rmt}
\bX_{\pi^{\natural -1}(j)}\bX_{j}^{\rmt}
\bM
\bX_{\pi^{\natural}(i)} \bX_{\pi^{\natural}(i)}^{\rmt}
\bM
\bX_{\pi^{\natural 2}(i)}\bX_{\pi^{\natural}(i)}^{\rmt}
\bX_{\pi^{\natural}(i)}}_{
\Ind_{j = \pi^{\natural 2}(i)}
\Expc \Fnorm{\bX_{\pi^{\natural}(i)}}^4
\bracket{\bX_{\pi^{\natural}(i)}^{\rmt}\bM\bM \bX_{\pi^{\natural}(i)}}
}\notag \\
+~& \underbrace{
\Expc \bX_{\pi^{\natural}(i)}^{\rmt}
\bX_{\pi^{\natural -1}(j)}\bX_{j}^{\rmt}
\bM
\bX_{\pi^{\natural}(i)} \bX_{\pi^{\natural}(i)}^{\rmt}
\bM
\bX_{\pi^{\natural}(j)}\bX_{j}^{\rmt}
\bX_{\pi^{\natural}(i)}}_{
\Ind_{j = \pi^{\natural 2}(j)}
\Bracket{
\Ind_{i = j}\Expc  \Fnorm{\bX_{\pi^{\natural}(i)}}^2 \bracket{\bX_{\pi^{\natural}(i)}^{\rmt}\bM \bX_{\pi^{\natural}(i)} }^2
+ \Ind_{i \neq j} \Expc\bracket{\bX_{\pi^{\natural}(i)}^{\rmt}\bM \bX_{\pi^{\natural}(i)} }^2
}
} \notag \\
+~&
\underbrace{\Expc \bX_{\pi^{\natural}(i)}^{\rmt}
{ 
\bX_{\pi^{\natural -1}(j)}\bX_{j}^{\rmt}
}
\bM
\bX_{\pi^{\natural}(i)} \bX_{\pi^{\natural}(i)}^{\rmt}
\bM
{
\bX_{j}\bX_{\pi^{\natural -1}(j)}^{\rmt}
}
\bX_{\pi^{\natural}(i)}}_{
\Ind_{j = \pi^{\natural 2}(i)}
\Expc \Fnorm{\bX_{\pi^{\natural}(i)}}^4
\bracket{\bX_{\pi^{\natural}(i)}^{\rmt} \bM \bM \bX_{\pi^{\natural}(i)}}
+ \Ind_{j \neq \pi^{\natural 2}(i)}
\Expc \Fnorm{\bX_{\pi^{\natural}(i)}}^2
\bracket{\bX_{\pi^{\natural}(i)}^{\rmt} \bM \bM \bX_{\pi^{\natural}(i)}}
} \notag \\
=~&
\bracket{1 + 2\Ind_{i = \pi^{\natural 2}(i)} + 3\Ind_{i = j} + 6\Ind_{i = j}\Ind_{j = \pi^{\natural 2}(i)}}
(p+4)\Bracket{\bracket{\trace(\bM)}^2 + 2\trace(\bM\bM)}
\notag \\
+~& \bracket{1 + 3\Ind_{j = \pi^{\natural 2}(i)}}
(p+2)(p+4)\trace(\bM\bM) \notag \\
+~&
2\Ind_{j = \pi^{\natural 2}(j)} \Ind_{i \neq j}
\Bracket{\bracket{\trace(\bM)}^2 + 2\trace(\bM\bM)} \notag \\
+~&
\Ind_{j \neq \pi^{\natural 2}(i)} (p+2)\trace(\bM\bM)
+ \Ind_{i\neq j}(p+2)\trace(\bM\bM).
\end{align}
Then we calculate $\Expc \Lambda_{3, 2}$ as
\begin{align}
\label{eq:xione_square_lambda32_dd_expc}
\Expc \Lambda_{3, 2} =~& 
\underbrace{\Expc
\bX_j^{\rmt} 
\bX_{i}\bX_{\pi^{\natural}(i)}^{\rmt}
\bM
\bX_{\pi^{\natural}(i)} \bX_{\pi^{\natural}(i)}^{\rmt}
\bM
\bX_{\pi^{\natural}(i)}\bX_{i}^{\rmt}
\bX_j}_{
\Ind_{i = j}\Expc\Fnorm{\bX_i}^4 \bracket{\bX_{\pi^{\natural}(i)}^{\rmt}\bM \bX_{\pi^{\natural}(i)}}^2
+ p\Ind_{i\neq j} \Expc \bracket{\bX_{\pi^{\natural}(i)}^{\rmt}\bM \bX_{\pi^{\natural}(i)}}^2
} \notag \\
+~& \
\underbrace{\Expc
\bX_j^{\rmt}
\bX_{i}\bX_{\pi^{\natural}(i)}^{\rmt}
\bM
\bX_{\pi^{\natural}(i)} \bX_{\pi^{\natural}(i)}^{\rmt}
\bM
\bX_{\pi^{\natural 2}(i)}\bX_{\pi^{\natural}(i)}^{\rmt}
\bX_j}_{
\Ind_{i = \pi^{\natural 2}(i)}
\Bracket{
\Ind_{i = j} \Expc \bX_i^{\rmt}\bX_i \bX_{\pi^{\natural}(i)}^{\rmt}\bM \bX_{\pi^{\natural}(i)} \bX_{\pi^{\natural}(i)}^{\rmt}\bM \bX_i \bX_i^{\rmt}\bX_{\pi^{\natural}(i)}
+ \Ind_{i\neq j}\Expc\bracket{\bX_{\pi^{\natural}(i)}^{\rmt}\bM \bX_{\pi^{\natural}(i)}}^2
}
} \notag \\
+~&
\underbrace{\Expc
\bX_j^{\rmt}
\bX_{i}\bX_{\pi^{\natural}(i)}^{\rmt}
\bM
\bX_{\pi^{\natural}(i)} \bX_{\pi^{\natural}(i)}^{\rmt}
\bM
\bX_{\pi^{\natural}(j)}\bX_{j}^{\rmt}
\bX_j}_{
\Ind_{i = j}
\Expc \Fnorm{\bX_i}^4
\bracket{\bX_{\pi^{\natural}(i)}^{\rmt}\bM \bX_{\pi^{\natural}(i)}}^2
} \notag \\
+~&
\underbrace{\Expc
\bX_j^{\rmt}
\bX_{i}\bX_{\pi^{\natural}(i)}^{\rmt}
\bM
\bX_{\pi^{\natural}(i)} \bX_{\pi^{\natural}(i)}^{\rmt}
\bM
\bX_{j}\bX_{\pi^{\natural -1}(j)}^{\rmt}
\bX_j}_{
\Ind_{i = j }\Ind_{i = \pi^{\natural 2}(i)}
\Expc \bX_i^{\rmt}\bX_i \bX_{\pi^{\natural}(i)}^{\rmt}\bM \bX_{\pi^{\natural}(i)} \bX_{\pi^{\natural}(i)}^{\rmt}\bM \bX_i \bX_i^{\rmt}\bX_{\pi^{\natural}(i)}
} \notag \\
+~&
\underbrace{\Expc
\bX_j^{\rmt} 
\bX_{\pi^{\natural}(i)}\bX_{\pi^{\natural 2}(i)}^{\rmt}
\bM
\bX_{\pi^{\natural}(i)} \bX_{\pi^{\natural}(i)}^{\rmt}
\bM
\bX_{\pi^{\natural}(i)}\bX_{i}^{\rmt}
\bX_j}_{
\Ind_{i = \pi^{\natural 2}(i)}
\Bracket{
\Ind_{i = j} \Expc \bX_i^{\rmt}\bX_i \bX_{\pi^{\natural}(i)}^{\rmt}\bM \bX_{\pi^{\natural}(i)} \bX_{\pi^{\natural}(i)}^{\rmt}\bM \bX_i \bX_i^{\rmt}\bX_{\pi^{\natural}(i)}
+ \Ind_{i\neq j}\Expc\bracket{\bX_{\pi^{\natural}(i)}^{\rmt}\bM \bX_{\pi^{\natural}(i)}}^2
}
} \notag \\
+~&
\underbrace{\Expc
\bX_j^{\rmt} 
\bX_{\pi^{\natural}(i)}\bX_{\pi^{\natural 2}(i)}^{\rmt}
\bM
\bX_{\pi^{\natural}(i)} \bX_{\pi^{\natural}(i)}^{\rmt}
\bM
\bX_{\pi^{\natural 2}(i)}\bX_{\pi^{\natural}(i)}^{\rmt}
\bX_j}_{
\Ind_{j \neq \pi^{\natural 2}(i)}
\Expc \Fnorm{\bX_{\pi^{\natural}(i)}}^2 \bX_{\pi^{\natural}(i)}^{\rmt}\bM\bM\bX_{\pi^{\natural}(i)} +
\Ind_{j= \pi^{\natural 2}(i)}
\Expc\bracket{\bX_j^{\rmt} \bX_{\pi^{\natural}(i)}}^2 \bracket{\bX_{\pi^{\natural}(i)}^{\rmt}\bM\bX_j }^2
} \notag \\
+~& \underbrace{\Expc
\bX_j^{\rmt}
\bX_{\pi^{\natural}(i)}\bX_{\pi^{\natural 2}(i)}^{\rmt}
\bM
\bX_{\pi^{\natural}(i)} \bX_{\pi^{\natural}(i)}^{\rmt}
\bM
\bX_{\pi^{\natural}(j)}\bX_{j}^{\rmt}
\bX_j}_{
\Ind_{i = j }\Ind_{i = \pi^{\natural 2}(i)}
\Expc \bX_i^{\rmt}\bX_i \bX_{\pi^{\natural}(i)}^{\rmt}\bM \bX_{\pi^{\natural}(i)} \bX_{\pi^{\natural}(i)}^{\rmt}\bM \bX_i \bX_i^{\rmt}\bX_{\pi^{\natural}(i)}
} \notag \\
+~& \underbrace{
\Expc
\bX_j^{\rmt} 
\bX_{\pi^{\natural}(i)}\bX_{\pi^{\natural 2}(i)}^{\rmt}
\bM
\bX_{\pi^{\natural}(i)} \bX_{\pi^{\natural}(i)}^{\rmt}
\bM
\bX_{j}\bX_{\pi^{\natural -1}(j)}^{\rmt}
\bX_j}_{
\Ind_{j= \pi^{\natural 2}(i)}
\Expc\bracket{\bX_j^{\rmt} \bX_{\pi^{\natural}(i)}}^2 \bracket{\bX_{\pi^{\natural}(i)}^{\rmt}\bM\bX_j }^2
} \notag \\
+~&
\underbrace{\Expc
\bX_j^{\rmt}
\bX_{j}\bX_{\pi^{\natural}(j)}^{\rmt}
\bM
\bX_{\pi^{\natural}(i)} \bX_{\pi^{\natural}(i)}^{\rmt}
\bM
\bX_{\pi^{\natural}(i)}\bX_{i}^{\rmt}
\bX_j}_{
\Ind_{i = j}
\Expc \Fnorm{\bX_i}^4
\bracket{\bX_{\pi^{\natural}(i)}^{\rmt}\bM \bX_{\pi^{\natural}(i)}}^2
} \notag \\
+~&
\underbrace{\Expc
\bX_j^{\rmt}
\bX_{j}\bX_{\pi^{\natural}(j)}^{\rmt}
\bM
\bX_{\pi^{\natural}(i)} \bX_{\pi^{\natural}(i)}^{\rmt}
\bM
\bX_{\pi^{\natural 2}(i)}\bX_{\pi^{\natural}(i)}^{\rmt}
\bX_j}_{
\Ind_{i = j }\Ind_{i = \pi^{\natural 2}(i)}
\Expc \bX_i^{\rmt}\bX_i \bX_{\pi^{\natural}(i)}^{\rmt}\bM \bX_{\pi^{\natural}(i)} \bX_{\pi^{\natural}(i)}^{\rmt}\bM \bX_i \bX_i^{\rmt}\bX_{\pi^{\natural}(i)}
} \notag \\
+~&
\underbrace{\Expc
\bX_j^{\rmt}
\bX_{j}\bX_{\pi^{\natural}(j)}^{\rmt}
\bM
\bX_{\pi^{\natural}(i)} \bX_{\pi^{\natural}(i)}^{\rmt}
\bM
\bX_{\pi^{\natural}(j)}\bX_{j}^{\rmt}
\bX_j}_{
\Ind_{i = j}
\Expc \Fnorm{\bX_i}^4
\bracket{\bX_{\pi^{\natural}(i)}^{\rmt}\bM \bX_{\pi^{\natural}(i)}}^2
+ \Ind_{i\neq j}
\Expc \Fnorm{\bX_j}^4 \bX_{\pi^{\natural}(i)}^{\rmt} \bM \bM \bX_{\pi^{\natural}(i)}
} \notag \\
+~& \
\underbrace{\Expc
\bX_j^{\rmt}
\bX_{j}\bX_{\pi^{\natural}(j)}^{\rmt}
\bM
\bX_{\pi^{\natural}(i)} \bX_{\pi^{\natural}(i)}^{\rmt}
\bM
\bX_{j}\bX_{\pi^{\natural -1}(j)}^{\rmt}
\bX_j}_{
\Ind_{j = \pi^{\natural 2}(j)}
\Bracket{
\Ind_{i = j}
 \Expc \bX_i^{\rmt}\bX_i \bX_{\pi^{\natural}(i)}^{\rmt}\bM \bX_{\pi^{\natural}(i)} \bX_{\pi^{\natural}(i)}^{\rmt}\bM \bX_i \bX_i^{\rmt}\bX_{\pi^{\natural}(i)}
+ \Ind_{i\neq j}
\Expc \Fnorm{\bX_j}^2 \bX_j^{\rmt}\bM \bM \bX_j
}
} \notag \\
+~& \underbrace{
\Expc
\bX_j^{\rmt}
\bX_{\pi^{\natural -1}(j)}\bX_{j}^{\rmt}
\bM
\bX_{\pi^{\natural}(i)} \bX_{\pi^{\natural}(i)}^{\rmt}
\bM
\bX_{\pi^{\natural}(i)}\bX_{i}^{\rmt}
\bX_j}_{
\Ind_{i = j }\Ind_{i = \pi^{\natural 2}(i)}
\Expc \bX_i^{\rmt}\bX_i \bX_{\pi^{\natural}(i)}^{\rmt}\bM \bX_{\pi^{\natural}(i)} \bX_{\pi^{\natural}(i)}^{\rmt}\bM \bX_i \bX_i^{\rmt}\bX_{\pi^{\natural}(i)}
} \notag \\
+~&
\underbrace{\Expc
\bX_j^{\rmt} 
\bX_{\pi^{\natural -1}(j)}\bX_{j}^{\rmt}
\bM
\bX_{\pi^{\natural}(i)} \bX_{\pi^{\natural}(i)}^{\rmt}
\bM
\bX_{\pi^{\natural 2}(i)}\bX_{\pi^{\natural}(i)}^{\rmt}
\bX_j}_{
\Ind_{j= \pi^{\natural 2}(i)}
\Expc\bracket{\bX_j^{\rmt} \bX_{\pi^{\natural}(i)}}^2 \bracket{\bX_{\pi^{\natural}(i)}^{\rmt}\bM\bX_j }^2
} \notag \\
+~&
\underbrace{\Expc
\bX_j^{\rmt}
\bX_{\pi^{\natural -1}(j)}\bX_{j}^{\rmt}
\bM
\bX_{\pi^{\natural}(i)} \bX_{\pi^{\natural}(i)}^{\rmt}
\bM
\bX_{\pi^{\natural}(j)}\bX_{j}^{\rmt}
\bX_j}_{
\Ind_{j = \pi^{\natural 2}(j)}
\Bracket{
\Ind_{i = j}
 \Expc \bX_i^{\rmt}\bX_i \bX_{\pi^{\natural}(i)}^{\rmt}\bM \bX_{\pi^{\natural}(i)} \bX_{\pi^{\natural}(i)}^{\rmt}\bM \bX_i \bX_i^{\rmt}\bX_{\pi^{\natural}(i)}
+ \Ind_{i\neq j}
\Expc \Fnorm{\bX_j}^2 \bX_j^{\rmt}\bM \bM \bX_j
}
} \notag \\
+~&
\underbrace{\Expc
\bX_j^{\rmt}
\bracket{ 
\bX_{\pi^{\natural -1}(j)}\bX_{j}^{\rmt}
}
\bM
\bX_{\pi^{\natural}(i)} \bX_{\pi^{\natural}(i)}^{\rmt}
\bM
\bracket{
\bX_{j}\bX_{\pi^{\natural -1}(j)}^{\rmt}
}
\bX_j}_{
\Ind_{j = \pi^{\natural 2}(i)}
\Expc\bracket{\bX_j^{\rmt} \bX_{\pi^{\natural}(i)}}^2 \bracket{\bX_{\pi^{\natural}(i)}^{\rmt}\bM\bX_j }^2
+ \Ind_{j \neq \pi^{\natural 2}(i)}
\Expc \Fnorm{\bX_j }^2 \bX_j^{\rmt}\bM \bM \bX_j
} \notag \\
=~&
4\Ind_{i = j}p(p+2)\Bracket{[\trace(\bM)]^2 + 2\trace(\bM\bM)
} + \Ind_{i\neq j} p(p+2)\trace(\bM\bM) \notag \\
+~& 8\Ind_{i = j}\Ind_{i = \pi^{\natural 2}(i)}
(p+2)\Bracket{\bracket{\trace(\bM)}^2 + 2\trace(\bM\bM)}
\notag \\
+~& 4\Ind_{j = \pi^{\natural 2}(i)}
\Bracket{2\Bracket{\trace(\bM)}^2 + (p+6)\trace(\bM\bM)}
\notag \\
+~& 2\Ind_{j\neq \pi^{\natural 2}(i)}(p+2)\trace(\bM\bM)
+ 2\Ind_{i \neq j}\Ind_{j = \pi^{\natural 2}(j)}
(p+2)\trace(\bM\bM)\notag \\
+~&
\Ind_{i\neq j}\bracket{p + 2\Ind_{i = \pi^{\natural 2}(i)}}
\Bracket{\bracket{\trace(\bM)}^2 + 2\trace(\bM\bM)}.
\end{align}
The term $\Lambda_{3, 3}$ is computed as
\begin{align}
\label{eq:xione_square_lambda33_dd_expc}	
\Expc \Lambda_{3,3} =~&
\underbrace{\Expc \bX_{\pi^{\natural}(i)}^{\rmt}
\bX_{i}\bX_{\pi^{\natural}(i)}^{\rmt}
\bM \bX_{\pi^{\natural}(i)} \bX_{\pi^{\natural}(i)}^{\rmt} \bM
\bX_{\pi^{\natural}(i)}\bX_{i}^{\rmt}
\bX_j}_{0} \notag \\
+~&
\underbrace{\Expc \bX_{\pi^{\natural}(i)}^{\rmt}
\bX_{i}\bX_{\pi^{\natural}(i)}^{\rmt}
\bM \bX_{\pi^{\natural}(i)} \bX_{\pi^{\natural}(i)}^{\rmt} \bM
\bX_{\pi^{\natural 2}(i)}\bX_{\pi^{\natural}(i)}^{\rmt}
\bX_j}_{0} \notag \\
+~&
\underbrace{\Expc \bX_{\pi^{\natural}(i)}^{\rmt}
\bX_{i}\bX_{\pi^{\natural}(i)}^{\rmt}
\bM \bX_{\pi^{\natural}(i)} \bX_{\pi^{\natural}(i)}^{\rmt} \bM
\bX_{\pi^{\natural}(j)}\bX_{j}^{\rmt}
\bX_j}_{
\Ind_{i = \pi^{\natural}(j)}p \Expc \Bracket{\bX_{\pi^{\natural}(i)}^{\rmt}\bM \bX_{\pi^{\natural}(i)} }^2
} \notag \\
+~&
\underbrace{\Expc \bX_{\pi^{\natural}(i)}^{\rmt} 
\bX_{i}\bX_{\pi^{\natural}(i)}^{\rmt}
\bM \bX_{\pi^{\natural}(i)} \bX_{\pi^{\natural}(i)}^{\rmt} \bM
\bX_{j}\bX_{\pi^{\natural -1}(j)}^{\rmt}
\bX_j}_{0} \notag \\
+~&
\underbrace{\Expc \bX_{\pi^{\natural}(i)}^{\rmt}
\bX_{\pi^{\natural}(i)}\bX_{\pi^{\natural 2}(i)}^{\rmt}
\bM \bX_{\pi^{\natural}(i)} \bX_{\pi^{\natural}(i)}^{\rmt} \bM
\bX_{\pi^{\natural}(i)}\bX_{i}^{\rmt}
\bX_j}_{0} \notag \\
+~&
\underbrace{\Expc \bX_{\pi^{\natural}(i)}^{\rmt}
\bX_{\pi^{\natural}(i)}\bX_{\pi^{\natural 2}(i)}^{\rmt}
\bM \bX_{\pi^{\natural}(i)} \bX_{\pi^{\natural}(i)}^{\rmt} \bM
\bX_{\pi^{\natural 2}(i)}\bX_{\pi^{\natural}(i)}^{\rmt}
\bX_j}_{0} \notag \\
+~&
\underbrace{\Expc \bX_{\pi^{\natural}(i)}^{\rmt}
\bX_{\pi^{\natural}(i)}\bX_{\pi^{\natural 2}(i)}^{\rmt}
\bM \bX_{\pi^{\natural}(i)} \bX_{\pi^{\natural}(i)}^{\rmt} \bM
\bX_{\pi^{\natural}(j)}\bX_{j}^{\rmt}
\bX_j}_{0} \notag \\ +
~&
\underbrace{\Expc \bX_{\pi^{\natural}(i)}^{\rmt}
\bX_{\pi^{\natural}(i)}\bX_{\pi^{\natural 2}(i)}^{\rmt}
\bM \bX_{\pi^{\natural}(i)} \bX_{\pi^{\natural}(i)}^{\rmt} \bM
\bX_{j}\bX_{\pi^{\natural -1}(j)}^{\rmt}
\bX_j}_{
\Ind_{j = \pi^{\natural 3}(i)} \Expc \Fnorm{\bX_{\pi^{\natural}(i)}}^2 \bX_{\pi^{\natural}(i)}^{\rmt} \bM \bM \bX_{\pi^{\natural}(i)}
} \notag \\
+~&
\underbrace{
\Expc \bX_{\pi^{\natural}(i)}^{\rmt}
\bX_{j}\bX_{\pi^{\natural}(j)}^{\rmt}
\bM \bX_{\pi^{\natural}(i)} \bX_{\pi^{\natural}(i)}^{\rmt} \bM
\bX_{\pi^{\natural}(i)}\bX_{i}^{\rmt}
\bX_j}_{
\Ind_{i = \pi^{\natural}(j)} \Expc \Bracket{\bX_{\pi^{\natural}(i)}^{\rmt}\bM \bX_{\pi^{\natural}(i)} }^2
} \notag \\
+~&
\underbrace{\Expc \bX_{\pi^{\natural}(i)}^{\rmt}
\bX_{j}\bX_{\pi^{\natural}(j)}^{\rmt}
\bM \bX_{\pi^{\natural}(i)} \bX_{\pi^{\natural}(i)}^{\rmt} \bM
\bX_{\pi^{\natural 2}(i)}\bX_{\pi^{\natural}(i)}^{\rmt}
\bX_j}_{0} \notag \\
+~&
\underbrace{\Expc \bX_{\pi^{\natural}(i)}^{\rmt}
\bX_{j}\bX_{\pi^{\natural}(j)}^{\rmt}
\bM \bX_{\pi^{\natural}(i)} \bX_{\pi^{\natural}(i)}^{\rmt} \bM
\bX_{\pi^{\natural}(j)}\bX_{j}^{\rmt}
\bX_j}_{0} \notag \\
+ ~&
\underbrace{\Expc \bX_{\pi^{\natural}(i)}^{\rmt}
\bX_{j}\bX_{\pi^{\natural}(j)}^{\rmt}
\bM \bX_{\pi^{\natural}(i)} \bX_{\pi^{\natural}(i)}^{\rmt} \bM
\bX_{j}\bX_{\pi^{\natural -1}(j)}^{\rmt}
\bX_j}_{0} \notag \\
+~&
\underbrace{\Expc \bX_{\pi^{\natural}(i)}^{\rmt}
\bX_{\pi^{\natural -1}(j)}\bX_{j}^{\rmt}
\bM \bX_{\pi^{\natural}(i)} \bX_{\pi^{\natural}(i)}^{\rmt} \bM
\bX_{\pi^{\natural}(i)}\bX_{i}^{\rmt}
\bX_j}_{0} \notag \\
+~&
\underbrace{
\Expc \bX_{\pi^{\natural}(i)}^{\rmt} 
\bX_{\pi^{\natural -1}(j)}\bX_{j}^{\rmt}
\bM \bX_{\pi^{\natural}(i)} \bX_{\pi^{\natural}(i)}^{\rmt} \bM
\bX_{\pi^{\natural 2}(i)}\bX_{\pi^{\natural}(i)}^{\rmt}
\bX_j}_{
\Ind_{j = \pi^{\natural 3}(i)} \Expc \Bracket{\bX_{\pi^{\natural}(i)}^{\rmt}\bM \bX_{\pi^{\natural}(i)} }^2
} \notag \\
+~&
\underbrace{\Expc \bX_{\pi^{\natural}(i)}^{\rmt}
\bX_{\pi^{\natural -1}(j)}\bX_{j}^{\rmt}
\bM \bX_{\pi^{\natural}(i)} \bX_{\pi^{\natural}(i)}^{\rmt} \bM
\bX_{\pi^{\natural}(j)}\bX_{j}^{\rmt}
\bX_j}_{0} \notag \\
+~&
\underbrace{\Expc \bX_{\pi^{\natural}(i)}^{\rmt}
\bX_{\pi^{\natural -1}(j)}\bX_{j}^{\rmt}
\bM \bX_{\pi^{\natural}(i)} \bX_{\pi^{\natural}(i)}^{\rmt} \bM
\bX_{j}\bX_{\pi^{\natural -1}(j)}^{\rmt}
\bX_j}_{0} \notag \\
=~&
\Bracket{ (p+1)\Ind_{i = \pi^{\natural}(j)} + \Ind_{j = \pi^{\natural 3}(i)}} [\trace(\bM)]^2 \notag \\
+~&
\Bracket{2(p+1)\Ind_{i = \pi^{\natural}(j)} + (p+4)\Ind_{j = \pi^{\natural 3}(i)} } \trace(\bM\bM).
\end{align}
Then, we consider the term $\Expc \Lambda_{3, 4}$, which can be written as
\begin{align}
\label{eq:xione_square_lambda34_dd_expc}	
\Expc \Lambda_{3, 4} =~&
\underbrace{\Expc \bX_j^{\rmt}
\bX_{i}\bX_{\pi^{\natural}(i)}^{\rmt}
\bM
\bX_{\pi^{\natural}(i)} \bX_{\pi^{\natural}(i)}^{\rmt}
\bM
\bX_{\pi^{\natural}(i)}\bX_{i}^{\rmt}
\bX_{\pi^{\natural}(i)}}_{0}  \notag \\
+~& \underbrace{\Expc \bX_j^{\rmt}
\bX_{i}\bX_{\pi^{\natural}(i)}^{\rmt}
\bM
\bX_{\pi^{\natural}(i)} \bX_{\pi^{\natural}(i)}^{\rmt}
\bM
\bX_{\pi^{\natural 2}(i)}\bX_{\pi^{\natural}(i)}^{\rmt}
\bX_{\pi^{\natural}(i)}}_{0} \notag \\
+~& \underbrace{\Expc \bX_j^{\rmt}
\bX_{i}\bX_{\pi^{\natural}(i)}^{\rmt}
\bM
\bX_{\pi^{\natural}(i)} \bX_{\pi^{\natural}(i)}^{\rmt}
\bM
\bX_{\pi^{\natural}(j)}\bX_{j}^{\rmt}
\bX_{\pi^{\natural}(i)}}_{
\Ind_{i = \pi^{\natural}(j)}
\Expc \Bracket{\bX_{\pi^{\natural}(i)}^{\rmt}\bM \bX_{\pi^{\natural}(i)} }^2
} \notag \\
+~& \underbrace{\Expc \bX_j^{\rmt}
\bX_{i}\bX_{\pi^{\natural}(i)}^{\rmt}
\bM
\bX_{\pi^{\natural}(i)} \bX_{\pi^{\natural}(i)}^{\rmt}
\bM
\bX_{j}\bX_{\pi^{\natural -1}(j)}^{\rmt}
\bX_{\pi^{\natural}(i)}}_{0} \notag \\
+~& \underbrace{\Expc \bX_j^{\rmt} 
\bX_{\pi^{\natural}(i)}\bX_{\pi^{\natural 2}(i)}^{\rmt}
\bM
\bX_{\pi^{\natural}(i)} \bX_{\pi^{\natural}(i)}^{\rmt}
\bM
\bX_{\pi^{\natural}(i)}\bX_{i}^{\rmt}
\bX_{\pi^{\natural}(i)}}_{0} \notag \\
+~& \underbrace{\Expc \bX_j^{\rmt}
\bX_{\pi^{\natural}(i)}\bX_{\pi^{\natural 2}(i)}^{\rmt}
\bM
\bX_{\pi^{\natural}(i)} \bX_{\pi^{\natural}(i)}^{\rmt}
\bM
\bX_{\pi^{\natural 2}(i)}\bX_{\pi^{\natural}(i)}^{\rmt}
\bX_{\pi^{\natural}(i)}}_{0} \notag \\
+~& \underbrace{\Expc \bX_j^{\rmt}
\bX_{\pi^{\natural}(i)}\bX_{\pi^{\natural 2}(i)}^{\rmt}
\bM
\bX_{\pi^{\natural}(i)} \bX_{\pi^{\natural}(i)}^{\rmt}
\bM
\bX_{\pi^{\natural}(j)}\bX_{j}^{\rmt}
\bX_{\pi^{\natural}(i)}}_{0}\notag \\
+~&
\underbrace{\Expc \bX_j^{\rmt}
\bX_{\pi^{\natural}(i)}\bX_{\pi^{\natural 2}(i)}^{\rmt}
\bM
\bX_{\pi^{\natural}(i)} \bX_{\pi^{\natural}(i)}^{\rmt}
\bM
\bX_{j}\bX_{\pi^{\natural -1}(j)}^{\rmt}
\bX_{\pi^{\natural}(i)}}_{
\Ind_{j = \pi^{\natural 3}(i)}
\Expc \Bracket{\bX_{\pi^{\natural}(i)}^{\rmt}\bM \bX_{\pi^{\natural}(i)} }^2
}\notag \\
+~&
\underbrace{\Expc \bX_j^{\rmt}
\bX_{j}\bX_{\pi^{\natural}(j)}^{\rmt}
\bM
\bX_{\pi^{\natural}(i)} \bX_{\pi^{\natural}(i)}^{\rmt}
\bM
\bX_{\pi^{\natural}(i)}\bX_{i}^{\rmt}
\bX_{\pi^{\natural}(i)}}_{
\Ind_{i = \pi^{\natural}(j)}p
\Expc \Bracket{\bX_{\pi^{\natural}(i)}^{\rmt}\bM \bX_{\pi^{\natural}(i)} }^2
} \notag \\
+~& \underbrace{\Expc \bX_j^{\rmt}
\bX_{j}\bX_{\pi^{\natural}(j)}^{\rmt}
\bM
\bX_{\pi^{\natural}(i)} \bX_{\pi^{\natural}(i)}^{\rmt}
\bM
\bX_{\pi^{\natural 2}(i)}\bX_{\pi^{\natural}(i)}^{\rmt}
\bX_{\pi^{\natural}(i)}}_{0} \notag \\
+~& \underbrace{\Expc \bX_j^{\rmt}
\bX_{j}\bX_{\pi^{\natural}(j)}^{\rmt}
\bM
\bX_{\pi^{\natural}(i)} \bX_{\pi^{\natural}(i)}^{\rmt}
\bM
\bX_{\pi^{\natural}(j)}\bX_{j}^{\rmt}
\bX_{\pi^{\natural}(i)}}_{0} \notag \\
+~& \underbrace{\Expc \bX_j^{\rmt}
\bX_{j}\bX_{\pi^{\natural}(j)}^{\rmt}
\bM
\bX_{\pi^{\natural}(i)} \bX_{\pi^{\natural}(i)}^{\rmt}
\bM
\bX_{j}\bX_{\pi^{\natural -1}(j)}^{\rmt}
\bX_{\pi^{\natural}(i)}}_{0} \notag \\
+~& \underbrace{\Expc \bX_j^{\rmt}
\bX_{\pi^{\natural -1}(j)}\bX_{j}^{\rmt}
\bM
\bX_{\pi^{\natural}(i)} \bX_{\pi^{\natural}(i)}^{\rmt}
\bM
\bX_{\pi^{\natural}(i)}\bX_{i}^{\rmt}
\bX_{\pi^{\natural}(i)}}_{0} \notag \\
+~& \underbrace{\Expc \bX_j^{\rmt}
\bX_{\pi^{\natural -1}(j)}\bX_{j}^{\rmt}
\bM
\bX_{\pi^{\natural}(i)} \bX_{\pi^{\natural}(i)}^{\rmt}
\bM
\bX_{\pi^{\natural 2}(i)}\bX_{\pi^{\natural}(i)}^{\rmt}
\bX_{\pi^{\natural}(i)}}_{
\Ind_{j = \pi^{\natural 3}(i)}
\Expc \Fnorm{\bX_{\pi^{\natural}(i)}}^2
\bX_{\pi^{\natural}(i)}^{\rmt } \bM \bM \bX_{\pi^{\natural}(i)}
} \notag \\
+~& \underbrace{\Expc \bX_j^{\rmt}
\bX_{\pi^{\natural -1}(j)}\bX_{j}^{\rmt}
\bM
\bX_{\pi^{\natural}(i)} \bX_{\pi^{\natural}(i)}^{\rmt}
\bM
\bX_{\pi^{\natural}(j)}\bX_{j}^{\rmt}
\bX_{\pi^{\natural}(i)}}_{0} \notag \\
+~& \underbrace{\Expc \bX_j^{\rmt}
\bX_{\pi^{\natural -1}(j)}\bX_{j}^{\rmt}
\bM
\bX_{\pi^{\natural}(i)} \bX_{\pi^{\natural}(i)}^{\rmt}
\bM
\bX_{j}\bX_{\pi^{\natural -1}(j)}^{\rmt}
\bX_{\pi^{\natural}(i)}}_{0} \notag \\
=~&
\Bracket{ (p+1)\Ind_{i = \pi^{\natural}(j)} + \Ind_{j = \pi^{\natural 3}(i)}} [\trace(\bM)]^2 \notag \\
+~&
\Bracket{2(p+1)\Ind_{i = \pi^{\natural}(j)} + (p+4)\Ind_{j = \pi^{\natural 3}(i)} } \trace(\bM\bM).
\end{align}
Combing ~\eqref{eq:xione_square_lambda31_dd_expc},
~\eqref{eq:xione_square_lambda32_dd_expc}, ~\eqref{eq:xione_square_lambda33_dd_expc},
and ~\eqref{eq:xione_square_lambda34_dd_expc} together then yields
\begin{align}
\label{eq:xione_square_lambda3_dd_expc}
\Expc \Lambda_3 =
\frac{4hp(p+2)}{n^2}\Bracket{1 + o(1)}\Bracket{\trace(\bM)}^2	
+ 2p^2\Bracket{1 + o(1)}\trace(\bM\bM).
\end{align}
The proof is thus completed by summarizing the computations thereof.
\end{proof}.

\begin{lemma}
\label{lemma:xitwo_square_expc}
We have
\[
\Expc \Xi_2^2 =
2\Bracket{(p+2)(p+3) + (n-2)(p+1)}
\Fnorm{\bB^{\natural}}^2 =
2np\bracket{1 +p/n + o(1)}\Fnorm{\bB^{\natural}}^2,
\]	
where $\Xi_2$ is defined in ~\eqref{eq:xi_decomposition}.
\end{lemma}

\begin{proof}
We have
\[
\Expc \Xi_2^2 =~&
\Expc\Bracket{
\bX_{\pi^{\natural}(i)}^{\rmt}\bB^{\natural} \bW^{\rmt}\bX\bracket{\bX_{\pi^{\natural}(i)} - \bX_j}
\bracket{\bX_{\pi^{\natural}(i)} - \bX_j}^{\rmt}\bX^{\rmt}\bW
\bB^{\natural \rmt}
\bX_{\pi^{\natural}(i)}} \\
=~&
\Expc\Bracket{\bX_{\pi^{\natural}(i)}^{\rmt}\bB^{\natural}
\trace\Bracket{\bX\bracket{\bX_{\pi^{\natural}(i)} - \bX_j}
\bracket{\bX_{\pi^{\natural}(i)} - \bX_j}^{\rmt}\bX^{\rmt}}
\bB^{\natural \rmt}
\bX_{\pi^{\natural}(i)}} \\
=~&
\Expc\Bracket{\norm{\bX\bracket{\bX_{\pi^{\natural}(i)} - \bX_j}}{2}^2
\bX_{\pi^{\natural}(i)}^{\rmt}\bB^{\natural}
\bB^{\natural \rmt}
\bX_{\pi^{\natural}(i)}} \\
=~& \Expc \Bracket{\norm{\bX\bracket{\bX_{\pi^{\natural}(i)} - \bX_j}}{2}^2 \times
\norm{\bB^{\natural \rmt} \bX_{\pi^{\natural}(i)}}{2}^2}.
\]
For the conciseness of notation, we assume $\pi^{\natural}(i) = 1$ and $j = 2$ w.l.o.g. Decomposing the term $\norm{\bX\bracket{\bX_1 - \bX_2}}{2}^2$ as
\[
\Fnorm{\bX\bracket{\bX_{1} - \bX_2 }}^2
= \underbrace{\Bracket{\bX_1^{\rmt}\bracket{\bX_1 - \bX_2}}^2}_{\calT_1}
+ \underbrace{\Bracket{\bX_2^{\rmt}\bracket{\bX_1 - \bX_2}}^2}_{\calT_2}
+ \underbrace{\sum_{i=3}^n \Bracket{\bX_i^{\rmt}\bracket{\bX_1 - \bX_2}}^2}_{\calT_3},
\]
we then separately bound the above three terms.
For the first term $\Expc {\calT_1\Fnorm{\bB^{\natural \rmt}\bX_1}^2}$, we have
\begin{align}
\Expc{\calT_1 \Fnorm{\bB^{\natural \rmt}\bX_1}^2} =~&
\Expc\Bracket{\bracket{\norm{\bX_1}{2}^4 + \bracket{\bX_1^{\rmt}\bX_2}^2}
\Fnorm{\bB^{\natural \rmt}\bX_1}^2} \notag \\
=~&
\underbrace{\Expc \norm{\bX_1}{2}^4 \Fnorm{\bB^{\natural \rmt}\bX_1}^2}_{
(p+2)(p+4)\Fnorm{\bB^{\natural}}^2
} +
\underbrace{\Expc \bracket{\bX_1^{\rmt}\bX_2}^2\Fnorm{\bB^{\natural \rmt}\bX_1}^2}_{
(p+2)\Fnorm{\bB^{\natural}}^2
}
\stackrel{\cirone}{=} (p+2)(p+5)\Fnorm{\bB^{\natural}}^2,
\label{eq:xitwosquare_t1_expc}
\end{align}
where $\cirone$ is due to ~\eqref{eq:four_order_expc} and ~\eqref{eq:six_order_expc}.

Similarly, for term $\Expc{\calT_2\Fnorm{\bB^{\natural \rmt}\bX_1}^2}$, we
invoke ~\eqref{eq:four_order_expc} and ~\eqref{eq:six_order_expc},
which gives
\begin{align}
\Expc{\calT_2\Fnorm{\bB^{\natural \rmt}\bX_1}^2}
=~&
\underbrace{\Expc\Bracket{\bracket{\bX_1^{\rmt}\bX_2}^2\Fnorm{\bB^{\natural \rmt}\bX_1}^2 }}_{(p+2)\Fnorm{\bB^{\natural}}^2}
+ \underbrace{\Expc \norm{\bX_2}{2}^4 \times \Expc\Fnorm{\bB^{\natural \rmt}\bX_1}^2}_{p(p+2)\Fnorm{\bB^{\natural}}^2} \\
\stackrel{\cirtwo}{=}~& (p+1)(p+2)\Fnorm{\bB^{\natural}}^2,
\label{eq:xitwosquare_t2_expc}
\end{align}
where $\cirtwo$ is due to ~\eqref{eq:four_order_expc}.

For the last term $\Expc{\calT_3\Fnorm{\bB^{\natural \rmt}\bX_1}^2}$,
we exploit the independence among the rows of matrix $\bX$ and
have
\begin{align}
\Expc{ T_3 \Fnorm{\bB^{\natural \rmt}\bX_1}^2}
= ~& \sum_{i\geq 3} \Expc\Bracket{\bracket{\bX_i^{\rmt}\bracket{\bX_1 - \bX_2}}^2 \Fnorm{\bB^{\natural \rmt}\bX_1}^2 } \notag \\
=~& \sum_{i\geq 3} \Expc \Bracket{\norm{\bX_1 - \bX_2}{2}^2
\cdot \Fnorm{\bB^{\natural \rmt}\bX_1}^2 } \notag \\
=~& \sum_{i\geq 3}
\Expc\Bracket{\bracket{\norm{\bX_1}{2}^2 + \norm{\bX_2}{2}^2} \cdot \Fnorm{\bB^{\natural \rmt}\bX_1}^2  } \notag \\
=~& 2\sum_{i\geq 3}(p+1)\Fnorm{\bB^{\natural}}^2 =
2(n-2)(p+1)\Fnorm{\bB^{\natural}}^2.
\label{eq:xitwosquare_t3_expc}
\end{align}
The proof is then completed by combining
~\eqref{eq:xitwosquare_t1_expc}, ~\eqref{eq:xitwosquare_t2_expc}, and ~\eqref{eq:xitwosquare_t3_expc}.

\end{proof}

\begin{lemma}
\label{lemma:xithree_square_expc}
We have
\[
\Expc \Xi_3^2 = ~&
2 n^2\Bracket{\frac{p}{n}+ \bracket{1 - \frac{h}{n}}^2
+ \frac{p^2}{n^2} + \frac{4p(n-h)^2}{n^3} + o(1)} \trace(\bM),
\]
where $\Xi_3$ is defined in ~\eqref{eq:xi_decomposition}.
\end{lemma}

\begin{proof}
To begin with, we decompose the term $\Expc \Xi_3^2$ as
\begin{align}
\label{eq:xithree_square_decompose}	
\Expc \Xi_3^2
=~&
\Expc\underbrace{ \Bracket{
\bracket{\bX_{\pi^{\natural}(i)} - \bX_j}^{\rmt} \
\bSigma \bM \bSigma^{\rmt}\bracket{\bX_{\pi^{\natural}(i)} - \bX_j}}}_{\defequal \Lambda_1}
+
2 \Expc\underbrace{ \Bracket{
\bracket{\bX_{\pi^{\natural}(i)} - \bX_j}^{\rmt} \
\bSigma \bM \bDelta^{\rmt}\bracket{\bX_{\pi^{\natural}(i)} - \bX_j} }}_{\defequal \Lambda_2} \notag \\
+~& \Expc\underbrace{\Bracket{
\bracket{\bX_{\pi^{\natural}(i)} - \bX_j}^{\rmt} \
\bDelta \bM \bDelta^{\rmt}\bracket{\bX_{\pi^{\natural}(i)} - \bX_j}}}_{\defequal \Lambda_3}.
\end{align}

\noindent\textbf{Step I.}
First we consider $\Expc \Lambda_1$, which can be written as
\begin{align}
\label{eq:xithree_square_lambda1_sum_expc}	
\Expc \Lambda_1 =~&
2 \Expc\trace\bracket{\bSigma \bM \bSigma^{\rmt}}
\stackrel{\cirone}{=}
2 n^2\Bracket{\frac{p}{n}+ \bracket{1 - \frac{h}{n}}^2
+ o(1)} \trace(\bM),
\end{align}
where $\cirone$ is due to Lemma~\ref{lemma:sigmaMsigma_trace_expc}.

\vspace{0.1in}
\noindent\textbf{Step II.}
Then we turn to $\Expc \Lambda_2$, which can be written as
\[
\Expc \Lambda_2 =~& \
(n-h)\Expc \underbrace{\Bracket{
\bX_{\pi^{\natural}(i)}^{\rmt} \
\bM \bDelta^{\rmt}\bX_{\pi^{\natural}(i)} }}_{\Lambda_{2,1}} + \
(n-h)\Expc\underbrace{\Bracket{
\bX_j^{\rmt} \
 \bM \bDelta^{\rmt}\bX_j}}_{\Lambda_{2,2}} \\
-~&
(n-h)\Expc \underbrace{\Bracket{
{\bX_{\pi^{\natural}(i)}}^{\rmt} \
\bM \bDelta^{\rmt}{\bX_j}}}_{\Lambda_{2, 3}}-
(n-h)\Expc \underbrace{\Bracket{
{ \bX_j}^{\rmt} \
\bM \bDelta^{\rmt}{\bX_{\pi^{\natural}(i)}}}}_{\Lambda_{2, 4}}.
\]

\noindent\textbf{Case $(s, s)$: $i = \pi^{\natural}(i)$ and $j = \pi^{\natural}(j)$.}
We have
\[
\Expc \Lambda_{2, 1}
=~&  \Expc \bX_{i}^{\rmt} \
\bM \bracket{\bX_i \bX_i^{\rmt} + \bX_j \bX_j^{\rmt}}\bX_{i} = \
(p+3)\trace(\bM), \\
\Expc \Lambda_{2, 2}
=~&  \Expc \bX_{j}^{\rmt} \
\bM \bracket{\bX_i \bX_i^{\rmt} + \bX_j \bX_j^{\rmt}}\bX_{j} = \
(p+3)\trace(\bM).
\]
In addition, we can verify that $\Expc \Lambda_{2, 2}$ and
$\Lambda_{2, 3}$ are both zero,
which suggests that
\begin{align}
\label{eq:xithree_square_lambda2_ss_expc}
\Expc \Lambda_2 = 2(n-h)(p+3)\trace(\bM).
\end{align}

\noindent\textbf{Case $(s, d)$: $i = \pi^{\natural}(i)$ and $j \neq \pi^{\natural}(j)$.}
We have
\[
\Expc \Lambda_{2, 1}
=~& \Expc \bX_{i}^{\rmt} \
\bM \bracket{\bX_i \bX_i^{\rmt} +\bX_{j}\bX_{\pi^{\natural}(j)}^{\rmt} +
\bX_{\pi^{\natural -1}(j)}\bX_{j}^{\rmt}}\bX_{i} = \
(p+2)\trace(\bM); \\
\Expc \Lambda_{2, 2} =~& \
\Expc \bX_j^{\rmt} \bM \bracket{
\bX_i \bX_i^{\rmt} +\bX_{j}\bX_{\pi^{\natural}(j)}^{\rmt} +
\bX_{\pi^{\natural -1}(j)}\bX_{j}^{\rmt}}
\bX_j = \trace(\bM) .
\]
Moreover, we have both $\Expc \Lambda_{2, 3}$ and $\Expc \Lambda_{2, 4}$
be zero,
which suggests that
\begin{align}
\label{eq:xithree_square_lambda2_sd_expc}
\Expc \Lambda_2 = (n-h)(p+3)\trace(\bM).
\end{align}

\noindent\textbf{Case $(d, s)$: $i\neq \pi^{\natural}(i)$ and $j = \pi^{\natural}(j)$.}
We have
\[
\Expc \Lambda_{2, 1}
=~&  \Expc  \bracket{
\bX_{\pi^{\natural}(i)}^{\rmt} \
\bM \bX_{j}\bX_{j}^{\rmt}\bX_{\pi^{\natural}(i)}} = \trace(\bM), \\
\Expc \Lambda_{2, 2}
=~& \Expc \bX_j^{\rmt}\bM \bX_j \bX_j^{\rmt}\bX_j = (p+2)\trace(\bM).
\]
Similar as above, we can verify both
$\Expc \Lambda_{2, 3}$ and $\Expc\Lambda_{2, 4}$ are zero,
which suggests that
\begin{align}
\label{eq:xithree_square_lambda2_ds_expc}
\Expc \Lambda_2 = (n-h)(p+3)\trace(\bM).
\end{align}

\newpage
\noindent\textbf{Case $(d, d)$: $i \neq \pi^{\natural}(i)$ and $j\neq \pi^{\natural}(j)$.}
Different from the above three cases, we have
$\Expc \Lambda_{2, 1}$ and $\Expc\Lambda_{2,2}$ be zero
and focus on the calculation of $\Expc \Lambda_{2, 3}$ and
$\Expc \Lambda_{2, 4}$, which proceeds as
\[
\Expc \Lambda_{2, 3} =~& \
\underbrace{\Expc \Bracket{
{\bX_{\pi^{\natural}(i)}}^{\rmt} \
\bM
\bX_{\pi^{\natural}(i)} \bX_{i}^{\rmt}
\bX_j }}_{p\Ind_{i=j}\Fnorm{\bB^{\natural}}^2 } +
\underbrace{\Expc \Bracket{
{\bX_{\pi^{\natural}(i)}}^{\rmt} \
\bM
\bX_{\pi^{\natural 2}(i)}\bX_{\pi^{\natural}(i)}^{\rmt}
\bX_j }}_{\Ind_{j = \pi^{\natural 2}(i)} \Fnorm{\bB^{\natural}}^2} \\
+~& \underbrace{\Expc \Bracket{
{\bX_{\pi^{\natural}(i)}}^{\rmt} \
\bM
\bX_{\pi^{\natural}(j)} \bX_{j}^{\rmt}
\bX_j }}_{p\Ind_{i=j} \Fnorm{\bB^{\natural}}^2}
+ \underbrace{\Expc \Bracket{
{\bX_{\pi^{\natural}(i)}}^{\rmt} \
\bM  \bX_{j}\bX_{\pi^{\natural -1}(j)}^{\rmt}
\bX_j }}_{\Ind_{\pi^{\natural}(i) = \pi^{\natural -1}(j)}
\Fnorm{\bB^{\natural}}^2} \\
=~&2\Bracket{p\Ind_{i=j} +
\Ind_{j = \pi^{\natural 2}(i)}} \trace(\bM); \\
\Expc \Lambda_{2, 4} =~&
\underbrace{\Expc \Bracket{ \bX_j^{\rmt} \
\bM {
\bX_{\pi^{\natural}(i)} \bX_{i}^{\rmt}
}{\bX_{\pi^{\natural}(i)}}}}_{\Ind_{i=j}\Fnorm{\bB^{\natural}}^2}
+ \underbrace{\Expc \Bracket{ {\bX_j}^{\rmt} \
\bM \bX_{\pi^{\natural 2}(i)}\bX_{\pi^{\natural}(i)}^{\rmt}
{\bX_{\pi^{\natural}(i)}}}}_{p\Ind_{j = \pi^{\natural 2}(i)}\Fnorm{\bB^{\natural}}^2 }  \\
+~& \underbrace{\Expc \Bracket{
{ \bX_j}^{\rmt} \
\bM \bX_{\pi^{\natural}(j)}\bX_{j}^{\rmt}
{\bX_{\pi^{\natural}(i)}}}}_{\Ind_{i= j} \Fnorm{\bB^{\natural}}^2}
+ \underbrace{\Expc \Bracket{
{ \bX_j}^{\rmt} \
\bM \bX_{j}\bX_{\pi^{\natural -1}(j)}^{\rmt}
{\bX_{\pi^{\natural}(i)}}} }_{p\Ind_{ j = \pi^{\natural 2}(i)}\Fnorm{\bB^{\natural}}^2} \\
=~& 2\Bracket{p\Ind_{j = \pi^{\natural 2}(i)} + \Ind_{i=j}}\trace(\bM),
\]
which suggests that
\begin{align}
\label{eq:xithree_square_lambda2_dd_expc}
\Expc \Lambda_2 = -2(n-h)(p+1)\bracket{\Ind_{j = \pi^{\natural 2}(i)} + \Ind_{i = j}}\trace(\bM).
\end{align}
Combing ~\eqref{eq:xithree_square_lambda2_ss_expc},
~\eqref{eq:xithree_square_lambda2_sd_expc},
~\eqref{eq:xithree_square_lambda2_ds_expc},
and ~\eqref{eq:xithree_square_lambda2_dd_expc}, we conclude
\begin{align}
\label{eq:xithree_square_lambda2_sum_expc}		
\Expc \Lambda_2 =
\frac{2p(n-h)^2}{n}\trace(\bM)\Bracket{1 + o(1)}.
\end{align}

\noindent\textbf{Step III.}
Then we turn to the calculation of  $\Expc \Lambda_3$. First we perform
the following decomposition
\[
\Lambda_3 =~& \
\underbrace{\bX_{\pi^{\natural}(i)}^{\rmt} \
\bDelta \bM \bDelta^{\rmt}\bX_{\pi^{\natural}(i)}}_{\Lambda_{3,1}}
+ \underbrace{\bX_j^{\rmt} \bDelta \bM \bDelta^{\rmt} \bX_j}_{\Lambda_{3,2}}
- \underbrace{\bX_{\pi^{\natural}(i)}^{\rmt} \
\bDelta \bM \bDelta^{\rmt}\bX_j}_{\Lambda_{3,3}}
- \underbrace{\bX_j^{\rmt} \bDelta \bM \bDelta^{\rmt}\bX_{\pi^{\natural}(i)}}_{\Lambda_{3, 4}}.
\]

\noindent\textbf{Case $(s, s)$: $i = \pi^{\natural}(i)$ and $j = \pi^{\natural}(j)$.}
We have
\[
\Expc \Lambda_{3, 1}
=~&
\underbrace{\Expc \bracket{\bX_{i}^{\rmt} \bX_i \bX_i^{\rmt} \
\bM \bX_i \bX_i^{\rmt} \bX_{i} }}_{\Expc \norm{\bX_i}{2}^4 \bX_i^{\rmt}\bM \bX_i}
+ \underbrace{\Expc \bracket{\bX_{i}^{\rmt} \bX_i \bX_i^{\rmt}  \
\bM\bX_j \bX_j^{\rmt} \bX_{i}}}_{(p+2)\Fnorm{\bB^{\natural}}^2} \\
+~& \underbrace{\Expc \bracket{\bX_{i}^{\rmt} \bX_j \bX_j^{\rmt} \
\bM \bX_i \bX_i^{\rmt} \bX_{i}} }_{(p+2)\Fnorm{\bB^{\natural}}^2}
+ \underbrace{\Expc \bracket{\bX_{i}^{\rmt} \bX_j \bX_j^{\rmt} \
\bM  \bX_j \bX_j^{\rmt} \bX_{i}}}_{(p+2)\Fnorm{\bB^{\natural}}^2} =
(p+2)(p+7)\trace(\bM); \\
\Expc \Lambda_{3, 2} =~& \Expc \bX_j^{\rmt} \bracket{\bX_i \bX_i^{\rmt} + \bX_j \bX_j^{\rmt}} \bM \bracket{\bX_i \bX_i^{\rmt} + \bX_j \bX_j^{\rmt}} \bX_j
=
(p+2)(p+7)\trace(\bM).
\]
As for $\Expc \Lambda_{3, 3}$ and $\Expc\Lambda_{3, 4}$, easily we can
verify that they are both zero and hence have
\begin{align}
\label{eq:xithree_square_lambda3_ss_expc}
\Expc \Lambda_3 = 2(p+2)(p+7)\trace(\bM)
= 2p^2\trace(\bM)\Bracket{1 + o(1)}.
\end{align}

\noindent\textbf{Case $(s, d)$: $i = \pi^{\natural}(i)$ and $j \neq \pi^{\natural}(j)$.}
We can write $\Lambda_{3, 1}$ as
\[
\Expc \Lambda_{3, 1}
=~& 
\underbrace{\Expc  \
\bracket{\bX_{i}^{\rmt}\bX_{i}\bX_{i}^{\rmt}  \bM
\bX_{i}\bX_{i}^{\rmt} \bX_{i}} }_{
\Expc\norm{\bX_i}{2}^4 \bX_i^{\rmt}\bM \bX_i
}
+ \underbrace{\Expc  \
\bracket{\bX_{i}^{\rmt}\bX_{i}\bX_{i}^{\rmt}  \bM
\bX_{\pi^{\natural}(j)}\bX_{j}^{\rmt}\bX_{i}}}_{0} \\
+~& \underbrace{\Expc  \
\bracket{\bX_{i}^{\rmt}\bX_{i}\bX_{i}^{\rmt}  \bM
\bX_{j}\bX_{\pi^{\natural -1}(j)}^{\rmt}\bX_{i}}}_{0} \\
+~& \underbrace{\Expc  \ 
\bracket{\bX_{i}^{\rmt}\bX_{j}\bX_{\pi^{\natural}(j)}^{\rmt} \bM
\bX_{i}\bX_{i}^{\rmt} \bX_{i}} }_{0}
+ \underbrace{\Expc  \ 
\bracket{ \bX_{i}^{\rmt}\bX_{j}\bX_{\pi^{\natural}(j)}^{\rmt}  \bM
\bX_{\pi^{\natural}(j)}\bX_{j}^{\rmt}\bX_{i} } }_{p\Fnorm{\bB^{\natural}}^2} \\
+~& \underbrace{\Expc  \ 
\bracket{ \bX_{i}^{\rmt}\bX_{j}\bX_{\pi^{\natural}(j)}^{\rmt} \bM
\bX_{j}\bX_{\pi^{\natural -1}(j)}^{\rmt}\bX_{i}}}_{
\Ind_{j = \pi^{\natural 2}(j)}\trace(\bM)
} \\
+~& \underbrace{\Expc  \ 
\bracket{ \bX_{i}^{\rmt}\bX_{\pi^{\natural -1}(j)}\bX_{j}^{\rmt}
\bM \bX_{i}\bX_{i}^{\rmt} \bX_{i}}}_{0} +
\underbrace{\Expc  \ 
\bracket{
\bX_{i}^{\rmt}\bX_{\pi^{\natural -1}(j)}\bX_{j}^{\rmt}
\bM \bX_{\pi^{\natural}(j)}\bX_{j}^{\rmt}\bX_{i} } }_{
\Ind_{j = \pi^{\natural 2}(j)}
\trace(\bM)
} \\
+~& \underbrace{\Expc  \ 
\bracket{
\bX_{i}^{\rmt}\bX_{\pi^{\natural -1}(j)}\bX_{j}^{\rmt}
\bM \bX_{j}\bX_{\pi^{\natural -1}(j)}^{\rmt}\bX_{i}}}_{p\Fnorm{\bB^{\natural}}^2} \\
=~& \bracket{p^2 + 8p + 8 + 2\Ind_{j = \pi^{\natural 2}(j)}}
\trace(\bM).
\]
Mean $\Lambda_{3, 2}$ can be written as
\[
\Expc \Lambda_{3,2}
=~& \underbrace{\Expc \bracket{ 
\bX_j^{\rmt}\bX_{i}\bX_{i}^{\rmt}  \bM
\bX_{i}\bX_{i}^{\rmt} \bX_j
} }_{(p+2)\Fnorm{\bB^{\natural}}^2 }
+ \underbrace{\Expc \bracket{ 
\bX_j^{\rmt}\bX_{i}\bX_{i}^{\rmt}
\bM \bX_{\pi^{\natural}(j)}\bX_{j}^{\rmt} \bX_j
} }_{0}  \\
+~& \underbrace{\Expc \bracket{ 
\bX_j^{\rmt}\bX_{i}\bX_{i}^{\rmt}
\bM
\bX_{j}\bX_{\pi^{\natural -1}(j)}^{\rmt}\bX_j
}}_{0}  \\
+~& \underbrace{\Expc \bracket{  
\bX_j^{\rmt}\bX_{j}\bX_{\pi^{\natural}(j)}^{\rmt}
\bM
\bX_{i}\bX_{i}^{\rmt} \bX_j
}}_{0}
+ \underbrace{\Expc \bracket{  
\bX_j^{\rmt}\bX_{j}\bX_{\pi^{\natural}(j)}^{\rmt}
\bM \bX_{\pi^{\natural}(j)}\bX_{j}^{\rmt} \bX_j
} }_{
\Expc \Fnorm{\bX_j}^4 \trace(\bM)
} \\
+~& \underbrace{\Expc \bracket{ 
\bX_j^{\rmt}\bX_{j}\bX_{\pi^{\natural}(j)}^{\rmt}
\bM
\bX_{j}\bX_{\pi^{\natural -1}(j)}^{\rmt}\bX_j
}}_{
\Ind_{\pi^{\natural -1}(j) = \pi^{\natural} (j)}
(p+2)\Fnorm{\bB^{\natural}}^2
}  \\
+~& \underbrace{\Expc \bracket{ 
\bX_j^{\rmt}\bX_{\pi^{\natural -1}(j)}\bX_{j}^{\rmt}
\bM \bX_{i}\bX_{i}^{\rmt} \bX_j
} }_{0}
+ \underbrace{\Expc \bracket{ 
\bX_j^{\rmt}\bX_{\pi^{\natural -1}(j)}\bX_{j}^{\rmt}
\bM \bX_{\pi^{\natural}(j)}\bX_{j}^{\rmt} \bX_j
}}_{\Ind_{\pi^{\natural -1}(j) = \pi^{\natural} (j)}(p+2)\Fnorm{\bB^{\natural}}^2 } \\
+~& \underbrace{\Expc \bracket{ 
\bX_j^{\rmt}\bX_{\pi^{\natural -1}(j)}\bX_{j}^{\rmt}
\bM \bX_{j}\bX_{\pi^{\natural -1}(j)}^{\rmt}\bX_j
}}_{(p+2)\Fnorm{\bB^{\natural}}^2} \\
=~&  (p+2)\bracket{p+2 + 2\Ind_{j = \pi^{\natural 2}(j)}}
\trace(\bM).
\]
And for $\Expc \Lambda_{3, 3}$ and $\Expc \Lambda_{3, 4}$,
easily we can verify that they are both zero. Then we conclude
\begin{align}
\label{eq:xithree_square_lambda3_sd_expc}
\Expc \Lambda_3 =
2\bracket{p^2 + 6p+6 + (p+3)\Ind_{j = \pi^{\natural 2}(j)}}
\trace(\bM) =
2p^2\trace(\bM)\Bracket{1 + o(1)}.
\end{align}

\noindent\textbf{Case $(d, s)$: $i\neq \pi^{\natural}(i)$ and $j = \pi^{\natural}(j)$.}
In this case, we can write $\Lambda_{3,1}$ as
\[
\Expc \Lambda_{3, 1}
=~&
\underbrace{\Expc  \ 
\bracket{
\bX_{\pi^{\natural}(i)}^{\rmt}\bX_{i}\bX_{\pi^{\natural}(i)}^{\rmt}
\bM \bX_{\pi^{\natural}(i)}\bX_{i}^{\rmt} \bX_{\pi^{\natural}(i)}
}}_{\Expc \Fnorm{\bX_i}^2 \bX_i^{\rmt}\bM\bX_i }
\\
+~&
\underbrace{\Expc  \ 
\bracket{
\bX_{\pi^{\natural}(i)}^{\rmt}\bX_{i}\bX_{\pi^{\natural}(i)}^{\rmt}
\bM \bX_{\pi^{\natural 2}(i)}\bX_{\pi^{\natural}(i)}^{\rmt}\bX_{\pi^{\natural}(i)}
}}_{
\Ind_{i = \pi^{\natural 2}(i)}
\Expc \Fnorm{\bX_i}^2 \bX_i^{\rmt}\bM\bX_i
} \\
+~&
\underbrace{\Expc  \ 
\bracket{
\bX_{\pi^{\natural}(i)}^{\rmt}\bX_{i}\bX_{\pi^{\natural}(i)}^{\rmt}
\bM \bX_{j}\bX_{j}^{\rmt}\bX_{\pi^{\natural}(i)}
}}_{0} \\
+~& \underbrace{\Expc  \ 
\bracket{
\bX_{\pi^{\natural}(i)}^{\rmt}\bX_{\pi^{\natural}(i)}\bX_{\pi^{\natural 2}(i)}^{\rmt}
\bM \bX_{\pi^{\natural}(i)}\bX_{i}^{\rmt} \bX_{\pi^{\natural}(i)}
}}_{
\Ind_{i = \pi^{\natural 2}(i)}
\Expc \Fnorm{\bX_i}^2 \bX_i^{\rmt}\bM\bX_i
} \\
+~&
\underbrace{\Expc  \ 
\bracket{
\bX_{\pi^{\natural}(i)}^{\rmt}\bX_{\pi^{\natural}(i)}\bX_{\pi^{\natural 2}(i)}^{\rmt}
\bM \bX_{\pi^{\natural 2}(i)}\bX_{\pi^{\natural}(i)}^{\rmt}\bX_{\pi^{\natural}(i)}
}}_{
\Expc \Fnorm{\bX_i}^4
\trace(\bM)
} \\
~&+
\underbrace{\Expc  \ 
\bracket{
\bX_{\pi^{\natural}(i)}^{\rmt}\bX_{\pi^{\natural}(i)}\bX_{\pi^{\natural 2}(i)}^{\rmt}
\bM \bX_{j}\bX_{j}^{\rmt}\bX_{\pi^{\natural}(i)}
}}_{0} \\
+~&
\underbrace{\Expc  \ 
\bracket{
\bX_{\pi^{\natural}(i)}^{\rmt}\bX_{j}\bX_{j}^{\rmt}
\bM
\bX_{\pi^{\natural}(i)}\bX_{i}^{\rmt} \bX_{\pi^{\natural}(i)}
}}_{0} \\
+~&
\underbrace{\Expc  \ 
\bracket{
\bX_{\pi^{\natural}(i)}^{\rmt}\bX_{j}\bX_{j}^{\rmt}
\bM
\bX_{\pi^{\natural 2}(i)}\bX_{\pi^{\natural}(i)}^{\rmt}\bX_{\pi^{\natural}(i)}
}}_{0} \\
+~&
\underbrace{\Expc  \ 
\bracket{
\bX_{\pi^{\natural}(i)}^{\rmt}\bX_{j}\bX_{j}^{\rmt}
\bM
\bX_{j}\bX_{j}^{\rmt}\bX_{\pi^{\natural}(i)}
}}_{\Expc \Fnorm{\bX_i}^2 \bX_i^{\rmt}\bM \bX_i} \\
=~& (p+2)\bracket{p+2 + 2\Ind_{i = \pi^{\natural 2}(i)}}
\trace(\bM).
\]
We consider $\Lambda_{3, 2}$ as
\[
\Expc \Lambda_{3,2}
=~&
\underbrace{\Expc
\bX_j^{\rmt}\bX_{i}\bX_{\pi^{\natural}(i)}^{\rmt}
\bM
\bX_{\pi^{\natural}(i)}\bX_{i}^{\rmt} \bX_j
}_{p\trace(\bM)} +
\underbrace{
\Expc
\bX_j^{\rmt}\bX_{i}\bX_{\pi^{\natural}(i)}^{\rmt}
\bM \bX_{\pi^{\natural 2}(i)}\bX_{\pi^{\natural}(i)}^{\rmt}\bX_j
}_{\Ind_{i = \pi^{\natural 2}(i)}\trace(\bM) } \\
+~& \underbrace{\Expc
\bX_j^{\rmt}\bX_{i}\bX_{\pi^{\natural}(i)}^{\rmt}  \bM
\bX_{j}\bX_{j}^{\rmt}\bX_j
}_{0}  \\
+~&  \underbrace{\Expc
\bX_j^{\rmt}\bX_{\pi^{\natural}(i)}\bX_{\pi^{\natural 2}(i)}^{\rmt}
\bM \bX_{\pi^{\natural}(i)}\bX_{i}^{\rmt} \bX_j
}_{\Ind_{i = \pi^{\natural 2}(i)}
\trace(\bM)
}  +
\underbrace{\Expc
\bX_j^{\rmt}\bX_{\pi^{\natural}(i)}\bX_{\pi^{\natural 2}(i)}^{\rmt}
\bM
\bX_{\pi^{\natural 2}(i)}\bX_{\pi^{\natural}(i)}^{\rmt}\bX_j
}_{p\trace(\bM)} \\
+~&
\underbrace{\Expc
\bX_j^{\rmt}\bX_{\pi^{\natural}(i)}\bX_{\pi^{\natural 2}(i)}^{\rmt}
\bM
\bX_{j}\bX_{j}^{\rmt}\bX_j
}_{0} \\
+~& \underbrace{\Expc
\bX_j^{\rmt}\bX_{j}\bX_{j}^{\rmt} \bM
\bX_{\pi^{\natural}(i)}\bX_{i}^{\rmt} \bX_j}_{0}
+ \underbrace{\Expc
\bX_j^{\rmt}\bX_{j}\bX_{j}^{\rmt}
\bM
\bX_{\pi^{\natural 2}(i)}\bX_{\pi^{\natural}(i)}^{\rmt}\bX_j
}_{0}
+ \underbrace{\Expc
\bX_j^{\rmt}\bX_{j}\bX_{j}^{\rmt}
\bM \bX_{j}\bX_{j}^{\rmt}\bX_j
}_{\Expc \Fnorm{\bX_i}^4 \bX_i^{\rmt}\bM \bX_i} \\
=~& (p^2 + 8p + 8 + 2\Ind_{i = \pi^{\natural 2}(i)})\trace(\bM).
\]
Similarly, as above, we can verify that
$\Expc \Lambda_{3,3} = 0$  and $\Expc \Lambda_{3, 4} = 0$.
Hence, we can conclude
\begin{align}
\label{eq:xithree_square_lambda3_ds_expc}
\Expc \Lambda_3 =
2\bracket{p^2 + 6p+6 + (p+3)\Ind_{i = \pi^{\natural 2}(i)}}
\trace(\bM) =
2p^2\trace(\bM)\Bracket{1 + o(1)}.
\end{align}

\noindent\textbf{Case $(d, d)$: $i \neq \pi^{\natural}(i)$ and $j\neq \pi^{\natural}(j)$.}
We write $\Lambda_{3, 1}$ as
\[
\Expc \Lambda_{3, 1}
=~&
\underbrace{\Expc  \ 
\bX_{\pi^{\natural}(i)}^{\rmt}\bX_{i}\bX_{\pi^{\natural}(i)}^{\rmt}
\bM
\bX_{\pi^{\natural}(i)}\bX_{i}^{\rmt} \bX_{\pi^{\natural}(i)}
}_{
\Expc \Fnorm{\bX_i}^2 \bX_i^{\rmt}\bM \bX_i
} +
\underbrace{\Expc  \ 
\bX_{\pi^{\natural}(i)}^{\rmt}\bX_{i}\bX_{\pi^{\natural}(i)}^{\rmt}
\bM
\bX_{\pi^{\natural 2}(i)} \bX_{\pi^{\natural}(i)}^{\rmt}\bX_{\pi^{\natural}(i)}
}_{
\Ind_{i = \pi^{\natural 2}(i)}
\Expc \Fnorm{\bX_i}^2 \bX_i^{\rmt}\bM \bX_i
} \\
+~& \underbrace{\Expc  \ 
\bX_{\pi^{\natural}(i)}^{\rmt}\bX_{i}\bX_{\pi^{\natural}(i)}^{\rmt}
\bM
\bX_{\pi^{\natural}(j)}\bX_{j}^{\rmt} \bX_{\pi^{\natural}(i)}
}_{
\Ind_{i = j}\Expc \Fnorm{\bX_i}^2 \bX_i^{\rmt}\bM \bX_i
} + \underbrace{\Expc  \ 
\bX_{\pi^{\natural}(i)}^{\rmt}\bX_{i}\bX_{\pi^{\natural}(i)}^{\rmt}
\bM
\bX_{j}\bX_{\pi^{\natural -1}(j)}^{\rmt}\bX_{\pi^{\natural}(i)}
}_{
\Ind_{i = j}\Ind_{i = \pi^{\natural 2}(i)}
\Expc\Fnorm{\bX_i}^2 \bX_i^{\rmt}\bM \bX_i
} \\
+~& \underbrace{\Expc  \ 
\bX_{\pi^{\natural}(i)}^{\rmt}\bX_{\pi^{\natural}(i)}\bX_{\pi^{\natural 2}(i)}^{\rmt}
\bM
\bX_{\pi^{\natural}(i)}\bX_{i}^{\rmt} \bX_{\pi^{\natural}(i)}
}_{\Ind_{i = \pi^{\natural 2}(i)}
\Expc \Fnorm{\bX_i}^2 \bX_i^{\rmt}\bM \bX_i
} + \underbrace{\Expc  \ 
\bX_{\pi^{\natural}(i)}^{\rmt}\bX_{\pi^{\natural}(i)}\bX_{\pi^{\natural 2}(i)}^{\rmt}
\bM
\bX_{\pi^{\natural 2}(i)} \bX_{\pi^{\natural}(i)}^{\rmt}\bX_{\pi^{\natural}(i)}
}_{
\Expc\Fnorm{\bX_i}^4 \trace(\bM)
} \\
+~& \underbrace{\Expc  \ 
{
\bX_{\pi^{\natural}(i)}^{\rmt}\bX_{\pi^{\natural}(i)}\bX_{\pi^{\natural 2}(i)}^{\rmt}
}
\bM
{
\bX_{\pi^{\natural}(j)}\bX_{j}^{\rmt} \bX_{\pi^{\natural}(i)}
}}_{
\Ind_{i = j}\Ind_{i = \pi^{\natural 2}(i)}
\Expc \Fnorm{\bX_i}^2 \bX_i^{\rmt}\bM \bX_i
} + \underbrace{\Expc  \ 
{
\bX_{\pi^{\natural}(i)}^{\rmt}\bX_{\pi^{\natural}(i)}\bX_{\pi^{\natural 2}(i)}^{\rmt}
}
\bM
{
\bX_{j}\bX_{\pi^{\natural -1}(j)}^{\rmt}\bX_{\pi^{\natural}(i)}
}}_{
\Ind_{j = \pi^{\natural 2}(i)}
\Expc\Fnorm{\bX_i}^4 \trace(\bM)
}
 \\
+~&
\underbrace{\Expc  \ 
\bX_{\pi^{\natural}(i)}^{\rmt}\bX_{j}\bX_{\pi^{\natural}(j)}^{\rmt}
\bM
\bX_{\pi^{\natural}(i)}\bX_{i}^{\rmt} \bX_{\pi^{\natural}(i)}
}_{
\Ind_{i = j}\Expc \Fnorm{\bX_i}^2 \bX_i^{\rmt}\bM \bX_i
} +
\underbrace{\Expc  \ 
\bX_{\pi^{\natural}(i)}^{\rmt}\bX_{j}\bX_{\pi^{\natural}(j)}^{\rmt}
\bM
\bX_{\pi^{\natural 2}(i)} \bX_{\pi^{\natural}(i)}^{\rmt}\bX_{\pi^{\natural}(i)}
}_{
\Ind_{i = j}\Ind_{i = \pi^{\natural 2}(i)}
\Expc \Fnorm{\bX_i}^2 \bX_i^{\rmt}\bM \bX_i
} \\
+~&
\underbrace{\Expc  \ 
\bX_{\pi^{\natural}(i)}^{\rmt}\bX_{j}\bX_{\pi^{\natural}(j)}^{\rmt}
\bM
\bX_{\pi^{\natural}(j)}\bX_{j}^{\rmt} \bX_{\pi^{\natural}(i)}
}_{
\Ind_{i=j}\Expc \Fnorm{\bX_i}^2 \bX_i^{\rmt}\bM\bX_i +
\Ind_{i\neq j}p\trace(\bM)
} +
\underbrace{\Expc  \ 
{
\bX_{\pi^{\natural}(i)}^{\rmt}\bX_{j}\bX_{\pi^{\natural}(j)}^{\rmt}
}
\bM
{
\bX_{j}\bX_{\pi^{\natural -1}(j)}^{\rmt}\bX_{\pi^{\natural}(i)}
}}_{
\Ind_{j = \pi^{\natural 2}(j)}
\bracket{
\Ind_{i = j}\Expc \Fnorm{\bX_i}^2 \bX_i^{\rmt}\bM \bX_i + \
\Ind_{i\neq j}\trace(\bM)
}
} \\
+~& \underbrace{\Expc  \ 
{
\bX_{\pi^{\natural}(i)}^{\rmt}\bX_{\pi^{\natural -1}(j)}\bX_{j}^{\rmt}
}
\bM
{
\bX_{\pi^{\natural}(i)}\bX_{i}^{\rmt} \bX_{\pi^{\natural}(i)}
}}_{\Ind_{i = j}\Ind_{i = \pi^{\natural 2}(i)}
\Expc \Fnorm{\bX_i}^2 \bX_i^{\rmt}\bM\bX_i
} + \underbrace{\Expc  \ 
\bX_{\pi^{\natural}(i)}^{\rmt}\bX_{\pi^{\natural -1}(j)}\bX_{j}^{\rmt}
\bM
\bX_{\pi^{\natural 2}(i)} \bX_{\pi^{\natural}(i)}^{\rmt}\bX_{\pi^{\natural}(i)}
}_{
\Ind_{j= \pi^{\natural 2}(i)}
\Expc\Fnorm{\bX_i}^4\trace(\bM)
} \\
+~&
\underbrace{\Expc  \ 
\bX_{\pi^{\natural}(i)}^{\rmt}\bX_{\pi^{\natural -1}(j)}\bX_{j}^{\rmt}
\bM
\bX_{\pi^{\natural}(j)}\bX_{j}^{\rmt} \bX_{\pi^{\natural}(i)}
}_{
\Ind_{j = \pi^{\natural 2}(j)}
\Bracket{
\Ind_{i= j} \Expc \Fnorm{\bX_i}^2 \bX_i^{\rmt}\bM \bX_i +
\Ind_{i \neq j}\trace(\bM)
}
} + \underbrace{\Expc  \ 
\bX_{\pi^{\natural}(i)}^{\rmt}\bX_{\pi^{\natural -1}(j)}\bX_{j}^{\rmt}
\bM
\bX_{j}\bX_{\pi^{\natural -1}(j)}^{\rmt}\bX_{\pi^{\natural}(i)}
}_{
\Ind_{j = \pi^{\natural 2}(i)}\Expc \Fnorm{\bX_i}^4 \trace(\bM)
+ \Ind_{j \neq \pi^{\natural 2}(i)} p \trace(\bM)
} \\ 
=~&
\bracket{p^2 + 5p + 2}\trace(\bM) +
\Ind_{i = \pi^{\natural 2}(i)} 2(p+2)\trace(\bM) +
\Ind_{j = \pi^{\natural 2}(i)} (3p^2 + 5p) \trace(\bM)  \\
+~& 2\Ind_{j = \pi^{\natural 2}(j) }\trace(\bM)
+ \Ind_{i = j}2(p+3)\trace(\bM) +
\Ind_{i = j}\Ind_{i = \pi^{\natural 2}(i)}
2(3p+5)\trace(\bM).
\]
We consider $\Lambda_{3, 2}$ as
\[
\Expc \Lambda_{3, 2}
=~&
\underbrace{\Expc 
\bX_j^{\rmt}\bX_{i}\bX_{\pi^{\natural}(i)}^{\rmt}
\bM
\bX_{\pi^{\natural}(i)}\bX_{i}^{\rmt} \bX_j
}_{
\Ind_{i = j}\Expc \Fnorm{\bX_i}^4 \trace(\bM) +
\Ind_{i \neq j} p \trace(\bM)
} +
\underbrace{\Expc  
 \bX_j^{\rmt}\bX_{i}\bX_{\pi^{\natural}(i)}^{\rmt}
\bM
\bX_{\pi^{\natural 2}(i)} \bX_{\pi^{\natural}(i)}^{\rmt}\bX_j
}_{
\Ind_{i = \pi^{\natural 2}(i)}
\Bracket{
\Ind_{i = j}\Expc \Fnorm{\bX_i}^2 \bX_i^{\rmt}\bM \bX_i +
\Ind_{i \neq j} \trace(\bM)
}
} \\
+~&
\underbrace{\Expc  
\bX_j^{\rmt}\bX_{i}\bX_{\pi^{\natural}(i)}^{\rmt}
\bM
\bX_{\pi^{\natural}(j)}\bX_{j}^{\rmt}\bX_j
}_{
\Ind_{i = j}\Expc\Fnorm{\bX_i}^4 \trace(\bM)
}+
\underbrace{\Expc 
\bX_j^{\rmt}\bX_{i}\bX_{\pi^{\natural}(i)}^{\rmt}
\bM
\bX_{j}\bX_{\pi^{\natural -1}(j)}^{\rmt}\bX_j
}_{
\Ind_{i = j}\Ind_{i = \pi^{\natural 2}(i)}
\Expc \Fnorm{\bX_i}^2 \bX_i^{\rmt}\bM \bX_i
} \\
+~&
\underbrace{\Expc 
\bX_j^{\rmt}\bX_{\pi^{\natural}(i)}\bX_{\pi^{\natural 2}(i)}^{\rmt}
\bM
\bX_{\pi^{\natural}(i)}\bX_{i}^{\rmt} \bX_j
}_{
\Ind_{i = \pi^{\natural 2}(i)}
\Bracket{
\Ind_{i = j}\Expc \Fnorm{\bX_i}^2 \bX_i^{\rmt}\bM \bX_i +
\Ind_{i \neq j}\trace(\bM)
}} +
\underbrace{\Expc  
\bX_j^{\rmt}\bX_{\pi^{\natural}(i)}\bX_{\pi^{\natural 2}(i)}^{\rmt}
\bM
\bX_{\pi^{\natural 2}(i)} \bX_{\pi^{\natural}(i)}^{\rmt}\bX_j
}_{
\Ind_{j = \pi^{\natural 2}(i)}\Expc \Fnorm{\bX_i}^2 \bX_i^{\rmt}\bM \bX_i +
\Ind_{j \neq \pi^{\natural 2}(i)}p\trace(\bM)
} \\
+~&
\underbrace{\Expc 
\bX_j^{\rmt}\bX_{\pi^{\natural}(i)}\bX_{\pi^{\natural 2}(i)}^{\rmt}
\bM
\bX_{\pi^{\natural}(j)}\bX_{j}^{\rmt}\bX_j
}_{
\Ind_{i = j}\Ind_{i = \pi^{\natural 2}(i)}
\Expc \Fnorm{\bX_i}^2 \bX_i^{\rmt}\bM \bX_i
}
+
\underbrace{\Expc
\bX_j^{\rmt}\bX_{\pi^{\natural}(i)}\bX_{\pi^{\natural 2}(i)}^{\rmt}
\bM
\bX_{j}\bX_{\pi^{\natural -1}(j)}^{\rmt}\bX_j
}_{
\Ind_{j = \pi^{\natural 2}(i)}
\Expc \Fnorm{\bX_i}^2 \bX_i^{\rmt}\bM \bX_i
} \\
+~&
\underbrace{\Expc 
\bX_j^{\rmt}\bX_{j}\bX_{\pi^{\natural}(j)}^{\rmt}
\bM
\bX_{\pi^{\natural}(i)}\bX_{i}^{\rmt} \bX_j
}_{
\Ind_{i =j}\Expc\Fnorm{\bX_i}^4 \trace(\bM)
} +
\underbrace{\Expc 
\bX_j^{\rmt}\bX_{j}\bX_{\pi^{\natural}(j)}^{\rmt}
\bM
\bX_{\pi^{\natural 2}(i)} \bX_{\pi^{\natural}(i)}^{\rmt}\bX_j
}_{
\Ind_{i = j}\Ind_{i = \pi^{\natural 2}(i)}
\Expc \Fnorm{\bX_i}^2 \bX_i^{\rmt}\bM \bX_i
} \\
+~&
\underbrace{
\Expc 
\bX_j^{\rmt}\bX_{j}\bX_{\pi^{\natural}(j)}^{\rmt}
\bM
\bX_{\pi^{\natural}(j)}\bX_{j}^{\rmt}\bX_j
}_{
\Expc \Fnorm{\bX_i}^4 \trace(\bM)
}
+
\underbrace{\Expc  
\bX_j^{\rmt}\bX_{j}\bX_{\pi^{\natural}(j)}^{\rmt}
\bM
\bX_{j}\bX_{\pi^{\natural -1}(j)}^{\rmt}\bX_j
}_{
\Ind_{j = \pi^{\natural 2}(j)}
\Expc \Fnorm{\bX_i}^2 \bX_i^{\rmt}\bM \bX_i
}\\
+~&
\underbrace{\Expc  
\bX_j^{\rmt}\bX_{\pi^{\natural -1}(j)}\bX_{j}^{\rmt}
\bM
\bX_{\pi^{\natural}(i)}\bX_{i}^{\rmt} \bX_j
}_{
\Ind_{i = j}\Ind_{i = \pi^{\natural 2}(i)}
\Expc\Fnorm{\bX_i}^2 \bX_i^{\rmt}\bM \bX_i
} +
\underbrace{\Expc  
\bX_j^{\rmt}\bX_{\pi^{\natural -1}(j)}\bX_{j}^{\rmt}
\bM
\bX_{\pi^{\natural 2}(i)} \bX_{\pi^{\natural}(i)}^{\rmt}\bX_j
}_{
\Ind_{j = \pi^{\natural 2}(i)}
\Expc \Fnorm{\bX_i}^2 \bX_i^{\rmt}\bM \bX_i
} \\
+~&
\underbrace{\Expc  
 \bX_j^{\rmt}\bX_{\pi^{\natural -1}(j)}\bX_{j}^{\rmt}
\bM
\bX_{\pi^{\natural}(j)}\bX_{j}^{\rmt}\bX_j
}_{
\Ind_{j = \pi^{\natural 2}(j)}
\Expc \Fnorm{\bX_i}^2 \bX_i^{\rmt}\bM \bX_i
} +
\underbrace{\Expc  
\bX_j^{\rmt}\bX_{\pi^{\natural -1}(j)}\bX_{j}^{\rmt}
\bM
\bX_{j}\bX_{\pi^{\natural -1}(j)}^{\rmt}\bX_j
}_{\Expc \Fnorm{\bX_i}^2 \bX_i^{\rmt}\bM \bX_i} \\ 
=~&
\bracket{p^2 + 5p + 2} \trace(\bM)  +
\Ind_{j = \pi^{\natural 2}(j)}
2(p+2)\trace(\bM) +\Ind_{i = j} \bracket{3p^2 + 5p} \trace(\bM) \\
+~&
2\Ind_{i = \pi^{\natural 2}(i)} \trace(\bM) +
\Ind_{j = \pi^{\natural 2}(i)}2(p+3)\trace(\bM) +
\Ind_{i = j}\Ind_{i = \pi^{\natural 2}(i)}
2\bracket{3p + 5}\trace(\bM).
\]
We consider $\Lambda_{3, 3}$ as
\[
\Expc \Lambda_{3,3}
=~&
\underbrace{\Expc  \
\bX_{\pi^{\natural}(i)}^{\rmt}\bX_{i}\bX_{\pi^{\natural}(i)}^{\rmt}
\bM
\bX_{\pi^{\natural}(i)}\bX_{i}^{\rmt} \bX_j}_{0}
+
\underbrace{\Expc  
\bX_{\pi^{\natural}(i)}^{\rmt}\bX_{i}\bX_{\pi^{\natural}(i)}^{\rmt}
\bM
\bX_{\pi^{\natural 2}(i)} \bX_{\pi^{\natural}(i)}^{\rmt}\bX_j
}_{0} \\
+~&
\underbrace{\Expc  \
\bX_{\pi^{\natural}(i)}^{\rmt}\bX_{i}\bX_{\pi^{\natural}(i)}^{\rmt}
\bM
\bX_{\pi^{\natural}(j)}\bX_{j}^{\rmt} \bX_j
}_{\Ind_{i = \pi^{\natural}(j)}p\trace(\bM)}
+
\underbrace{\Expc  \ 
\bX_{\pi^{\natural}(i)}^{\rmt}\bX_{i}\bX_{\pi^{\natural}(i)}^{\rmt}
\bM
\bX_{j}\bX_{\pi^{\natural -1}(j)}^{\rmt}\bX_j
}_{0} \\
+~&
\underbrace{\Expc  \  
\bX_{\pi^{\natural}(i)}^{\rmt}\bX_{\pi^{\natural}(i)}\bX_{\pi^{\natural 2}(i)}^{\rmt}
\bM
\bX_{\pi^{\natural}(i)}\bX_{i}^{\rmt} \bX_j
}_{0}
+
\underbrace{\Expc  \  
\bX_{\pi^{\natural}(i)}^{\rmt}\bX_{\pi^{\natural}(i)}\bX_{\pi^{\natural 2}(i)}^{\rmt}
\bM
\bX_{\pi^{\natural 2}(i)} \bX_{\pi^{\natural}(i)}^{\rmt}\bX_j
}_{0} \\
+~&
\underbrace{\Expc  \
\bX_{\pi^{\natural}(i)}^{\rmt}\bX_{\pi^{\natural}(i)}\bX_{\pi^{\natural 2}(i)}^{\rmt}
\bM
\bX_{\pi^{\natural}(j)}\bX_{j}^{\rmt} \bX_j
}_{0}
+
\underbrace{\Expc  \ 
\bX_{\pi^{\natural}(i)}^{\rmt}\bX_{\pi^{\natural}(i)}\bX_{\pi^{\natural 2}(i)}^{\rmt}
\bM
\bX_{j}\bX_{\pi^{\natural -1}(j)}^{\rmt}\bX_j
}_{\Ind_{j = \pi^{\natural 3}(i)}p \trace(\bM) } \\
+~&
\underbrace{\Expc  \
\bX_{\pi^{\natural}(i)}^{\rmt}\bX_{j}\bX_{\pi^{\natural}(j)}^{\rmt}
\bM
\bX_{\pi^{\natural}(i)}\bX_{i}^{\rmt} \bX_j
}_{\Ind_{i = \pi^{\natural}(j)}\trace(\bM)}
+
\underbrace{\Expc  \ 
\bX_{\pi^{\natural}(i)}^{\rmt}\bX_{j}\bX_{\pi^{\natural}(j)}^{\rmt}
\bM
\bX_{\pi^{\natural 2}(i)} \bX_{\pi^{\natural}(i)}^{\rmt}\bX_j
}_{0} \\
+~&
\underbrace{\Expc  \
\bX_{\pi^{\natural}(i)}^{\rmt}\bX_{j}\bX_{\pi^{\natural}(j)}^{\rmt}
\bM
\bX_{\pi^{\natural}(j)}\bX_{j}^{\rmt} \bX_j
}_{0} +
\underbrace{\Expc  \
\bX_{\pi^{\natural}(i)}^{\rmt}\bX_{j}\bX_{\pi^{\natural}(j)}^{\rmt}
\bM
\bX_{j}\bX_{\pi^{\natural -1}(j)}^{\rmt}\bX_j
}_{0} \\
+~&
\underbrace{\Expc  \ 
\bX_{\pi^{\natural}(i)}^{\rmt}\bX_{\pi^{\natural -1}(j)}\bX_{j}^{\rmt}
\bM
\bX_{\pi^{\natural}(i)}\bX_{i}^{\rmt} \bX_j
}_{0} +
\underbrace{\Expc  \ 
\bX_{\pi^{\natural}(i)}^{\rmt}\bX_{\pi^{\natural -1}(j)}\bX_{j}^{\rmt}
\bM
\bX_{\pi^{\natural 2}(i)} \bX_{\pi^{\natural}(i)}^{\rmt}\bX_j
}_{\Ind_{j = \pi^{\natural 3}(i)}\trace(\bM) } \\
+~&
\underbrace{\Expc  \
\bX_{\pi^{\natural}(i)}^{\rmt}\bX_{\pi^{\natural -1}(j)}\bX_{j}^{\rmt}
\bM
\bX_{\pi^{\natural}(j)}\bX_{j}^{\rmt} \bX_j
}_{0} +
\underbrace{\Expc  \  
\bX_{\pi^{\natural}(i)}^{\rmt}\bX_{\pi^{\natural -1}(j)}\bX_{j}^{\rmt}
\bM
\bX_{j}\bX_{\pi^{\natural -1}(j)}^{\rmt}\bX_j
}_{0} \\
=~&
(p+1)\Bracket{\Ind_{i = \pi^{\natural}(j)} + \Ind_{j = \pi^{\natural 3}(i)}} \trace(\bM).
\]
Then we consider $\Lambda_{3, 4}$ as
\[
\Expc \Lambda_{3, 4}
=~&  \underbrace{\Expc   
\bX_j^{\rmt}\bX_{i}\bX_{\pi^{\natural}(i)}^{\rmt}
\bM
\bX_{\pi^{\natural}(i)}\bX_{i}^{\rmt} \bX_{\pi^{\natural}(i)}
}_{0} +
\underbrace{\Expc   
\bX_j^{\rmt}\bX_{i}\bX_{\pi^{\natural}(i)}^{\rmt}
\bM
\bX_{\pi^{\natural 2}(i)} \bX_{\pi^{\natural}(i)}^{\rmt}\bX_{\pi^{\natural}(i)}
}_{0}  \\
+~&
\underbrace{\Expc   
\bX_j^{\rmt}\bX_{i}\bX_{\pi^{\natural}(i)}^{\rmt}
\bM
\bX_{\pi^{\natural}(j)}\bX_{j}^{\rmt}\bX_{\pi^{\natural}(i)}
}_{\Ind_{i = \pi^{\natural}(j)}\trace(\bM)}
+
\underbrace{
\Expc   
\bX_j^{\rmt}\bX_{i}\bX_{\pi^{\natural}(i)}^{\rmt}
\bM
\bX_{j}\bX_{\pi^{\natural -1}(j)}^{\rmt}\bX_{\pi^{\natural}(i)}
}_{0}  \\
+~&
\underbrace{\Expc   
\bX_j^{\rmt}\bX_{\pi^{\natural}(i)}\bX_{\pi^{\natural 2}(i)}^{\rmt}
\bM
\bX_{\pi^{\natural}(i)}\bX_{i}^{\rmt} \bX_{\pi^{\natural}(i)}
}_{0}  +
\underbrace{\Expc   
\bX_j^{\rmt}\bX_{\pi^{\natural}(i)}\bX_{\pi^{\natural 2}(i)}^{\rmt}
\bM
\bX_{\pi^{\natural 2}(i)} \bX_{\pi^{\natural}(i)}^{\rmt}\bX_{\pi^{\natural}(i)}
}_{0}  \\
+~&
\underbrace{\Expc   
\bX_j^{\rmt}\bX_{\pi^{\natural}(i)}\bX_{\pi^{\natural 2}(i)}^{\rmt}
\bM
\bX_{\pi^{\natural}(j)}\bX_{j}^{\rmt}\bX_{\pi^{\natural}(i)}
}_{0} +
\underbrace{\Expc  
\bX_j^{\rmt}\bX_{\pi^{\natural}(i)}\bX_{\pi^{\natural 2}(i)}^{\rmt}
\bM
\bX_{j}\bX_{\pi^{\natural -1}(j)}^{\rmt}\bX_{\pi^{\natural}(i)}
}_{
\Ind_{j = \pi^{\natural 3}(i)} \trace(\bM)
}  \\
+~&
\underbrace{\Expc   
\bX_j^{\rmt}\bX_{j}\bX_{\pi^{\natural}(j)}^{\rmt}
\bM
\bX_{\pi^{\natural}(i)}\bX_{i}^{\rmt} \bX_{\pi^{\natural}(i)}
}_{
p\Ind_{i = \pi^{\natural}(j) } \trace(\bM)
} +
\underbrace{\Expc   
\bX_j^{\rmt}\bX_{j}\bX_{\pi^{\natural}(j)}^{\rmt}
\bM
\bX_{\pi^{\natural 2}(i)} \bX_{\pi^{\natural}(i)}^{\rmt}\bX_{\pi^{\natural}(i)}
}_{0}  \\
+~&
\underbrace{\Expc   
\bX_j^{\rmt}\bX_{j}\bX_{\pi^{\natural}(j)}^{\rmt}
\bM
\bX_{\pi^{\natural}(j)}\bX_{j}^{\rmt}\bX_{\pi^{\natural}(i)}
}_{0} +
\underbrace{\Expc  
\bX_j^{\rmt}\bX_{j}\bX_{\pi^{\natural}(j)}^{\rmt}
\bM
\bX_{j}\bX_{\pi^{\natural -1}(j)}^{\rmt}\bX_{\pi^{\natural}(i)}
}_{0}  \\
+~&
\underbrace{\Expc  
\bX_j^{\rmt}\bX_{\pi^{\natural -1}(j)}\bX_{j}^{\rmt}
\bM
\bX_{\pi^{\natural}(i)}\bX_{i}^{\rmt} \bX_{\pi^{\natural}(i)}
}_{0}  +
\underbrace{\Expc   
\bX_j^{\rmt}\bX_{\pi^{\natural -1}(j)}\bX_{j}^{\rmt}
\bM
\bX_{\pi^{\natural 2}(i)} \bX_{\pi^{\natural}(i)}^{\rmt}\bX_{\pi^{\natural}(i)}
}_{p\Ind_{j = \pi^{\natural 3}(i)}\trace(\bM) } \\
+~&
\underbrace{\Expc  
\bX_j^{\rmt}\bX_{\pi^{\natural -1}(j)}\bX_{j}^{\rmt}
\bM
\bX_{\pi^{\natural}(j)}\bX_{j}^{\rmt}\bX_{\pi^{\natural}(i)}
}_{0}  +
\underbrace{\Expc   
\bX_j^{\rmt}\bX_{\pi^{\natural -1}(j)}\bX_{j}^{\rmt}
\bM
\bX_{j}\bX_{\pi^{\natural -1}(j)}^{\rmt}\bX_{\pi^{\natural}(i)} }_{0} \\
=~&
\bracket{p+1}\bracket{\Ind_{i = \pi^{\natural}(j)}
+ \Ind_{j = \pi^{\natural 3}(i)}
}\trace(\bM).
\]
In summary, we have
\begin{align}
\label{eq:xithree_square_lambda3_dd_expc}
\Expc \Lambda_3 =~&   2\bracket{p^2 + 5p + 2}\trace(\bM)
+  2(p+3)\Bracket{\Ind_{i = \pi^{\natural 2}(i)} + \Ind_{j = \pi^{\natural 2}(j)}} \trace(\bM) \notag
\\
+~& \bracket{3p^2 + 7p + 6} \bracket{\Ind_{i = j}  + \Ind_{j = \pi^{\natural 2}(i)}}
 \trace(\bM) +
\Ind_{i = j}\Ind_{i = \pi^{\natural 2}(i)}
4\bracket{3p + 5}\trace(\bM) \notag \\
-~& 2(p+1)\Bracket{\Ind_{i = \pi^{\natural}(j)} + \Ind_{j = \pi^{\natural 3}(i)}} \trace(\bM) = 2p^2\trace(\bM)\Bracket{1 + o(1)}.
\end{align}
Combining ~\eqref{eq:xithree_square_lambda3_ss_expc},
~\eqref{eq:xithree_square_lambda3_sd_expc},
~\eqref{eq:xithree_square_lambda3_ds_expc}, and
~\eqref{eq:xithree_square_lambda3_dd_expc} then yields
\begin{align}
\label{eq:xithree_square_lambda3_sum_expc}	
\Expc \Lambda_3 = 2p^2\trace(\bM)\Bracket{1 + o(1)}.
\end{align}
The proof is then completed by
~\eqref{eq:xithree_square_decompose},
~\eqref{eq:xithree_square_lambda1_sum_expc},
~\eqref{eq:xithree_square_lambda2_sum_expc},
and ~\eqref{eq:xithree_square_lambda3_sum_expc}.

\end{proof}

\begin{lemma}
\label{lemma:xifour_square_expc}
We have
\[
\Expc \Xi_4^2 =~&  m(m+1)
\Bracket{p\bracket{p+2}\bracket{\Ind_{i = \pi^{\natural}(i)} + \Ind_{i = j}}
+ p\bracket{\Ind_{i \neq \pi^{\natural}(i)} + \Ind_{i \neq j}}} +
2mp(n+p+1) \\
=~& \frac{(n-h)m^2p^2}{n}\Bracket{1 + o(1)},
\]	
where $\Xi_4$ is defined in ~\eqref{eq:xi_decomposition}.
\end{lemma}

\begin{proof}
For the conciseness of notation, we define
$\bGamma$ as $\bX \bracket{\bX_{\pi^{\natural}(i)} - \bX_j}
\bracket{\bX_{\pi^{\natural}(i)} - \bX_j}^{\rmt} \bX^{\rmt}$ and
hence have
\[
\Expc \Xi_4^2 =
\Expc {\bW_i^{\rmt} \bW^{\rmt} \bGamma\bW\bW_i}.
\]
We begin the discussion by expanding $\bW \bW_i$ as
\[
\begin{bmatrix}
\bW_1^{\rmt} \\
\bW_2^{\rmt} \\
\cdots \\
\bW_n^{\rmt}
\end{bmatrix}\bW_i =
\begin{bmatrix}
\bW_1^{\rmt}\bW_i \\
\bW_2^{\rmt}\bW_i \\
\cdots \\
\bW_n^{\rmt}\bW_i	
\end{bmatrix}.
\]
Then we obtain
\begin{align}
\Expc \Xi_4^2 =~&
\sum_{s = 1}^n \sum_{t=1}^n \
\Gamma_{ij} \Expc\Bracket{\bracket{\bW_s^{\rmt}\bW_i} \bracket{\bW_t^{\rmt}\bW_i}} =  \Gamma_{ii}\Expc \bracket{\bW_i^{\rmt}\bW_i}^2
+ \underbrace{\sum_{s \neq i} \sum_{t\neq i} \Gamma_{st} \Expc
\bracket{\bW_s^{\rmt}\bW_i \bW_t^{\rmt}\bW_i}}_{\sum_{s\neq i}
\Gamma_{ss}\Expc\bracket{\bW_s^{\rmt}\bW_i}^2} \notag \\
=~&\Gamma_{ii}\Expc\bracket{\sum_{j=1}^m W_{ij}^2}^2 +
\sum_{s\neq i} \Gamma_{ss} \cdot m =
m(m+1)\Expc \Gamma_{ii} + m\Expc \trace(\bGamma).
\label{eq:xifour_square_summary}
\end{align}
We can thus complete the proof by separately computing
$\Expc\trace(\bM)$ and $\Expc\Gamma_{ii}$.
First we compute $\Expc\Gamma_{ii}$, which proceeds as
\begin{align}
\Expc \Gamma_{ii} = ~&
\Expc\bracket{\bX_i^{\rmt}\bX_{\pi^{\natural}(i)}}^2
+ \Expc \bracket{\bX_i^{\rmt}\bX_j}^2 \notag \\
=~&
\Ind_{i = \pi^{\natural}(i)} p(p+2) + \Ind_{i \neq  \pi^{\natural}(i)} p
+ \Ind_{i = j}p(p+2) + \Ind_{i \neq j} p \notag \\
=~& p\bracket{p+2}\Bracket{\Ind_{i = \pi^{\natural}(i)} + \Ind_{i = j}}
+ p\Bracket{\Ind_{i \neq \pi^{\natural}(i)} + \Ind_{i \neq j}}.
\label{eq:xifour_square_diag}
\end{align}
Then we turn to the computation of $\Expc\trace(\bM)$, which proceeds as
\begin{align}
\Expc\trace(\bGamma) = ~&\Fnorm{\bX\bracket{\bX_{\pi^{\natural}(i)} - \bX_j}}^2 \notag \\
=~&
\Expc\norm{\bX_{\pi^{\natural}(i)}^{\rmt}\bracket{\bX_{\pi^{\natural}(i)} - \bX_j}}{2}^2
+ \Expc\norm{\bX_j^{\rmt}\bracket{\bX_{\pi^{\natural}(i)} - \bX_j}}{2}^2
+ \sum_{s\neq \pi^{\natural}(i), j} \Expc \norm{\bX_s^{\rmt}\bracket{\bX_{\pi^{\natural}(i)} - \bX_j}}{2}^2 \notag \\
=~& 2\Expc\Norm{\bX_{\pi^{\natural}(i)}}{2}^4 +
2 \Expc \bracket{\bX_{\pi^{\natural}(i)}^{\rmt}\bX_j}^2 + \
2\sum_{s\neq \pi^{\natural}(i), j}\Expc \norm{\bX_s}{2}^2 \notag \\
=~& 2p(p+3) + 2(n-2)p =  2p(n+p+1).
\label{eq:xifour_square_trace}
\end{align}
The proof is thus completed by combing
~\eqref{eq:xifour_square_summary}, ~\eqref{eq:xifour_square_trace},
and ~\eqref{eq:xifour_square_diag}.
\end{proof}

\begin{lemma}
\label{lemma:xionefour_expc}
We have
\[
\Expc \Xi_1 \Xi_4 = ~&
\frac{mp(n-h)(n+p-h)}{n}\Bracket{1 + o(1)}\trace(\bM),
\]	
where $\Xi_1$ and $\Xi_4$ are defined in
~\eqref{eq:xi_decomposition}.
\end{lemma}

\begin{proof}

We have
\[
\Expc \Xi_1 \Xi_4 =~&
\Expc{\underbrace{\bX_{\pi^{\natural}(i)}^{\rmt}\bM\bX^{\rmt}\bPi^{\natural \rmt} \bX
\bracket{\bX_{\pi^{\natural}(i)} - \bX_j}
\bracket{\bX_{\pi^{\natural}(i)} - \bX_j}^{\rmt}\bX^{\rmt}}_{\defequal \bv^{\rmt}}
\bW \bW_i }.
\]
First we conditional on $\bX$.
Expanding the product $\bW \bW_i$ as
\[
\Expc \begin{bmatrix}
\bW_1^{\rmt} \bW_i \\
\bW_2^{\rmt} \bW_i \\
\cdots \\
\bW_i^{\rmt}\bW_i \\
\cdots \\
\bW_n^{\rmt}\bW_i	
\end{bmatrix} =
\begin{bmatrix}
0 \\
0 \\
\cdots \\
m \\
\cdots \\
0	
\end{bmatrix},
\]
we can  compute $\Expc \Xi_1 \Xi_2$ w.r.t.
$\bW$ as
\[
\Expc\bracket{\Xi_1 \Xi_2} =
\Expc \bv^{\rmt}\bW \bW_i
= m \Expc \bv_i,
\]
where $\bv_i$ denotes the $i$-th entry of $\bv$ and can be written as
\[
\bv_i
=~&
\underbrace{
\bX_i^{\rmt}\bracket{\bX_{\pi^{\natural}(i)} - \bX_j}
\bracket{\bX_{\pi^{\natural}(i)} - \bX_j}^{\rmt}
\bSigma
\bM \bX_{\pi^{\natural}(i)}
}_{\Lambda_1}+
\underbrace{\bX_i^{\rmt}\bracket{\bX_{\pi^{\natural}(i)} - \bX_j}
\bracket{\bX_{\pi^{\natural}(i)} - \bX_j}^{\rmt}
\bDelta
\bM \bX_{\pi^{\natural}(i)}}_{\Lambda_2}.
\]
For $\Lambda_1$, we conclude
\[
\Expc\Lambda_1 =~&
\Expc \bX_i^{\rmt}
\bX_{\pi^{\natural}(i)}
\bX_{\pi^{\natural}(i)}^{\rmt}
\bSigma
\bM \bX_{\pi^{\natural}(i)}
+
\Expc \bX_i^{\rmt}
\bX_j \bX_j^{\rmt}
\bSigma
\bM \bX_{\pi^{\natural}(i)} \\
-~&
\Expc \bX_i^{\rmt}
\bX_{\pi^{\natural}(i)} \bX_j^{\rmt}
\bSigma
\bM \bX_{\pi^{\natural}(i)}
-
\Expc \bX_i^{\rmt}
\bX_j \bX_{\pi^{\natural}(i)}^{\rmt}
\bSigma
\bM \bX_{\pi^{\natural}(i)} \\
=~& \Ind_{i = \pi^{\natural}(i)}
(p+3)\Expc\trace(\bSigma \bM)
- \Ind_{i = j}(p+1)\Expc\trace(\bSigma \bM) \\
=~&
(n-h)\bracket{\Ind_{i = \pi^{\natural}(i)}
(p+3) - \Ind_{i = j}(p+1)}\trace(\bM) = \frac{p(n-h)^2}{n}\Bracket{1 + o(1)}
\trace(\bM).
\]
Then we turn to $\Expc \Lambda_2$ and obtain
\[
\Expc \Lambda_2 =~& \
\Expc \underbrace{\bX_i^{\rmt}
\bX_{\pi^{\natural}(i)}
\bX_{\pi^{\natural}(i)}^{\rmt}
\bDelta
\bM \bX_{\pi^{\natural}(i)}}_{\Lambda_{2, 1}} +
\Expc\underbrace{
\bX_i^{\rmt}
\bX_j \bX_j^{\rmt}
\bDelta
\bM \bX_{\pi^{\natural}(i)}
}_{\Lambda_{2, 2}} \\
-~& \Expc\underbrace{ \bX_i^{\rmt}
\bX_{\pi^{\natural}(i)}
\bX_j^{\rmt}
\bDelta
\bM \bX_{\pi^{\natural}(i)}}_{\Lambda_{2, 3}} -
\Expc \underbrace{\bX_i^{\rmt} \bX_j
\bX_{\pi^{\natural}(i)}^{\rmt}
\bDelta
\bM \bX_{\pi^{\natural}(i)}}_{\Lambda_{2, 4}}.
\]
We compute the value of $\Expc \Lambda_2$ under the four
different cases.

\noindent\textbf{Case $(s, s)$: $i = \pi^{\natural}(i)$ and $j = \pi^{\natural}(j)$.}
In this case, we have $\bDelta$ be
\[
\bDelta = \bDelta^{(s,s)}= \bX_i \bX_i^{\rmt} + \bX_j \bX_j^{\rmt}.
\]
We have
\[
\Expc \Lambda_{2, 1} = ~&
\Expc \bX_i^{\rmt}
\bX_{i}
\bX_{i}^{\rmt}
\bracket{ \bX_i \bX_i^{\rmt} + \bX_j \bX_j^{\rmt}}
\bM \bX_{i} =
(p+2)(p+5)\trace(\bM), \\
\Expc \Lambda_{2, 2} = ~&
\Expc \bX_i^{\rmt}
\bX_j \bX_j^{\rmt}
\bracket{\bX_i \bX_i^{\rmt} + \bX_j \bX_j^{\rmt}}
\bM \bX_{i} =
2(p+2)\trace(\bM), \\
\Expc \Lambda_{2, 3} = ~& \
\Expc \bX_i^{\rmt}
\bX_{i}
\bX_j^{\rmt}
\bracket{\bX_i \bX_i^{\rmt} + \bX_j \bX_j^{\rmt}}
\bM \bX_{i} = 0, \\
\Expc \Lambda_{2, 4} =~&
\Expc \bX_i^{\rmt} \bX_j
\bX_{i}^{\rmt}
\bracket{\bX_i \bX_i^{\rmt} + \bX_j \bX_j^{\rmt}}
\bM \bX_{i} = 0,
\]
which implies
\[
\Expc \Lambda_2 = (p+2)\bracket{p+7}\trace(\bM).
\]

\noindent\textbf{Case $(s, d)$: $i = \pi^{\natural}(i)$ and $j \neq \pi^{\natural}(j)$.}
First we write $\bDelta$ as
\[
\bDelta^{(s, d)} =
\bX_{i}\bX_{i}^{\rmt}
+ \bX_{j} \bX_{\pi^{\natural}(j)}^{\rmt}
+ \bX_{\pi^{\natural -1}(j)} \bX_{j}^{\rmt}.
\]
Then we conclude
\[
\Expc \Lambda_{2, 1} =~& \Expc  \bX_i^{\rmt}
\bX_{i}
\bX_{i}^{\rmt}
\bracket{
\bX_{i}\bX_{i}^{\rmt}
+ \bX_{j} \bX_{\pi^{\natural}(j)}^{\rmt}
+ \bX_{\pi^{\natural -1}(j)} \bX_{j}^{\rmt}
}
\bM \bX_{i} = (p+2)(p+4)\trace(\bM), \\
\Expc \Lambda_{2, 2} =~& \Expc \bX_i^{\rmt}
\bX_j \bX_j^{\rmt}
\bracket{
\bX_{i}\bX_{i}^{\rmt}
+ \bX_{j} \bX_{\pi^{\natural}(j)}^{\rmt}
+ \bX_{\pi^{\natural -1}(j)} \bX_{j}^{\rmt}
}
\bM \bX_{i} =
(p+2)\trace(\bM), \\
\Expc \Lambda_{2, 3} =~&
\Expc \bX_i^{\rmt}
\bX_{i}
\bX_j^{\rmt}
\bracket{
\bX_{i}\bX_{i}^{\rmt}
+ \bX_{j} \bX_{\pi^{\natural}(j)}^{\rmt}
+ \bX_{\pi^{\natural -1}(j)} \bX_{j}^{\rmt}
}
\bM \bX_{i} = 0 , \\
\Expc \Lambda_{2, 4} = ~& \Expc \bX_i^{\rmt} \bX_j
\bX_{i}^{\rmt}
\bracket{
\bX_{i}\bX_{i}^{\rmt}
+ \bX_{j} \bX_{\pi^{\natural}(j)}^{\rmt}
+ \bX_{\pi^{\natural -1}(j)} \bX_{j}^{\rmt}
}
\bM \bX_{i} = 0,
\]
which suggests that
\[
\Expc \Lambda_2 = (p+2)(p+5)\trace(\bM).
\]

\noindent\textbf{Case $(d, s)$: $i\neq \pi^{\natural}(i)$ and $j = \pi^{\natural}(j)$.}
In this case, $\bDelta$ reduces to
\[
\bDelta^{(d, s)} =
\bX_{i}\bX_{\pi^{\natural}(i)}^{\rmt}
+ \bX_{\pi^{\natural}(i)}\bX_{\pi^{\natural 2}(i)}^{\rmt}
+ \bX_{j}\bX_{j}^{\rmt}.
\]
We have
\[
\Expc \Lambda_{2, 1}=~&
\underbrace{\Expc \bX_i^{\rmt}
\bX_{\pi^{\natural}(i)}
\bX_{\pi^{\natural}(i)}^{\rmt}
\bX_{i}\bX_{\pi^{\natural}(i)}^{\rmt}
\bM \bX_{\pi^{\natural}(i)}}_{\Expc \Norm{\bX_i}{2}^2 \bX_i^{\rmt}\bM\bX_i}
+
\underbrace{\Expc \bX_i^{\rmt}
\bX_{\pi^{\natural}(i)}
\bX_{\pi^{\natural}(i)}^{\rmt}
\bX_{\pi^{\natural}(i)}\bX_{\pi^{\natural 2}(i)}^{\rmt}
\bM \bX_{\pi^{\natural}(i)}}_{
\Ind_{i = \pi^{\natural 2}(i)}\Expc \Norm{\bX_i}{2}^2 \bX_i^{\rmt}\bM\bX_i } \\
+~&
\underbrace{\Expc \bX_i^{\rmt}
\bX_{\pi^{\natural}(i)}
\bX_{\pi^{\natural}(i)}^{\rmt}
\bX_{j}\bX_{j}^{\rmt}
\bM \bX_{\pi^{\natural}(i)}}_{0} =
(p+2)\Bracket{1 + \Ind_{i = \pi^{\natural 2}(i)}}\trace(\bM),\\
\Expc \Lambda_{2, 2} =~&
\underbrace{\Expc
\bX_i^{\rmt}
\bX_j \bX_j^{\rmt}
\bX_{i}\bX_{\pi^{\natural}(i)}^{\rmt}
\bM \bX_{\pi^{\natural}(i)}}_{p\trace(\bM)} +
\underbrace{\Expc
\bX_i^{\rmt}
\bX_j \bX_j^{\rmt}
\bX_{\pi^{\natural}(i)}\bX_{\pi^{\natural 2}(i)}^{\rmt}
\bM \bX_{\pi^{\natural}(i)}}_{\Ind_{i = \pi^{\natural 2}(i)} \trace(\bM)} \\
+~&
\underbrace{\Expc
\bX_i^{\rmt}
\bX_j \bX_j^{\rmt}
\bX_{j}\bX_{j}^{\rmt}
\bM \bX_{\pi^{\natural}(i)}}_{0}
= \bracket{p + \Ind_{i = \pi^{\natural 2}(i)}}\trace(\bM),  \\
\Expc \Lambda_{2, 3} =~& 0, \\
\Expc \Lambda_{2, 4} =~& 0,
\]
which suggests
\[
\Expc \Lambda_2 =
2(p+1)\trace(\bM) + \Ind_{i = \pi^{\natural 2}(i)}\bracket{p+3}\trace(\bM).
\]

\noindent\textbf{Case $(d, d)$: $i \neq \pi^{\natural}(i)$ and $j\neq \pi^{\natural}(j)$.}
In this case, $\bDelta$ is written as
\[
\bDelta^{(d,d)} =
\bX_{i}\bX_{\pi^{\natural}(i)}^{\rmt}
+ \bX_{\pi^{\natural}(i)}\bX_{\pi^{\natural 2}(i)}^{\rmt}
+ \bX_{j}\bX_{\pi^{\natural}(j)}^{\rmt}
+ \bX_{\pi^{\natural -1}(j)}\bX_{j}^{\rmt}.
\]
We have
\[
\Expc \Lambda_{2, 1}=~&
\underbrace{\Expc \bX_i^{\rmt}
\bX_{\pi^{\natural}(i)}
\bX_{\pi^{\natural}(i)}^{\rmt}
\bX_{i}\bX_{\pi^{\natural}(i)}^{\rmt}
\bM \bX_{\pi^{\natural}(i)}}_{\Expc \norm{\bX_i}{2}^2 \bX_i^{\rmt}\bM \bX_i}
+
\underbrace{\Expc \bX_i^{\rmt}
\bX_{\pi^{\natural}(i)}
\bX_{\pi^{\natural}(i)}^{\rmt}
\bX_{\pi^{\natural}(i)}\bX_{\pi^{\natural 2}(i)}^{\rmt}
\bM \bX_{\pi^{\natural}(i)}}_{
\Ind_{i = \pi^{\natural 2}(i)}
\Expc \norm{\bX_i}{2}^2 \bX_i^{\rmt}\bM \bX_i
} \\
+~&
\underbrace{\Expc \bX_i^{\rmt}
\bX_{\pi^{\natural}(i)}
\bX_{\pi^{\natural}(i)}^{\rmt}
\bX_{j}\bX_{\pi^{\natural}(j)}^{\rmt}
\bM \bX_{\pi^{\natural}(i)}}_{
\Ind_{i = j}
\Expc \norm{\bX_i}{2}^2 \bX_i^{\rmt}\bM \bX_i
} +
\underbrace{\Expc \bX_i^{\rmt}
\bX_{\pi^{\natural}(i)}
\bX_{\pi^{\natural}(i)}^{\rmt}
\bX_{\pi^{\natural -1}(j)}\bX_{j}^{\rmt}
\bM \bX_{\pi^{\natural}(i)}}_{
\Ind_{i = j} \Ind_{i = \pi^{\natural 2}(i)}
\Expc \norm{\bX_i}{2}^2 \bX_i^{\rmt}\bM \bX_i
} ;
\\
\Expc \Lambda_{2, 2} =~&
\underbrace{\Expc
\bX_i^{\rmt}
\bX_j \bX_j^{\rmt}
\bX_{i}\bX_{\pi^{\natural}(i)}^{\rmt}
\bM \bX_{\pi^{\natural}(i)}}_{
\Ind_{i = j }p(p+2)\trace(\bM) +
\Ind_{i\neq j}p\trace(\bM)
}
+
\underbrace{\Expc
\bX_i^{\rmt}
\bX_j \bX_j^{\rmt}
\bX_{\pi^{\natural}(i)}\bX_{\pi^{\natural 2}(i)}^{\rmt}
\bM \bX_{\pi^{\natural}(i)}}_{
\Ind_{i = \pi^{\natural 2}(i)}
\Bracket{
\Ind_{i = j} \Expc \norm{\bX_i}{2}^2 \bX_i^{\rmt}\bM \bX_i
+ \Ind_{i\neq j} \trace(\bM)
}
} \\
+~&
\underbrace{
\Expc
\bX_i^{\rmt}
\bX_j \bX_j^{\rmt}
\bX_{j}\bX_{\pi^{\natural}(j)}^{\rmt}
\bM \bX_{\pi^{\natural}(i)}
}_{
\Ind_{i = j} p(p+2)\trace(\bM)
} +
\underbrace{\Expc
\bX_i^{\rmt}
\bX_j \bX_j^{\rmt}
\bX_{\pi^{\natural -1}(j)}\bX_{j}^{\rmt}
\bM \bX_{\pi^{\natural}(i)}}_{
\Ind_{i = j}\Ind_{j = \pi^{\natural 2}(i)}
\Expc \norm{\bX}{2}^2 \bX^{\rmt}\bM \bX
}; \\
\Expc \Lambda_{2, 3} =~&
\underbrace{\Expc \bX_i^{\rmt}
\bX_{\pi^{\natural}(i)}
\bX_j^{\rmt}
\bX_{i}\bX_{\pi^{\natural}(i)}^{\rmt}
\bM \bX_{\pi^{\natural}(i)} }_{0}
+
\underbrace{
\Expc \bX_i^{\rmt}
\bX_{\pi^{\natural}(i)}
\bX_j^{\rmt}
\bX_{\pi^{\natural}(i)}\bX_{\pi^{\natural 2}(i)}^{\rmt}
\bM \bX_{\pi^{\natural}(i)}}_{0} \\
+~&
\underbrace{\Expc \bX_i^{\rmt}
\bX_{\pi^{\natural}(i)}
\bX_j^{\rmt}
\bX_{j}\bX_{\pi^{\natural}(j)}^{\rmt}
\bM \bX_{\pi^{\natural}(i)}}_{
\Ind_{i = \pi^{\natural}(j)}
p\trace(\bM)
}
+
\underbrace{\Expc \bX_i^{\rmt}
\bX_{\pi^{\natural}(i)}
\bX_j^{\rmt}
\bX_{\pi^{\natural -1}(j)}\bX_{j}^{\rmt}
\bM \bX_{\pi^{\natural}(i)}}_{0};
\\
\Expc \Lambda_{2, 4} =~&
\underbrace{\Expc \bX_i^{\rmt} \bX_j
\bX_{\pi^{\natural}(i)}^{\rmt}
\bX_{i}\bX_{\pi^{\natural}(i)}^{\rmt}
\bM \bX_{\pi^{\natural}(i)}}_{0} +
\underbrace{\Expc \bX_i^{\rmt} \bX_j
\bX_{\pi^{\natural}(i)}^{\rmt}
\bX_{\pi^{\natural}(i)}\bX_{\pi^{\natural 2}(i)}^{\rmt}
\bM \bX_{\pi^{\natural}(i)}}_{0} \\
+~&
\underbrace{\Expc \bX_i^{\rmt} \bX_j
\bX_{\pi^{\natural}(i)}^{\rmt}
\bX_{j}\bX_{\pi^{\natural}(j)}^{\rmt}
\bM \bX_{\pi^{\natural}(i)}}_{
\Ind_{i = \pi^{\natural}(j)}
\trace(\bM)
} +
\underbrace{\Expc \bX_i^{\rmt} \bX_j
\bX_{\pi^{\natural}(i)}^{\rmt}
\bX_{\pi^{\natural -1}(j)}\bX_{j}^{\rmt}
\bM \bX_{\pi^{\natural}(i)}}_{0}.
\]
Hence we conclude
\[
\Expc \Lambda_2 =~&
2(p+1)\trace(\bM) +
\Ind_{i = j}
2\bracket{p+1}^2
\trace(\bM) +
\Ind_{i = \pi^{\natural}(j)}
(p+1)\trace(\bM)  \\
+ ~&
\Ind_{i = \pi^{\natural 2}(i)}
(p+3)\trace(\bM)
+
\Ind_{i = j} \Ind_{i = \pi^{\natural 2}(i)}
(3p+5)\trace(\bM).
\]
\end{proof}


\begin{lemma}
\label{lemma:xitwothree_expc}
We have
\[
\Expc \Xi_2 \Xi_3 = ~& \
\frac{p(n-h)(n+p-h)}{n}\trace(\bM)\Bracket{1 + o(1)},
\]	
where $\Xi_2$ and $\Xi_3$ are defined in ~\eqref{eq:xi_decomposition}.
\end{lemma}

\begin{proof}
To start with, we write the expectation as $\Expc \Xi_2 \Xi_3$
\[
\Expc \Xi_2 \Xi_3 = ~&
\Expc
\underbrace{\bX_{\pi^{\natural}(i)}^{\rmt}\bB^{\natural}}_{\bu^{\rmt}} \bW^{\rmt}\underbrace{\bX\bracket{\bX_{\pi^{\natural}(i)} - \bX_j}}_{\bp \in \RR^{n\times 1}}
\bW_i^{\rmt} \underbrace{\bB^{\natural\rmt}\bX^{\rmt}\bPi^{\natural \rmt} \bX \bracket{\bX_{\pi^{\natural}(i)} - \bX_j}}_{\bv} \\
=~& \Expc \bu^{\rmt}\bW^{\rmt}\bp \bW_i^{\rmt}\bv
= \Expc \la \bW_i, \bu^{\rmt}\bW^{\rmt}\bp \bv\ra.
\]
Exploiting the independence among $\bX$ and $\bW$, we condition on $\bX$
and have
\[
\Expc_{\bW} \la \bW_i, \bu^{\rmt}\bW^{\rmt}\bp \bv\ra =
\Expc_{\bW} \trace\bracket{\nabla_{\bW_i}\bu^{\rmt}\bW^{\rmt}\bp \bv }.
\]
Note that only the diagonal entries of the Hessian
matrix $\nabla_{\bW_i}\bu^{\rmt}\bW^{\rmt}\bp \bv$ matters.
For an arbitrary index $s$, we can compute
the gradient of the $s$-th entry of $\bu^{\rmt}\bW^{\rmt}\bp \bv$ w.r.t.
$\bW_{i, s}$  as
\[
\dfrac{d}{dW_{i,s}}\bracket{\bu^{\rmt}\bW^{\rmt}\bp \bv_s} = \bv_{s} \dfrac{d}{dW_{i,s}}\bracket{\bu^{\rmt} \bW^{\rmt}\bp} =
\bv_s \sum_{t=1}^n \dfrac{d}{dW_{i, s}}\bracket{\bp_t \bW_t^{\rmt}\bu}
= \bv_s \dfrac{d}{dW_{i,s}} \bp_i \bW_i^{\rmt}\bu
= \bp_i \bv_s \bu_s.
\]
Invoking the definitions of $\bp, \bv$ and $\bu$, we have
\[
& \Expc_{\bW, \bX} \la \bW_i, \bu^{\rmt}\bW^{\rmt}\bp \bv\ra =
\Expc_{\bX}  \bracket{\bX_{\pi^{\natural}(i)} - \bX_j}^{\rmt}\bX_{i}
\sum_{s=1}^m
\Bracket{
\bX_{\pi^{\natural}(i)}^{\rmt}
\bracket{\bB^{\natural \rmt}}_{s}
\bracket{\bB^{\natural \rmt}}^{\rmt}_{s}
\bX^{\rmt}\bPi^{\natural \rmt}\bX\bracket{\bX_{\pi^{\natural}(i)} - \bX_j}
} \\
\stackrel{\cirone}{=}~&
\Expc \underbrace{
\Bracket{
\bracket{\bX_{\pi^{\natural}(i)} - \bX_j}^{\rmt}\bX_i
\bX_{\pi^{\natural}(i)}^{\rmt} \bM
\bSigma^{\rmt}
\bracket{\bX_{\pi^{\natural}(i)} - \bX_j}}}_{\Lambda_1} + \
\Expc\underbrace{
\Bracket{
\bracket{\bX_{\pi^{\natural}(i)} - \bX_j}^{\rmt}\bX_i
\bX_{\pi^{\natural}(i)}^{\rmt} \bM
\bDelta^{\rmt}
\bracket{\bX_{\pi^{\natural}(i)} - \bX_j}}}_{\Lambda_2},
\]
where in $\cirone$ we use the
relation $\sum_{s=1}^m \bracket{\bB^{\natural \rmt}}_s \bracket{\bB^{\natural \rmt}}_s^{\rmt} = \bB^{\natural }\bB^{\natural \rmt} = \bM$.

For the first term $\Lambda_1$, we obtain
\[
\Expc \Lambda_1 =~&
\underbrace{\Expc
\bracket{
\bX_{\pi^{\natural}(i)}^{\rmt}\bX_i
\bX_{\pi^{\natural}(i)}^{\rmt} \bM
\bSigma^{\rmt}
\bX_{\pi^{\natural}(i)}}}_{\Ind_{i = \pi^{\natural}(i)}
\Expc \Fnorm{\bX_i}^2 \bX_i^{\rmt}\bM\bSigma^{\rmt}\bX_i} +
\underbrace{\Expc
\bracket{
 \bX_j^{\rmt}\bX_i
\bX_{\pi^{\natural}(i)}^{\rmt} \bM
\bSigma^{\rmt}
\bX_j}}_{\Ind_{i = \pi^{\natural}(i)}
\Expc \bX_i^{\rmt}\bM\bSigma^{\rmt}\bX_i} \\
-~&
\underbrace{\Expc
\bracket{
\bX_{\pi^{\natural}(i)}^{\rmt}\bX_i
\bX_{\pi^{\natural}(i)}^{\rmt} \bM
\bSigma^{\rmt} \bX_j}}_{
\Ind_{i = j}\Ind_{i\neq \pi^{\natural}(i)} \Expc \bX_{\pi^{\natural}(i)}^{\rmt}\bM \bSigma^{\rmt}
\bX_{\pi^{\natural}(i)}
}
- \underbrace{\Expc
\bracket{
 \bX_j^{\rmt}\bX_i
\bX_{\pi^{\natural}(i)}^{\rmt} \bM
\bSigma^{\rmt}
\bX_{\pi^{\natural}(i)}}}_{
p \Ind_{i= j}\Ind_{i \neq \pi^{\natural}(i)}
\Expc \bX_{\pi^{\natural}(i)}^{\rmt}\bM \bSigma^{\rmt}
\bX_{\pi^{\natural}(i)}
} \\
=~& (n-h)\Bracket{
\Ind_{i = \pi^{\natural}(i)} (p+3)- (p+1) \Ind_{i= j}\Ind_{i \neq \pi^{\natural}(i)}
}\trace(\bM).
\]
Then we consider the second term $\Lambda_2$, which can be decomposed
further into four sub-terms reading as
\[
\Expc \Lambda_2 =~&
\Expc \underbrace{\bracket{
\bX_{\pi^{\natural}(i)} ^{\rmt}\bX_i
\bX_{\pi^{\natural}(i)}^{\rmt} \bM
\bDelta^{\rmt}
\bX_{\pi^{\natural}(i)}}}_{\Lambda_{2, 1}}
+ \Expc \underbrace{\bracket{
\bX_j^{\rmt}\bX_i
\bX_{\pi^{\natural}(i)}^{\rmt} \bM
\bDelta^{\rmt} \bX_j}}_{\Lambda_{2, 2}} \\
-~& \Expc \underbrace{\bracket{
\bX_{\pi^{\natural}(i)}^{\rmt}\bX_i
\bX_{\pi^{\natural}(i)}^{\rmt} \bM
\bDelta^{\rmt}
\bX_j}}_{\Lambda_{2, 3}} - \Expc
\underbrace{\bracket{
 \bX_j^{\rmt}\bX_i
\bX_{\pi^{\natural}(i)}^{\rmt} \bM
\bDelta^{\rmt}
\bX_{\pi^{\natural}(i)} }}_{\Lambda_{2, 4}}.
\]

\noindent\textbf{Case $(s, s)$: $i = \pi^{\natural}(i)$ and $j = \pi^{\natural}(j)$.}
In this case, we have $\bDelta$ be
\[
\bDelta^{(s,s)}= \bX_i \bX_i^{\rmt} + \bX_j \bX_j^{\rmt}.
\]
Hence we conclude
\[
\Expc \Lambda_{2, 1} =~&  \Expc \Bracket{
\bX_{i} ^{\rmt}\bX_i
\bX_{i}^{\rmt} \bM
\bracket{\bX_i \bX_i^{\rmt} + \bX_j \bX_j^{\rmt}}
\bX_{i}} = \Expc \Fnorm{\bX_i}^4 \bX_i^{\rmt}\bM \bX_i +
\Expc \Fnorm{\bX_i}^2 \bX_i^{\rmt}\bM \bX_i \\
=~& (p+2)(p+4)\trace(\bM) + (p+2)\trace(\bM), \\
\Expc \Lambda_{2, 2} =~&
\Expc\Bracket{
\bX_j^{\rmt}\bX_i
\bX_{i}^{\rmt}
\bM
\bracket{\bX_i \bX_i^{\rmt} + \bX_j \bX_j^{\rmt}} \bX_j} = \Expc \Fnorm{\bX_i}^2 \bX_i^{\rmt}\bM \bX_i
+ \Expc \Fnorm{\bX_j}^2 \bX_j^{\rmt}\bM \bX_j \\
=~& 2\Expc \Fnorm{\bX_i}^2 \bX_i^{\rmt}\bM \bX_i = 2(p+2)\trace(\bM), \\
\Expc \Lambda_{2, 3} =~&
\Expc \Bracket{
\bX_{i}^{\rmt}\bX_i
\bX_{i}^{\rmt} \bM
\bracket{\bX_i \bX_i^{\rmt} + \bX_j \bX_j^{\rmt}}
\bX_j} = 0, \\
\Expc \Lambda_{2, 4} =~&
\Expc \Bracket{
\bX_j^{\rmt}\bX_i
\bX_{i}^{\rmt} \bM
\bracket{\bX_i \bX_i^{\rmt} + \bX_j \bX_j^{\rmt}}
\bX_{i} } = 0,
\]
which suggests
$\Expc \Lambda_2 = \bracket{p+2}\bracket{p+7}\trace(\bM)$.

\noindent\textbf{Case $(s, d)$: $i = \pi^{\natural}(i)$ and $j \neq \pi^{\natural}(j)$.}
First we write $\bDelta$ as
\[
\bDelta^{(s, d)\rmt} =
\bX_{i}\bX_{i}^{\rmt}
+ \bX_{\pi^{\natural}(j)} \bX_{j}^{\rmt}
+ \bX_{j}\bX_{\pi^{\natural -1}(j)}^{\rmt}.
\]
Then we conclude
\[
\Expc \Lambda_{2, 1} =~&
\underbrace{\Expc \bracket{
\bX_{i}^{\rmt}\bX_i
\bX_{i}^{\rmt} \bM
\bX_{i}\bX_{i}^{\rmt}
\bX_{i}}}_{\Expc \Fnorm{\bX_i}^4 \bX_i^{\rmt}\bM \bX_i }
+ \underbrace{\Expc \bracket{
\bX_{i}^{\rmt}\bX_i
\bX_{i}^{\rmt} \bM
\bX_{\pi^{\natural}(j)} \bX_{j}^{\rmt}
\bX_{i}}}_{0} \\
+~& \underbrace{\Expc \bracket{
\bX_{i}^{\rmt}\bX_i
\bX_{i}^{\rmt} \bM
\bX_{j}\bX_{\pi^{\natural -1}(j)}^{\rmt}
\bX_{i}}}_{0} =
\Expc \Fnorm{\bX_i}^4 \bX_i^{\rmt}\bM \bX_i
= (p+2)(p+4)\trace(\bM); \\
\Expc \Lambda_{2, 2} = ~&
\underbrace{\Expc \bracket{
\bX_j^{\rmt}\bX_i
\bX_{i}^{\rmt} \bM
\bX_{i}\bX_{i}^{\rmt}
\bX_j}}_{\Expc \Fnorm{\bX_i}^2 \bX_i^{\rmt}\bM \bX_i}
+ \underbrace{\Expc \bracket{
\bX_j^{\rmt}\bX_i
\bX_{i}^{\rmt} \bM
\bX_{\pi^{\natural}(j)} \bX_{j}^{\rmt}
\bX_j}}_{0}  \\
+~& \underbrace{\Expc \bracket{
\bX_j^{\rmt}\bX_i
\bX_{i}^{\rmt} \bM
\bX_{j}\bX_{\pi^{\natural -1}(j)}^{\rmt}
\bX_j}}_{0} = \Expc \Fnorm{\bX_i}^2 \bX_i^{\rmt}\bM \bX_i = (p+2)\trace(\bM), \\
\Expc \Lambda_{2, 3} =~&
\Expc \Bracket{
 \bX_i^{\rmt}\bX_i
\bX_{i}^{\rmt} \bM
\bracket{
\bX_{i}\bX_{i}^{\rmt}
+ \bX_{\pi^{\natural}(j)} \bX_{j}^{\rmt}
+ \bX_{j}\bX_{\pi^{\natural -1}(j)}^{\rmt}
}
\bX_{j}} = 0, \\
\Expc \Lambda_{2, 4} =~&
\Expc \Bracket{
 \bX_j^{\rmt}\bX_i
\bX_{i}^{\rmt} \bM
\bracket{
\bX_{i}\bX_{i}^{\rmt}
+ \bX_{\pi^{\natural}(j)} \bX_{j}^{\rmt}
+ \bX_{j}\bX_{\pi^{\natural -1}(j)}^{\rmt}
}
\bX_{i}} = 0,
\]
which suggests $\Expc \Lambda_2 = (p+2)(p+5)\trace(\bM)$.

\vspace{0.2in}

\noindent\textbf{Case $(d, s)$: $i\neq \pi^{\natural}(i)$ and $j = \pi^{\natural}(j)$.}
In this case, $\bDelta$ reduces to
\[
\bDelta^{(d, s) \rmt} =
\bX_{\pi^{\natural}(i)}\bX_{i}^{\rmt}
+ \bX_{\pi^{\natural 2}(i)}\bX_{\pi^{\natural}(i)}^{\rmt}
+ \bX_{j}\bX_{j}^{\rmt}.
\]
Then we obtain
\[
\Expc \Lambda_{2, 1} = ~&
\underbrace{\Expc \bracket{
\bX_{\pi^{\natural}(i)} ^{\rmt}\bX_i
\bX_{\pi^{\natural}(i)}^{\rmt} \bM
\bX_{\pi^{\natural}(i)}\bX_{i}^{\rmt}
\bX_{\pi^{\natural}(i)}}}_{
\Expc \Fnorm{\bX_i}^2 \bX_i^{\rmt}\bM \bX_i
}  + \underbrace{\Expc \bracket{
\bX_{\pi^{\natural}(i)} ^{\rmt}\bX_i
\bX_{\pi^{\natural}(i)}^{\rmt} \bM
\bX_{\pi^{\natural 2}(i)}\bX_{\pi^{\natural}(i)}^{\rmt}
\bX_{\pi^{\natural}(i)}}}_{
\Ind_{i = \pi^{\natural 2}(i)}
\Expc \Fnorm{\bX_i}^2 \bX_i^{\rmt}\bM \bX_i
} \\
+~& \underbrace{\Expc \bracket{
\bX_{\pi^{\natural}(i)} ^{\rmt}\bX_i
\bX_{\pi^{\natural}(i)}^{\rmt} \bM
\bX_{j}\bX_{j}^{\rmt}
\bX_{\pi^{\natural}(i)}}}_{0}
= \bracket{1 + \Ind_{i = \pi^{\natural 2}(i)} }(p+2)\trace(\bM),   \\
\Expc \Lambda_{2, 2} =~&
\underbrace{
\Expc \bracket{
\bX_j^{\rmt}\bX_i
\bX_{\pi^{\natural}(i)}^{\rmt} \bM
\bX_{\pi^{\natural}(i)}\bX_{i}^{\rmt}
\bX_j}}_{
\Ind_{i = j}\Expc \Fnorm{\bX_i}^4 \trace(\bM) +
\Ind_{i\neq j}p\trace(\bM)
} + \underbrace{\Expc \bracket{
\bX_j^{\rmt}\bX_i
\bX_{\pi^{\natural}(i)}^{\rmt} \bM
\bX_{\pi^{\natural 2}(i)}\bX_{\pi^{\natural}(i)}^{\rmt}
\bX_j}}_{
\Ind_{i = \pi^{\natural 2}(i)}
\Bracket{
\Ind_{i = j} \Expc \Fnorm{\bX_i}^2 \bX_i^{\rmt}\bM \bX_i +
\Ind_{i \neq j}\trace(\bM)
}
} \\
+~& \underbrace{\Expc \bracket{
\bX_j^{\rmt}\bX_i
\bX_{\pi^{\natural}(i)}^{\rmt} \bM
\bX_{j}\bX_{j}^{\rmt}
\bX_j}}_{0} =
\bracket{p +\Ind_{i = \pi^{\natural 2}(i)} }\trace(\bM),
\\
\Expc \Lambda_{2, 3} =~&
\Expc \Bracket{
\bX_{\pi^{\natural}(i)}^{\rmt}\bX_i
\bX_{\pi^{\natural}(i)}^{\rmt} \bM
\bracket{
\bX_{\pi^{\natural}(i)}\bX_{i}^{\rmt}
+ \bX_{\pi^{\natural 2}(i)}\bX_{\pi^{\natural}(i)}^{\rmt}
+ \bX_{j}\bX_{j}^{\rmt}
}
\bX_j} = 0, \\
\Expc \Lambda_{2, 4} =~&
\Expc
\Bracket{
 \bX_j^{\rmt}\bX_i
\bX_{\pi^{\natural}(i)}^{\rmt} \bM
\bracket{
\bX_{\pi^{\natural}(i)}\bX_{i}^{\rmt}
+ \bX_{\pi^{\natural 2}(i)}\bX_{\pi^{\natural}(i)}^{\rmt}
+ \bX_{j}\bX_{j}^{\rmt}
}
\bX_{\pi^{\natural}(i)}
}= 0,
\]
which suggests
\[
\Expc \Lambda_2 =
2(p+1) \trace(\bM)+
(p+3)\Ind_{i = \pi^{\natural 2}(i)} \trace(\bM).
\]

\newpage

\noindent\textbf{Case $(d, d)$: $i \neq \pi^{\natural}(i)$ and $j\neq \pi^{\natural}(j)$.}
In this case, $\bDelta$ is written as
\[
\bDelta^{(d,d)\rmt} =
\bX_{\pi^{\natural}(i)}\bX_{i}^{\rmt}
+ \bX_{\pi^{\natural 2}(i)}\bX_{\pi^{\natural}(i)}^{\rmt}
+ \bX_{\pi^{\natural}(j)}\bX_{j}^{\rmt}
+ \bX_{j}\bX_{\pi^{\natural -1}(j)}^{\rmt}.
\]
Then we have
\[
\Expc \Lambda_{2, 1} = ~&
\underbrace{\Expc \bracket{
\bX_{\pi^{\natural}(i)} ^{\rmt}\bX_i
\bX_{\pi^{\natural}(i)}^{\rmt} \bM
\bX_{\pi^{\natural}(i)}\bX_{i}^{\rmt}
\bX_{\pi^{\natural}(i)}}}_{
\Expc \Fnorm{\bX_i}^2 \bX_i^{\rmt}\bM \bX_i}
+ \underbrace{\Expc \bracket{
\bX_{\pi^{\natural}(i)} ^{\rmt}\bX_i
\bX_{\pi^{\natural}(i)}^{\rmt} \bM
\bX_{\pi^{\natural 2}(i)}\bX_{\pi^{\natural}(i)}^{\rmt}
\bX_{\pi^{\natural}(i)}}}_{
\Ind_{i = \pi^{\natural 2}(i)}
\Expc \Fnorm{\bX_i}^2 \bX_i^{\rmt} \bM \bX_i
} \\
+~& \underbrace{\Expc \bracket{
\bX_{\pi^{\natural}(i)} ^{\rmt}\bX_i
\bX_{\pi^{\natural}(i)}^{\rmt} \bM
 \bX_{\pi^{\natural}(j)}\bX_{j}^{\rmt}
\bX_{\pi^{\natural}(i)}}}_{
\Ind_{i = j}
\Expc \Fnorm{\bX_i}^2 \bX_i^{\rmt} \bM \bX_i
}
+ \underbrace{\Expc \bracket{
\bX_{\pi^{\natural}(i)} ^{\rmt}\bX_i
\bX_{\pi^{\natural}(i)}^{\rmt} \bM
\bX_{j}\bX_{\pi^{\natural -1}(j)}^{\rmt}
\bX_{\pi^{\natural}(i)}}}_{
\Ind_{i = j}\Ind_{i = \pi^{\natural 2}(i)}
\Expc  \Fnorm{\bX_i}^2 \bX_i^{\rmt} \bM \bX_i
} \\
=~& (p+2)\Bracket{1 +
\Ind_{i = \pi^{\natural 2}(i)} + \Ind_{i = j}
+ \Ind_{i = \pi^{\natural 2}(i)} \Ind_{i = j}
}\trace(\bM),\\
\Expc \Lambda_{2, 2} = ~&
\underbrace{\Expc\bracket{
\bX_j^{\rmt}\bX_i
\bX_{\pi^{\natural}(i)}^{\rmt} \bM
\bX_{\pi^{\natural}(i)}\bX_{i}^{\rmt}
\bX_j}}_{
\Ind_{i = j}\Expc \Fnorm{\bX_i}^4\trace(\bM)
+ p\Ind_{i\neq j}\trace(\bM)
} +
\underbrace{\Expc\bracket{
\bX_j^{\rmt}\bX_i
\bX_{\pi^{\natural}(i)}^{\rmt} \bM
\bX_{\pi^{\natural 2}(i)}\bX_{\pi^{\natural}(i)}^{\rmt}
\bX_j}}_{
\Ind_{i = \pi^{\natural 2}(i)}
\Bracket{
\Ind_{i = j} (p+2) \trace(\bM) +
\Ind_{i \neq j }\trace(\bM)
}
} \\
+~& \underbrace{\Expc\bracket{
\bX_j^{\rmt}\bX_i
\bX_{\pi^{\natural}(i)}^{\rmt} \bM
\bX_{\pi^{\natural}(j)}\bX_{j}^{\rmt}
\bX_j}}_{
\Ind_{i = j}\Expc \Fnorm{\bX_i}^4\trace(\bM)
} +\underbrace{\Expc\bracket{
\bX_j^{\rmt}\bX_i
\bX_{\pi^{\natural}(i)}^{\rmt} \bM
\bX_{j}\bX_{\pi^{\natural -1}(j)}^{\rmt}
\bX_j}}_{
\Ind_{i = j}\Ind_{i = \pi^{\natural 2}(i)}
\Expc \Fnorm{\bX_i}^2 \bX_i^{\rmt}\bM \bX_i
} \\
=~&
2 \Ind_{i = j}p(p+2)\trace(\bM)
+ p\Ind_{i\neq j}\trace(\bM)
+
\Ind_{i = \pi^{\natural 2}(i)}
\Bracket{
\Ind_{i = j} 2(p+2) \trace(\bM) +
\Ind_{i \neq j }\trace(\bM)
},
\\
\Expc \Lambda_{2, 3} =~& \underbrace{\Expc \bracket{
\bX_{\pi^{\natural}(i)}^{\rmt}\bX_i
\bX_{\pi^{\natural}(i)}^{\rmt} \bM
\bX_{\pi^{\natural}(i)}\bX_{i}^{\rmt}
\bX_j}}_{0} +
\underbrace{\Expc \bracket{
\bX_{\pi^{\natural}(i)}^{\rmt}\bX_i
\bX_{\pi^{\natural}(i)}^{\rmt} \bM
\bX_{\pi^{\natural 2}(i)}\bX_{\pi^{\natural}(i)}^{\rmt}
\bX_j}}_{0} \\
+~& \underbrace{\Expc \bracket{
\bX_{\pi^{\natural}(i)}^{\rmt}\bX_i
\bX_{\pi^{\natural}(i)}^{\rmt} \bM
\bX_{\pi^{\natural}(j)}\bX_{j}^{\rmt}
\bX_j}}_{
p\Ind_{i = \pi^{\natural}(j)}
\trace(\bM)} +
\underbrace{\Expc \bracket{
\bX_{\pi^{\natural}(i)}^{\rmt}\bX_i
\bX_{\pi^{\natural}(i)}^{\rmt} \bM
\bX_{j}\bX_{\pi^{\natural -1}(j)}^{\rmt}
\bX_j}}_{0},
\\
\Expc \Lambda_{2, 4} =~& \
\underbrace{\Expc
\bracket{
 \bX_j^{\rmt}\bX_i
\bX_{\pi^{\natural}(i)}^{\rmt} \bM
\bX_{\pi^{\natural}(i)}\bX_{i}^{\rmt}
\bX_{\pi^{\natural}(i)} }}_{0}
+ \underbrace{\Expc
\bracket{
 \bX_j^{\rmt}\bX_i
\bX_{\pi^{\natural}(i)}^{\rmt} \bM
\bX_{\pi^{\natural 2}(i)}\bX_{\pi^{\natural}(i)}^{\rmt}
\bX_{\pi^{\natural}(i)} }}_{0} \\
+~& \underbrace{\Expc
\bracket{
 \bX_j^{\rmt}\bX_i
\bX_{\pi^{\natural}(i)}^{\rmt} \bM
\bX_{\pi^{\natural}(j)}\bX_{j}^{\rmt}
\bX_{\pi^{\natural}(i)} }}_{
\Ind_{i = \pi^{\natural}(j)}\trace(\bM)
}
+ \underbrace{\Expc
\bracket{
 \bX_j^{\rmt}\bX_i
\bX_{\pi^{\natural}(i)}^{\rmt} \bM
\bX_{j}\bX_{\pi^{\natural -1}(j)}^{\rmt}
\bX_{\pi^{\natural}(i)} }}_{0},
\]
which gives
\[
\Expc \Lambda_2
=~& (p+1)\Bracket{
2 + 2(p+1)\Ind_{i = j}
+ \Ind_{i = \pi^{\natural}(j)}
}\trace(\bM)
+\Ind_{i = \pi^{\natural 2}(i)}
\Bracket{p+3 + (3p+5)\Ind_{i= j}}\trace(\bM).
\]
\end{proof}

\newpage

\subsection{Supporting Lemmas}
We study the high-order expectations of Gaussian random vectors' inner product, which hopefully
will serve independent interests.
\begin{lemma}
Assume $\bx \in \RR^p$ and $\by \in \RR^p$ are Gaussian distributed
random vectors whose entries follow the i.i.d. standard normal distribution, then we have
\begin{align}
\Expc \trace\bracket{\by \by^{\rmt} \bx \bx^{\rmt}\bM }
=~& \trace(\bM), \label{eq:cross_term_quad_expc} \\
\Expc \trace\bracket{\by \by^{\rmt} \bx^{\rmt}\bM \bx}
=~& \Expc \norm{\by}{2}^2 \trace\bracket{\bx^{\rmt}\bM \bx} = p\trace(\bM)  \label{eq:single_quad_expc}, \\
\Expc\bracket{\bx^{\rmt} \bM \bx}^2 =~& \
\Bracket{\trace(\bM)}^2 + \trace(\bM \bM)+ \trace\bracket{\bM^{\rmt}\bM}, \label{eq:double_quad_expc} \\
\Expc \norm{\bx}{2}^2  (\bx^{\rmt}\bM \bx)
=~& (p+2)\trace(\bM), \label{eq:four_order_expc}\\
\Expc \norm{\bx}{2}^4
(\bx^{\rmt}\bM \bx)
=~& (p+2)(p+4) \trace(\bM) \label{eq:six_order_expc}, \\
\Expc \norm{\bx}{2}^2 \
\bracket{\bx^{\rmt}\bM \bx}^2 =~&
\bracket{p+4}\Bracket{\bracket{\trace(\bM)}^2 + \trace(\bM \bM)+ \trace\bracket{\bM^{\rmt}\bM}}
\label{eq:six_order_expc_type2}, \\
\Expc \norm{\bx}{2}^4 \
\bracket{\bx^{\rmt}\bM \bx}^2 =~&
\bracket{p+4}\bracket{p+6}\Bracket{\bracket{\trace(\bM)}^2 + \trace(\bM \bM)+ \trace\bracket{\bM^{\rmt}\bM}},
\label{eq:eight_order_expc} \\
\Expc (\bx^{\rmt}\by)^2
\by^{\rmt}\bM_1 \bx \bx^{\rmt} \bM_2 \by =~&  2\trace(\bM_1)\trace(\bM_2) + \bracket{p+4}\trace(\bM_1\bM_2)
+ 2\trace(\bM_1 \bM_2^{\rmt}),
\label{eq:eight_cross_term_expc}
\end{align}
where $\bM \in \RR^{p\times p}$ is a fixed matrix.
\end{lemma}

\begin{remark}
If we assume $\bM = \bI_{p\times p}$, we can get
$\Expc \norm{\bx}{2}^4 = p(p+2)$,
$\Expc \norm{\bx}{2}^6 = p(p+2)(p+4)$,
and $\Expc \norm{\bx}{2}^8 = p(p+2)(p+4)(p+6)$.
\end{remark}

\begin{proof}
This lemma is proved by iteratively applying
the Wick's theorem in Theorem~\ref{thm:wick},
Stein's lemma in Lemma~\ref{lemma:stein_lemma},
and Lemma~\ref{lemma:neudecker1987fourth}.

\begin{itemize}[leftmargin=*]
\item
\textbf{Proof of ~\eqref{eq:cross_term_quad_expc} and ~\eqref{eq:single_quad_expc}}.
The proof can be conducted easily with the property such that
$\trace(\bu \bv^{\rmt}) = \bu^{\rmt} \bv = \trace(\bv \bu^{\rmt})$ holds for arbitrary
vectors $\bu$ and $\bv$.

\item
\textbf{Proof of ~\eqref{eq:double_quad_expc}}.
This property is a direct consequence of~\citet{neudecker1987fourth} (Equation~$(3.2)$), which is attached in
Lemma~\ref{lemma:neudecker1987fourth} for the sake of self-containing.

\item
\textbf{Proof of ~\eqref{eq:four_order_expc}}.
Invoking the Stein's lemma, we have
\[
\Expc \norm{\bx}{2}^2 (\bx^{\rmt}\bM\bx)
= \Expc\Bracket{\nabla_{\bx} (\bx^{\rmt}\bM \bx) \bx}.
\]
Then our goal transforms to computing the trace of the Hessian matrix
$\trace\Bracket{\nabla_{\bx} \trace(\bx^{\rmt}\bM \bx) \bx}$.
For the $i$-th entry of the gradient,
we have
\[
\dfrac{d}{dx_i} \bx^{\rmt}\bM \bx =
\la \bM_{i}, \bx \ra + \langle (\bM^{\rmt})_i, \bx\rangle,
\]
where $\bM_i$ is the $i$-th row (or column) of $\bM$. Then we obtain
\[
\dfrac{d}{dx_i} \Bracket{x_i \trace(\bx^{\rmt}\bM \bx)}
= \bx^{\rmt}\bM \bx
+ x_i \Bracket{\la \bM_{i}, \bx \ra + \langle (\bM^{\rmt})_i, \bx\rangle},
\]
and hence
\[
\Expc \norm{\bx}{2}^2 (\bx^{\rmt}\bM\bx)
=~& \sum_{i=1}^p \Expc (\bx^{\rmt}\bM \bx)
+ \sum_{i=1}^p \Expc\Bracket{x_i \bracket{\la \bM_{i}, \bx \ra + \langle (\bM^{\rmt})_i, \bx\rangle}} \\
=~& p\trace(\bM) + 2\sum_i \bM_{ii} = (p+2)\trace(\bM).
\]

\item
\textbf{Proof of ~\eqref{eq:six_order_expc}}.
Following the same strategy as in proving ~\eqref{eq:four_order_expc}, we have
\[
\Expc \Fnorm{\bx}^4 (\bx^{\rmt}\bM\bx)
= \Expc\Bracket{\nabla_{\bx} \norm{\bx}{2}^2 (\bx^{\rmt}\bM \bx) \bx}.
\]
Then our goal transforms to computing the trace of the Hessian matrix
$\trace\Bracket{\nabla_{\bx} \norm{\bx}{2}^2 (\bx^{\rmt}\bM \bx) \bx}$.
For the $i$-th entry of the gradient, we obtain
\[
\dfrac{d}{dx_i}\Bracket{x_i \norm{\bx}{2}^2 (\bx^{\rmt}\bM \bx)} =
& \norm{\bx}{2}^2 \cdot (\bx^{\rmt}\bM \bx)
+ x_i \dfrac{d}{dx_i}\Bracket{\norm{\bx}{2}^2 (\bx^{\rmt}\bM \bx) } \\
=~& \norm{\bx}{2}^2 \cdot (\bx^{\rmt}\bM \bx)
+ 2x_i^2 (\bx^{\rmt}\bM \bx) +
x_i \norm{\bx}{2}^2 \Bracket{\la \bM_{i}, \bx \ra + \langle (\bM^{\rmt})_i, \bx\rangle},
\]
whose expectation reads as
\[
& \Expc \norm{\bx}{2}^2 \cdot (\bx^{\rmt}\bM \bx)
+ 2\Expc\Bracket{x_i^2 (\bx^{\rmt}\bM \bx)}
+ \Expc x_i \norm{\bx}{2}^2 \Bracket{\la \bM_{i}, \bx \ra + \langle (\bM^{\rmt})_i, \bx\rangle} \\
=~& (p+2)\trace(\bM)
+ 2 \Expc\Bracket{x_i^2 M_{ii}x_i^2 + x_i^2 \bracket{\sum_{j\neq i} M_{jj}x_j^2}}
+ 2 M_{ii}\Expc\Bracket{ x_i^2 \norm{\bx}{2}^2} \\
=~& (p+2)\trace(\bM) + 2M_{ii}\bracket{\Expc x_i^4}
+ 2\sum_{j\neq i}M_{jj}(\Expc x_i^2)(\Expc x_j^2) +
2M_{ii}\Bracket{\Expc(x_i^4) + \sum_{j\neq i} (\Expc x_i^2)(\Expc x_j^2)}\\
=~&  (p+2)\trace(\bM) + 6M_{ii} + 2\sum_{j\neq i}M_{jj}
+ 2M_{ii}\bracket{3 + p-1} \\
=~& (p+4)\trace(\bM) + 2(p+4)M_{ii}.
\]
Then we conclude
\[
\Expc \trace\Bracket{\nabla_{\bx} \norm{\bx}{2}^2 (\bx^{\rmt}\bM\bx) \bx}
=~& p(p+2)\trace(\bM) + 2(p+4)\sum_i M_{ii} + 2p \trace(\bM) \\
=~& (p+2)(p+4)\trace(\bM).
\]

\item
\textbf{Proof of ~\eqref{eq:six_order_expc_type2}.}
Invoking the Stein's lemma, we have
\[
\Expc \norm{\bx}{2}^2\bracket{\bx^{\rmt}\bM \bx}^2 = ~&
\sum_i \dfrac{d}{dx_i}\Bracket{x_i \bracket{\bx^{\rmt}\bM \bx}^2}
= p\bracket{\bx^{\rmt}\bM \bx}^2
+ 4\sum_i x_i\bracket{\bx^{\rmt}\bM \bx}
\la \bM^{\sym}_i, \bx\ra.
\]
The proof is then completed by invoking
Lemma~\ref{lemma:wick_expansion}.

\item
\textbf{Proof of ~\eqref{eq:eight_order_expc}.}
Following the same strategy as in proving ~\eqref{eq:six_order_expc_type2},
we consider the
$i$-th gradient w.r.t $x_i$, which can be written as
\begin{align}
& \Expc \norm{\bx}{2}^4 \
\bracket{\bx^{\rmt}\bM \bx}^2 = \sum_i \dfrac{d}{dx_i}\Bracket{x_i \norm{\bx}{2}^2 \bracket{\bx^{\rmt}\bM \bx}^2}
\notag \\
=~& \sum_i  \norm{\bx}{2}^2\bracket{\bx^{\rmt}\bM \bx}^2
+ 2 \sum_i x_i^2 \bracket{\bx^{\rmt}\bM \bx}^2
+ 4\sum_i x_i \norm{\bx}{2}^2
(\bx^{\rmt}\bM \bx)\la \bM^{\sym}, \bx\ra \notag \\
=~& \sum_i \norm{\bx}{2}^2\bracket{\bx^{\rmt}\bM \bx}^2
+ 2 \sum_i \Expc \dfrac{d}{dx_i}\Bracket{x_i \bracket{\bx^{\rmt}\bM \bx}^2}
+4 \sum_i \Expc \dfrac{d}{dx_i}\Bracket{ \norm{\bx}{2}^2
(\bx^{\rmt}\bM \bx)\la \bM^{\sym}_i, \bx\ra}.
\label{eq:eight_order_expc_stein}
\end{align}
Noticing the following relations
\begin{align}
\dfrac{d}{dx_i}\Bracket{x_i \bracket{\bx^{\rmt}\bM \bx}^2}
=~& \bracket{\bx^{\rmt}\bM \bx}^2  +
4x_i \bracket{\bx^{\rmt}\bM \bx}\la \bM^{\sym}_i, \bx\ra,
\label{eq:eight_order_expc_grad1} \\
\dfrac{d}{dx_i}\Bracket{ \norm{\bx}{2}^2
\trace(\bx^{\rmt}\bM \bx)\la \bM^{\sym}_i, \bx\ra} =~& \
2 x_i \bracket{\bx^{\rmt}\bM \bx} \la \bM^{\sym}_i, \bx\ra
\notag \\
+~&  2\norm{\bx}{2}^2 \la \bM^{\sym}_i, \bx\ra^2
+ M_{ii}\norm{\bx}{2}^2 \bracket{\bx^{\rmt}\bM \bx},
\label{eq:eight_order_expc_grad2}
\end{align}
we can conclude the proof by combining
~\eqref{eq:six_order_expc_type2},
~\eqref{eq:eight_order_expc_stein}, ~\eqref{eq:eight_order_expc_grad1},
~\eqref{eq:eight_order_expc_grad2},
and Lemma~\ref{lemma:wick_expansion}.

\item
\textbf{Proof of ~\eqref{eq:eight_cross_term_expc}.}
Due to the independence between $\bx$ and $\by$, we first condition
on $\bx$ and have
\[
& \Expc (\bx^{\rmt}\by)^2 \by^{\rmt}\bM_1 \bx \bx^{\rmt}\bM_2 \by = \
\Expc_{\bx} \Expc_{\by} \by^{\rmt}\bx \bx^{\rmt}\by
\by^{\rmt}\bM_1 \bx \bx^{\rmt}\bM_2 \by  \\
\stackrel{\cirone}{=}~& \Expc_{\bx}\trace\bracket{\bx\bx^{\rmt}}\trace\bracket{\bM_1\bx \bx^{\rmt}\bM_2}
+ \Expc_{\bx}\trace\bracket{\bx\bx^{\rmt}\bM_1\bx \bx^{\rmt}\bM_2 }
+ \Expc_{\bx}\trace\bracket{\bx\bx^{\rmt}\bM_2^{\rmt}\bx \bx \bM_1^{\rmt}} \\
=~& \Expc_{\bx} \norm{\bx}{2}^2 \bx^{\rmt}\bM_2\bM_1 \bx + \
\Expc_{\bx} \bx^{\rmt}\bM_1\bx \bx^{\rmt}\bM_2 \bx + \
\Expc_{\bx} \bx^{\rmt}\bM_1^{\rmt}\bx \bx^{\rmt}\bM_2^{\rmt}\bx \\
\stackrel{\cirtwo}{=}~&  2\trace(\bM_1)\trace(\bM_2) + \bracket{p+4}\trace(\bM_1\bM_2)
+ 2\trace(\bM_1 \bM_2^{\rmt}),
\]
where in $\cirone$ and $\cirtwo$ we both use
Lemma~\ref{lemma:neudecker1987fourth}.

\end{itemize}
\end{proof}

\begin{lemma}
\label{lemma:wick_expansion}
For a fixed matrix $\bM \in \RR^{p\times p}$, we associate it with
a symmetric matrix
$\bM^{\sym}$ defined as $\nfrac{(\bM + \bM^{\rmt})}{2}$. Consider
the Gaussian distributed random vector $\bx\sim \normdist(\bZero, \bI)$,
we have
\[
\Expc \sum_i x_i(\bx^{\rmt}\bM \bx)
\la \bM^{\sym}_i, \bx\ra
= \bracket{\trace(\bM)}^2
+ \Fnorm{\bM}^2 + \trace(\bM \bM).
\]
\end{lemma}
\begin{proof}
This lemma is a direct application of Wick's theorem, which
is completed by showing
\[
& \Expc x_i\bracket{\bx^{\rmt}\bM \bx}
\la \bM^{\sym}_i, \bx\ra = \Expc \sum_{j}\sum_{\ell_1, \ell_2}
M^{\sym}_{i,j}M_{\ell_1, \ell_2}x_i x_j x_{\ell_1}x_{\ell_2} \\
\stackrel{}{=}~&
\underbrace{\Expc \sum_{j}\sum_{\ell_1, \ell_2}
\Ind_{\ell_1 = i}\Ind_{\ell_2 = j}
M^{\sym}_{i,j}M_{\ell_1, \ell_2}x_i x_j x_{\ell_1}x_{\ell_2}}_{
 \sum_j M^{\sym}_{i,j}M_{i, j}
} +
\underbrace{\Expc \sum_{j}\sum_{\ell_1, \ell_2}
\Ind_{\ell_2 = i}\Ind_{\ell_1 = j}
M^{\sym}_{i,j}M_{\ell_1, \ell_2}x_i x_j x_{\ell_1}x_{\ell_2}}_{
 \sum_j
M^{\sym}_{i,j}M_{j, i}
} \\
+~&
\underbrace{\Expc \sum_{j}\sum_{\ell_1, \ell_2}
\Ind_{i = j}\Ind_{\ell_1 = \ell_2}
M^{\sym}_{i,j}M_{\ell_1, \ell_2}x_i x_j x_{\ell_1}x_{\ell_2}}_{
 \sum_{\ell}
M^{\sym}_{i, i}M_{\ell, \ell} } =
\sum_{j} 2\Bracket{M^{\sym}_{i,j}}^2 + M_{i, i} \trace(\bM),
\]
where $\bM^{\sym}$ is defined as $\bracket{\bM + \bM^{\rmt}}/2$.
\end{proof}

\newpage

Then we study the properties of $\bSigma$, which is defined as
$\bX^{\rmt}\bPi^{\natural}\bX - \bDelta$.

\begin{lemma}
\label{lemma:sigmaMsigma_trace_expc}
For a fixed matrix $\bM$, we have
\[
\Expc\trace\bracket{\bSigma \bM \bSigma^{\rmt}} =
n^2\Bracket{\frac{p}{n}+ \bracket{1 - \frac{h}{n}}^2
+ o(1)} \trace(\bM),
\]
where matrix $\bSigma$ is defined in ~\eqref{eq:sigma_matrix_def}.
\end{lemma}
\begin{proof}
We conclude the proof by showing
\[
\Expc\trace\bracket{\bSigma \bM \bSigma^{\rmt}}
\stackrel{\cirone}{=}~& \sum_{\ell_1, \ell_2\in \calS}
\Expc \trace\Bracket{\bX_{\ell_1}\bX_{\ell_1}^{\rmt}\bM \bX_{\ell_2} \bX_{\ell_2}^{\rmt} }
+ \sum_{\ell_1, \ell_2\in \calD}
\Expc \trace\Bracket{\bX_{\ell_1}\bX_{\pi^{\natural}(\ell_1)}^{\rmt}\bM \bX_{\pi^{\natural}(\ell_2)}\bX_{\ell_2}^{\rmt} } \\
=~& \sum_{\ell \in \calS} \Expc\trace\bracket{\bX_{\ell}\bX_{\ell}^{\rmt}\bM \bX_{\ell}\bX_{\ell}^{\rmt}} \
+ \sum_{\ell_1, \ell_2 \in \calS, \ell_1\neq \ell_2}
\Expc \trace\Bracket{\bX_{\ell_1}\bX_{\ell_1}^{\rmt}\bM \bX_{\ell_2} \bX_{\ell_2}^{\rmt} } \\
+~&
\sum_{\ell \in \calD} \
\underbrace{\Expc \trace\Bracket{\bX_{\ell}\bX_{\pi^{\natural}(\ell)}^{\rmt}\bM \bX_{\pi^{\natural}(\ell)}\bX_{\ell}^{\rmt} }}_{p \trace(\bM)} +
\sum_{(\ell_1, \ell_2)\in \calD_{\-{pair}}}
\underbrace{\Expc \trace\Bracket{\bX_{\ell_1}\bX_{\ell_2}^{\rmt}\bM \bX_{\ell_1}\bX_{\ell_2}^{\rmt} }}_{\trace(\bM)} \\
=~& (n-h)(p+2)\trace(\bM) + (n-h)(n-h-1)\trace(\bM)
+ hp\trace(\bM)+ |\calD_{\textup{pair}}|\trace(\bM) \\
\stackrel{\cirtwo}{=}~&
n^2\Bracket{\frac{p}{n}+ \bracket{1 - \frac{h}{n}}^2
+ o(1)} \trace(\bM),
\]
where $\cirone$ is due to the definitions of index sets
$\calS$ and $\calD$ (Equation~\eqref{eq:index_sset_def} and in
Equation~\eqref{eq:index_dset_def}), and $\cirtwo$ is because $|\calD_{\-{pair}}| \leq h$.
\end{proof}

\vspace{0.2in}

\begin{lemma}
\label{lemma:sigmaMsigmaM_trace_expc}
For a fixed matrix $\bM$,
we have
\[
\Expc \trace(\bSigma \bM \bSigma\bM) =
\bracket{n-h + |\calD_{\-{pair}}|}\Bracket{\trace\bracket{\bM}}^2 +
(n-h)^2\trace\bracket{\bM \bM}  +
n\trace\bracket{\bM\bM^{\rmt}}
\]
\end{lemma}

\begin{proof}
Following the same strategy as in proving Lemma~\ref{lemma:sigmaMsigma_trace_expc},
we complete the proof by showing
\[
\Expc \trace(\bSigma \bM \bSigma\bM)
=~&
\sum_{\ell = \pi^{\natural}(\ell)}
\Expc \trace\bracket{
\bX_{\ell}^{\rmt} \bM \bX_{\ell}
\bX_{\ell}^{\rmt} \bM \bX_{\ell} } + \
\sum_{\ell_1, \ell_2\in \calS, \ell_1 \neq \ell_2}
\Expc \trace\bracket{\bX_{\ell_1}\bX_{\ell_1}^{\rmt} \bM
\bX_{\ell_2}\bX_{\ell_2}^{\rmt} \bM} \\
+~& \sum_{\ell \in \calD}
\Expc\trace\bracket{\bX_{\ell}\bX_{\pi^{\natural}(\ell)}^{\rmt} \bM
\bX_{\ell}\bX_{\pi^{\natural}(\ell)}^{\rmt} \bM} +
\sum_{\ell \in \calD_{\-{pair}}}
\Expc\trace\bracket{
\bX_{\ell}\bX_{\pi^{\natural}(\ell)}^{\rmt} \bM
\bX_{\pi^{\natural}(\ell)}\bX_{\ell}^{\rmt} \bM
} \\
=~&
\bracket{n-h + |\calD_{\-{pair}}|}\Bracket{\trace\bracket{\bM}}^2 +
(n-h)^2\trace\bracket{\bM \bM}  +
n\trace\bracket{\bM\bM^{\rmt}}.
\]
\end{proof}

\newpage

\begin{lemma}
\label{lemma:sigmaM_square}
For a fixed matrix $\bM$, we have
\[
\Expc\Bracket{\trace(\Sigma \bM)}^2 = \
(n-h)^2\Bracket{\trace(\bM)}^2 +
n\trace\bracket{\bM^{\rmt}\bM} + \bracket{n - h + \abs{ \calD_{\-{pair}}}}\trace(\bM \bM).
\]
\end{lemma}

\begin{proof}
We complete the proof by showing
\[
\Expc\bracket{\trace(\Sigma \bM)}^2
=~& \sum_{\ell \in \calS} \Expc\bracket{\bX^{\rmt}_{\ell}\bM \bX_{\ell}}^2
+ \sum_{\ell_1, \ell_2\in \calS, \ell_1\neq \ell_2}\bracket{\trace(\bM)}^2
+ \sum_{\ell \in \calD}
\underbrace{\Expc
\bX_{\pi^{\natural}(\ell)}^{\rmt} \bM \bX_{\ell}
\bX_{\ell}^{\rmt} \bM^{\rmt}\bX_{\pi^{\natural}(\ell)}}_{
\trace(\bM^{\rmt} \bM)
} \\
+~& \sum_{\ell \in \calD_{\-{pair}}}
\underbrace{\Expc \bX_{\pi^{\natural}(\ell)}^{\rmt}\bM\bX_{\ell}\bX_{\ell}^{\rmt}
\bM \bX_{\pi^{\natural}(\ell)}}_{\trace(\bM \bM)} \\
=~&
(n-h)^2\Bracket{\trace(\bM)}^2 +
n\trace\bracket{\bM^{\rmt}\bM} + \bracket{n - h + \abs{ \calD_{\-{pair}}}}\trace(\bM \bM).
\]
\end{proof}

\begin{lemma}
\label{lemma:xioneone_expc}
For a fixed matrix $\bM$, we have
\[
\Expc  \sum_{\ell = \pi^{\natural}(\ell)}
\bX_{\pi^{\natural}(i)}^{\rmt} \bM \bX_{\ell}\bX_{\ell}^{\rmt}\bX_{\pi^{\natural}(i)}
 =~& (n-h)\trace(\bM) + (p+1)\Ind_{i = \pi^{\natural}(i)}\trace(\bM).
\]

\end{lemma}

\begin{proof}
Provided that $i = \pi^{\natural}(i)$, we have
\begin{align}
\Expc \Xi_{1, 1} =~& \Expc \bX_i^{\rmt}\bM \bX_i \bX_i^{\rmt}\bX_i
+ \sum_{\ell \neq i,~\ell = \pi^{\natural}(\ell)}
\Expc \bX_{i}^{\rmt}\bM \bX_{\ell}\bX_{\ell}^{\rmt}\bX_i \notag \\
=~& (p+2)\trace(\bM) + \sum_{\ell \neq i,~\ell = \pi^{\natural}(\ell)}
\trace(\bM) =
\bracket{n-h+p+1}\trace(\bM) \Ind_{i = \pi^{\natural}(i)}.
\label{eq:xioneone_expc_equal}
\end{align}
Provided that $i\neq \pi^{\natural}(i)$, we have
\begin{align}
\label{eq:xioneone_expc_diff}
\Expc \Xi_{1, 1} =~& \sum_{\ell = \pi^{\natural}(\ell)}
\Expc \bX_{\pi^{\natural}(i)}^{\rmt}\bM \bX_{\ell}\bX_{\ell}^{\rmt}\bX_{\pi^{\natural}(i)} = (n-h)\trace(\bM)\Ind_{i\neq  \pi^{\natural}(i)}.
\end{align}
Combining ~\eqref{eq:xioneone_expc_equal} and ~\eqref{eq:xioneone_expc_diff} then completes the proof.
\end{proof}

\begin{lemma}
\label{lemma:xionetwo_expc}
For a fixed $\bM$, we have
\[
\Expc \sum_{\ell} \bX_{\pi^{\natural}(i)}^{\rmt} \bM
\bX_{\pi^{\natural}(\ell)}\bX_{\ell}^{\rmt} \bX_j =
\bracket{p \Ind_{i = j}\Ind_{i\neq \pi^{\natural}(i)} + \Ind_{j = \pi^{\natural 2}(i)}}\trace(\bM).
\]
\end{lemma}
\noindent We omit its proof as it is a direct application of Wick's theorem
(Theorem~\ref{thm:wick}).


\begin{lemma}
We have
\[
\Expc \Ind(i = \pi^{\natural}(i)) = ~&  \frac{n-h}{n}(1 + \oprate{1} ),  \\
\Expc \Ind_{i =j} =~&  \frac{h}{n^2}(1+\oprate{1} ), \\
\Expc \Ind_{j = \pi^{\natural 2}(i)} =~& \frac{h}{n^2}(1 + \oprate{1} ), \\
\Expc \Ind_{i = j}\Ind_{i = \pi^{\natural 2}(i)} =~& \
\frac{|\calD_{\-{pair}}|}{n^2}(1+\oprate{1}).
\]
\end{lemma}
This lemma can be easily proved by assuming
the indices $i$, $j$, $\pi^{\natural}(i)$,
and $\pi^{\natural}(j)$ are uniformly sampled from
the set $\{1, 2,\cdots, n\}$

\section{Useful Facts}
This section collects some useful facts for the
sake of self-containing.

\begin{theorem}[Wick's theorem (Theorem~$1.28$ in~\citet{janson_1997})]
\label{thm:wick}
Considering the centered jointly normal variables
$g_1, g_2, \cdots, g_n$, we
conclude
\[
\Expc\bracket{g_1 g_2 \cdots g_n}
= \sum_{\substack{\textup{all possible disjoint}\\ \textup{pairs } (i_k, j_k)
\in \set{1, 2,\cdots, n} }}\prod_{k}\Expc\bracket{g_{i_k}g_{j_k}}.
\vspace{-2mm}
\]
\end{theorem}
With Wick's theorem, we can reduce the computation of
high-order Gaussian moments to calculating the expectations of
a series of low-order Gaussian moments.

\begin{lemma}[Equation~$(3.2)$ in~\citet{neudecker1987fourth}]
\label{lemma:neudecker1987fourth}
For a normally distributed random matrix $\bG \in \RR^{n\times p}$ which satisfies
$\Expc \bG = \bZero$ and $\Expc\-{vec}(\bG)\-{vec}(\bG)^{\rmt} = \bU \otimes \bV$,
we have
\[
\Expc\bracket{\bG^{\rmt}\bA\bG \bC\bG^{\rmt}\bB\bG} = \
& \trace\bracket{\bA\bU}\trace\bracket{\bB\bU}\bV\bC\bV  + \trace\bracket{\bA \bU\bB^{\rmt}\bU} \bV\bC^{\rmt}\bV \\
+~& \trace\bracket{\bA\bU\bB\bU}\trace\bracket{\bC\bV} \bV,
\]
where $\-{vec}(\cdot)$ is the vector operation; $\otimes$ is the
Kronecker product~\citep{horn1990matrix}; and $\bA, \bB$ and $\bC$
are arbitrary fixed matrices.
\end{lemma}

\begin{lemma}[Stein's Lemma (cf. Section~$1.3$
in~\citet{talagrand2010mean})]
\label{lemma:stein_lemma}
Let $g \sim \normdist(0, 1)$. Then for any
differentiable function $f:\RR\mapsto \RR$ we have
\[
\Expc [gf(g)] = \Expc f^{'}(g),
\]
where $\lim_{\|g\|\rightarrow \infty}f(g)e^{-a\|g\|^2_2} = 0$
for any $a > 0$.
\end{lemma}

\end{document}